\documentclass{article}

\usepackage{arxiv}

\usepackage[
  backend=biber,
  bibstyle=ieee,
  citestyle=numeric-comp,
  maxbibnames=9,
  mincitenames=1,
  maxcitenames=1,
  sorting=none
]{biblatex}
\DeclareSourcemap{
  \maps[datatype=bibtex, overwrite]{
    \map{
      \step[fieldset=address, null]
      \step[fieldset=publisher, null]
      \step[fieldset=url, null]
      \step[fieldset=urldate, null]
      \step[fieldset=isbn, null]
      \step[fieldset=issn, null]
      \step[fieldset=number, null]
      \step[fieldset=doi, null]
      \step[fieldset=abstract, null]
      \step[fieldset=volume, null]
      \step[fieldset=pages, null]
      \step[fieldset=language, null]
      \step[fieldset=month, null]
      \step[fieldset=series, null]
      \step[fieldset=file, null]
      \step[fieldset=note, null]
    }
  }
}

\usepackage[utf8]{inputenc} 

\usepackage[ruled,vlined,linesnumbered]{algorithm2e}
\usepackage{amsfonts}       
\usepackage{amssymb}
\usepackage{amsthm}
\usepackage[english]{babel}
\usepackage{bm}
\usepackage{bbm}
\usepackage{booktabs}       
\usepackage{colortbl}
\usepackage{csquotes}
\usepackage{enumitem}
\usepackage{etoolbox}
\usepackage[T1]{fontenc}    
\usepackage{hyperref}       
\usepackage{graphicx}
\usepackage{lipsum}
\usepackage{makecell}
\usepackage{mathtools}
\usepackage{microtype}      
\usepackage{multirow}
\usepackage{mwe}
\usepackage{nicefrac}       
\usepackage{subcaption}
\usepackage{thmtools}
\usepackage{url}            
\usepackage{xcolor}         

\makeatletter
\patchcmd{\@algocf@start}
  {-1.5em}
  {0pt}
  {}{}
\makeatother

\hypersetup{
  colorlinks=true,
  linkcolor=blue,
  citecolor=blue
}

\SetCommentSty{mycommfont}

\newtheorem{theorem}{Theorem}[section]
\newtheorem{lemma}[theorem]{Lemma}

\newtheorem{corollary[theorem]}{Corollary}
\newtheorem{definition}[theorem]{Definition}
\newtheorem{condition}[theorem]{Condition}

\newtheorem{remark}{Remark}
\newtheorem{assumption}{Assumption}
\let\oldremark\remark
\renewcommand{\remark}{\oldremark\normalfont}

\newenvironment{hproof}{%
  \proof}{\endproof}

\declaretheoremstyle[
notefont=\normalfont\itshape,
notebraces={}{},
headfont=\normalfont\itshape,
bodyfont=\normalfont,
headformat=\NAME\NOTE
]{nopar}

\declaretheorem[name={Proof of}, style=nopar, qed=$\blacksquare$]{refproof}

\definecolor{gainsboro}{rgb}{0.86, 0.86, 0.86}
  
\newcommand{\Var}{\operatorname{Var}}
\newcommand{\Tr}{\operatorname{Tr}}
\DeclareMathOperator*{\argmax}{argmax}

\SetKwInput{KwInput}{Input}
\SetKwInput{KwInit}{Initialize}
\SetKwInput{KwTest}{Test}
\SetKwInput{KwDef}{Definition}
\SetKwInput{KwReturn}{Return}

\def\HiLi{\leavevmode\rlap{\hbox to \hsize{\color{yellow!33}\leaders\hrule height .8\baselineskip depth .5ex\hfill}}}

\title{An Optimization-based Algorithm for Non-stationary Kernel Bandits without Prior Knowledge}

\author{%
  Kihyuk Hong \\
  University of Michigan\\
  \texttt{kihyukh@umich.edu} \\
  \And
  Yuhang Li \\
  University of Michigan\\
  \texttt{liyuhang@umich.edu} \\
  \And
  Ambuj Tewari \\
  University of Michigan \\
  \texttt{tewaria@umich.edu} \\
}

\begin{document}

\maketitle

\begin{abstract}

We propose an algorithm for non-stationary kernel bandits that does not require prior knowledge of the degree of non-stationarity.
The algorithm follows randomized strategies obtained by solving optimization problems that balance exploration and exploitation.
It adapts to non-stationarity by restarting when a change in the reward function is detected.
Our algorithm enjoys a tighter dynamic regret bound than previous work on non-stationary kernel bandits.
Moreover, when applied to non-stationary linear bandits by using a linear kernel, our algorithm is nearly minimax optimal, solving an open problem in the non-stationary linear bandit literature.
We extend our algorithm to use a neural network for dynamically adapting the feature mapping to observed data.
We prove a dynamic regret bound of the extension using the neural tangent kernel theory.
We demonstrate empirically that our algorithm and the extension can adapt to varying degrees of non-stationarity.

\end{abstract}

\section{Introduction}

The linear bandit (LB) problem \parencite{dani_stochastic_2008} and the kernel bandit (KB) problem \parencite{srinivas_gaussian_2010} are important paradigms for sequential decision making under uncertainty.
They extend the multi-armed bandit (MAB) problem \parencite{robbins_aspects_1952} by modeling the reward function with the side information of each arm provided as a feature vector.
LB assumes the reward function is linear.
KB extends LB to model non-linearity by assuming the reward function lies in the RKHS induced by a kernel.

A recent line of work studies the non-stationary variants of LB and KB where the reward functions can vary over time subject to two main types of non-stationarity budgets: the number of changes and the total variation in the sequence of reward functions.
A common algorithm design principle for adapting to non-stationarity is the principle of forgetting the past.
It has been applied to the non-stationary MAB to design nearly minimax optimal algorithms \parencite{garivier_upper-confidence_2011, besbes_stochastic_2014}.
Similarly, the principle has been applied to the non-stationary LB \parencite{cheung_learning_2019, russac_weighted_2019, zhao_simple_2020, kim_randomized_2020} and the non-stationary KB \parencite{zhou_no-regret_2021, deng_weighted_2022}.

Recently, \textcite{zhao_non-stationary_2021} found an error in a key technical lemma by \textcite{cheung_learning_2019} that affects the concentration bound of regression-based reward estimates under non-stationarity.
Unfortunately, the error is inherited by \textcite{russac_weighted_2019}, \textcite{zhao_simple_2020} and \textcite{kim_randomized_2020}.
The corrected regret bounds of the affected papers are worse than what were originally reported.
Since the correction, finding a nearly minimax optimal algorithm for the non-stationary LB setting has been an open problem.
The same error affected the work on non-stationary KB by \textcite{zhou_no-regret_2021} and they had to correct their initially reported regret bound to a worse one.

Algorithms using the principle of forgetting require the knowledge of the non-stationarity budgets.
For example, sliding window algorithms \parencite{garivier_upper-confidence_2011, cheung_learning_2019, zhou_no-regret_2021} that forget the past by discarding data older than certain time window require the knowledge of the non-stationarity budgets to optimally tune the size of the window.
Since having a prior knowledge of the non-stationarity budgets may not be realistic in practical settings, researchers have developed change detection based algorithms that do not require the knowledge of non-stationarity budgets.
A seminal paper by \textcite{auer_adaptively_2018} demonstrates a change detection based algorithm for the non-stationary two-armed bandit setting.
Their design principle has been applied to MAB \parencite{auer_adaptively_2019} and the contextual bandit setting \parencite{chen_new_2019}.
More recently, \textcite{wei_non-stationary_2021} proposed a reduction called MASTER that equips an algorithm designed for a stationary environment with change detection subroutines to adapt to non-stationarity without the knowledge of non-stationarity budgets.
They provided a reduction of the OFUL algorithm \parencite{abbasi-yadkori_improved_2011} and claimed near-minimax optimality for non-stationary linear bandits.
However, due to the aforementioned error, they had to correct their regret bound to a suboptimal one.

In this paper, we design an algorithm that sidesteps the error and recover the tighter dynamic regret bounds for non-stationary LB and KB that were once thought to be achieved.
We make the following contributions.

\begin{itemize}[leftmargin=*]
\item We design a novel optimization-based algorithm OPKB for stationary kernel bandits that uses inverse propensity score based reward estimates that sidestep the aforementioned error specific to regression based reward estimates.

\item We design an algorithm ADA-OPKB that adapts OPKB to non-stationary settings using change detection.
ADA-OPKB does not require the knowledge of the non-stationarity budgets and enjoys a dynamic regret bound tighter than previous work on non-stationary KB.

\item We show ADA-OPKB is nearly minimax optimal in the non-stationary linear bandit setting, solving an open problem in the non-stationary linear bandit literature.

\item We provide an extension of ADA-OPKB called ADA-OPNN that trains a neural network to dynamically adapt the feature mapping to observed data.
We show a dynamic regret bound for ADA-OPNN when the width of the network is sufficiently large using the neural tangent kernel theory \parencite{jacot_neural_2018}.

\end{itemize}

\subsection{Related Work}

\paragraph{Non-stationary Linear/Kernel Bandits}
Common approaches for non-stationary bandits include restarting periodically, using recent data within fixed time window (sliding-window) and exponentially decaying past observations (discounting).
These approaches require the knowledge of non-stationarity.
\textcite{zhou_no-regret_2021} analyze restarting and sliding-window approaches for adapting a UCB-based algorithm for kernel bandits.
\textcite{deng_weighted_2022} analyze a discounting approach for kernel bandits.
\textcite{russac_weighted_2019}, \textcite{cheung_learning_2019} and \textcite{zhao_simple_2020} propose discounting, sliding-window and restarting approaches for adapting a UCB-based algorithm for linear bandits respectively.
\textcite{cheung_hedging_2022} discuss restarting adversarial linear bandit algorithm.
For the non-stationary setting where the learner does not have the knowledge of the non-stationarity, \textcite{cheung_learning_2019}, \textcite{zhao_simple_2020} and \textcite{cheung_hedging_2022} discuss bandit-over-bandit (BOB) reduction.
\textcite{wei_non-stationary_2021} propose a change detection based reduction (MASTER) and show a reduction of a UCB-based algorithm for linear bandits.

\paragraph{Optimization-based Algorithms}

First proposed for contextual bandits optimization-based algorithms solve optimization problems to find randomized strategies that balance exploration and exploitation \parencite{dudik_efficient_2011, agarwal_taming_2014}. 
The idea is adapted to linear bandits \parencite{lattimore_end_2017, hao_adaptive_2020, lee_achieving_2021}.
Our paper is the first to apply the approach to kernel bandits.

\section{Problem Statement}\label{section:problem-statement}

We consider a bandit problem where the learner and the nature interact sequentially for $T$ time steps.
At each time $t$, the learner plays an action $x_t$ chosen from a finite set of actions $\mathcal{X} = \{a_1, \dots, a_N\} \subset \mathbb{R}^d$.
Then the nature reveals a noisy reward $y_t = r_t(x_t) + \eta_t$ where $r_t : \mathcal{X} \rightarrow \mathbb{R}$ is an unknown reward function at time $t$ and $\{ \eta_t \}_{t = 1}^T$ are independent zero-mean noises with a bound $\vert \eta_t \vert \leq S$.
\footnote{
The boundedness noise assumption is for making use of the Freedman-style inequality (Lemma~\ref{lemma:freedman}).
We can relax this assumption to a subgaussian noise assumption by modifying the Freedman-style inequality using a truncation argument. 
See Appendix~\ref{subsection-subgaussian} for detail.
}

Following the kernel bandit setting commonly used in the literature, we make the following regularity assumption on the reward functions.

\begin{assumption}[Kernel bandit]\label{assumption:rkhs}
The reward functions $r_t$ live in the RKHS $\mathcal{H}$ induced by a continuous positive semi-definite kernel $k : \mathcal{X} \times \mathcal{X} \rightarrow \mathbb{R}$ with $k(x, x) \leq 1$ for all $x \in \mathcal{X}$.
Their norms satisfy $\Vert r_t \Vert_{\mathcal{H}} \leq B$ for all $t = 1, \dots, T$.
The kernel $k$ and the bounds $S$, $B$ are known to the learner.
\end{assumption}

Note that Assumption~\ref{assumption:rkhs} implies $\vert r_t(x) \vert = \langle r_t, k(\cdot, x) \rangle_\mathcal{H} \leq \Vert r_t \Vert_\mathcal{H} \Vert k(\cdot, x) \Vert_\mathcal{H} \leq B$ for all $t = 1, \dots, T$ and $x \in \mathcal{X}$ by the reproducing property of RKHS and Cauchy-Schwarz.
For the rest of the paper, when making Assumption~\ref{assumption:rkhs}, we assume that the learner scales the problem (by $S + B$) so that $\vert r_t(x) \vert \leq 1$ and $\vert  y_t(x) \vert \leq 1$ for simpler exposition.

Before the learner interacts with the nature, the nature chooses a sequence of reward functions $\{r_t\}_{t=1}^T$ subject to two types of non-stationarity budgets simultaneously.
The first budget $V_T$ limits the total variation of the sequence of reward functions: $\sum_{t = 1}^{T - 1} \Vert r_{t + 1} - r_t \Vert_{\infty} \leq V_T$.
The second budget $L_T$ limits the number of changes in the sequence of reward functions: $1 + \sum_{t = 1}^{T - 1} \mathbb{I} \{ r_{t + 1} \neq r_t \} \leq L_T$.

The learner aims to minimize the \textit{dynamic regret} $\textsc{Reg}_T \coloneqq \sum_{t = 1}^T (r_t(x_t^\star) - r_t(x_t))$ where $x_t^\star \coloneqq \argmax_{x \in \mathcal{X}} r_t(x)$ is the best action at time $t$.
Note that $\textsc{Reg}_T$ is the cumulative expected regret against the optimal strategy with full knowledge of the sequence of reward functions.

\begin{table*}[t]
\caption{Regret Bound Comparison of Algorithms for Non-stationary Kernel/Linear Bandits}
\label{table:contribution}
\centering
\begin{tabular}{cccc} 
 \toprule
 Setting & Algorithm & Regret bound in $\widetilde{\mathcal{O}}(\cdot)$ & \thead{Required \\ knowledge} \\
 \midrule
 \multirow{3}{*}{\makecell{Kernel \\ Bandit}} &  R/SW-GPUCB~\parencite{zhou_no-regret_2021} & $\gamma_T^{\frac{7}{8}} T^{\frac{3}{4}} (1 + V_T)^{\frac{1}{4}}$ & $V_T$ \\ \cline{2-4}
 & WGPUCB~\parencite{deng_weighted_2022} & $\dot{\gamma}_T^{\frac{7}{8}} T^{\frac{3}{4}} (1 + V_T)^{\frac{1}{4}}$ \footnote{} & $V_T$ \\
 \cline{2-4}
 & GPUCB+MASTER~(Appendix~\ref{section:master-reduction}) & $\min \{ \gamma_T \sqrt{T L_T}, \gamma_T T^{\frac{2}{3}} V_T^{\frac{1}{3}} + \gamma_T \sqrt{T} \}$ &  \\
 \cline{2-4}
 & \textbf{ADA-OPKB (Ours)}  & $\min \{ \sqrt{d \gamma_T T L_T}, d^{\frac{1}{3}} \gamma_T^{\frac{1}{3}} T^{\frac{2}{3}} V_T^{\frac{1}{3}} + \sqrt{d \gamma_T T} \}$ &  \\
 \midrule
 \multirow{5}{*}{\makecell{Linear \\ Bandit}} & D-LinUCB~\parencite{russac_weighted_2019} & $d^{\frac{7}{8}} T^{\frac{3}{4}} V_T^{\frac{1}{4}} + d \sqrt{T}$ & $V_T$ \\ 
 \cline{2-4}
 & SW-UCB+BOB~\parencite{cheung_learning_2019} & $d^{\frac{7}{8}} T^{\frac{3}{4}} V_T^{\frac{1}{4}} + d \sqrt{T}$ &  \\ 
 \cline{2-4}
 & RestartUCB+BOB~\parencite{zhao_simple_2020} & $d^{\frac{7}{8}} T^{\frac{3}{4}} V_T^{\frac{1}{4}} + d \sqrt{T}$ &  \\ 
 \cline{2-4}
 & Restart-Adv~\parencite{cheung_hedging_2022} & $d^{\frac{2}{3}} T^{\frac{2}{3}} V_T^{\frac{1}{3}} + d \sqrt{T}$ &  $V_T$ \\ 
 \cline{2-4}
 & Restart-Adv+BOB~\parencite{cheung_hedging_2022} & $d^{\frac{2}{3}} T^{\frac{2}{3}} V_T^{\frac{1}{3}} + d^{\frac{1}{2}} T^{\frac{3}{4}}$ & \\ 
 \cline{2-4}
 & LinUCB+MASTER~\parencite{wei_non-stationary_2021} & $\min \{ d \sqrt{T L_T}, d T^{\frac{2}{3}} V_T^{\frac{1}{3}} + d \sqrt{T} \}$ &  \\
 \cline{2-4}
 & \textbf{ADA-OPKB (Ours)} & $\min \{ d \sqrt{T L_T}, d^{\frac{2}{3}} T^{\frac{2}{3}} V_T^{\frac{1}{3}} + d \sqrt{T} \}$ &  \\
 \bottomrule
\end{tabular}
\end{table*}

\subsection{Preliminaries and Notations}

\paragraph{Feature Mapping}

By Mercer's theorem, given a continuous positive semi-definite kernel $k: \mathcal{X} \times \mathcal{X} \rightarrow \mathbb{R}$, there exists a feature mapping $\psi : \mathcal{X} \rightarrow \ell^2$ with $k(x, x') = \langle \psi(x), \psi(x') \rangle$ for all $x, x' \in \mathcal{X}$.
We say feature mappings $\varphi_1$ and $\varphi_2$ are \textit{equivalent} if $\langle \varphi_1(x), \varphi_1(x') \rangle = \langle \varphi_2(x), \varphi_2(x') \rangle$ for all $x, x' \in \mathcal{X}$.
Given a feature mapping $\psi$, we can always find an equivalent $N$-dimensional feature mapping $\varphi : \mathcal{X} \rightarrow \mathbb{R}^N$.
For example, we can decompose the kernel matrix $K = \{ \langle \psi(a_i), \psi(a_j) \rangle \}_{i, j \in [N]}$ into $K = \Phi \Phi^T$ using the Cholesky decomposition where $\Phi \in \mathbb{R}^{N \times N}$, then take $\varphi(a_i) = \Phi^T e_i$ for all $i = 1, \dots, N$.

\paragraph{Maximum Information Gain}

The maximum information gain~\parencite{srinivas_gaussian_2010} of the RKHS induced by a kernel $k$ is defined as the maximum mutual information between observations $\{ f(x_t) + \epsilon_t \}_{t = 1}^T$ with $\epsilon_t \sim N(0, 1)$ and $f(\cdot)$ sampled from a Gaussian process $GP(0, \sigma^{-1} k(\cdot, \cdot))$.
It is a widely used dimensionality measure of RKHS.
As done by \textcite{camilleri_high-dimensional_2021}, we generalize the original definition to support $T$ \textit{fractional} observations, and define $\gamma_{\varphi, T} = \max_{P \in \mathcal{P}_\mathcal{X}} \log \det S_\varphi(TP / \sigma, 1)$ where $\mathcal{P}_\mathcal{A}$ is the set of probability distributions on $\mathcal{A}$, $S_\varphi(Q, \lambda) \coloneqq \sum_{x \in \mathcal{X}} Q(x) \varphi(x) \varphi(x)^T + \lambda I$ and $\varphi$ is an $N$-dimensional feature mapping of $k$.
It can be shown that for equivalent feature mappings $\varphi_1$ and $\varphi_2$ of $k$, we have $\gamma_{\varphi_1, T} = \gamma_{\varphi_2, T}$ (see Appendix~\ref{section:feature-mapping-equivalence}).
Hence, $\gamma_{\varphi, T}$ is fully determined by the underlying kernel $k$ and does not depend on the particular choice of the feature mapping $\varphi$ induced by the kernel.
We suppress the subscript $\varphi$ and write $S(\cdot, \cdot)$ and $\gamma_T$ when clear from the context.
For the connection between the original definition of the maximum information gain and our definition, see Appendix~\ref{section:maximum-information-gain}.

\paragraph{Other Notations}

We use $[n]$ to denote $\{1, \dots, n\}$.
For a semi-positive definite matrix $M$ and a vector $x$, we write $\Vert x \Vert_M^2 = x^T M x$.
We denote by $\mathbb{E}_t[\cdot]$ and $\Var_t[\cdot]$ the conditional expectation and variance respectively given history up to time $t - 1$.
For an interval $\mathcal{I} = [s, t]$, we define $V_\mathcal{I} = \sum_{\tau = s}^{t - 1} \Vert r_{\tau + 1} - r_\tau \Vert_\infty$ and $L_\mathcal{I} = 1 + \sum_{\tau = s}^{t - 1} \mathbb{I}\{ r_{\tau + 1} \neq r_\tau \}$.

\section{Main Result}

The main result of this paper provides a worst-case bound on the dynamic regret of our novel algorithm called ADA-OPKB for the non-stationary kernel bandit setting.

\begin{theorem}\label{theorem:ada-opkb}
Under Assumption~\ref{assumption:rkhs}, without the knowledge of non-stationarity budgets $V_T$ and $L_T$, the dynamic regret of ADA-OPKB is bounded, with high probability, by
$$
\widetilde{\mathcal{O}}(
\min\{ \sqrt{\gamma_{T} L_T T \log N},
( \gamma_{T} V_T \log N)^{1/3} T^{2/3} + \sqrt{\gamma_{T} T \log N} \}
).
$$
\end{theorem}

When the action set $\mathcal{X} \subset \mathbb{R}^d$ is an infinite bounded set, we can take a hypercube of side length $R$ that contains $\mathcal{X}$ and discretize it into $\mathcal{O}((R d / \epsilon)^d)$ hypercubes as done by \textcite{chowdhury_kernelized_2017} where $\epsilon$ is the maximum error of expected reward from discretization.
Discretizing the action set with $\epsilon = 1 / T$ and running ADA-OPKB on the discretized action set lead to a dynamic regret bound of $\widetilde{\mathcal{O}}(\min\{\sqrt{d \gamma_T L_T T}, (d \gamma_T V_T)^{1/3} T^{2/3} + \sqrt{d \gamma_T T} \})$.
We use this bound to compare with previous work on the setting with an infinite action set.

We can reduce the kernel bandit setting to the linear bandit setting by using the linear kernel $k(x, x') = \langle x, x' \rangle$.
As shown in Lemma~\ref{lemma:information-gain-linear}, the maximum information gain of the linear space is $\gamma_T = \mathcal{O}(d \log T)$ and the dynamic regret bound of ADA-OPKB that uses the linear kernel becomes $\widetilde{\mathcal{O}}( \min\{\sqrt{d L_T T \log N}, (d V_T \log N)^{1/3} T^{2/3} + \sqrt{d L_T T \log N} \})$ for the finite action set.
For the infinite action set, we get $\widetilde{\mathcal{O}}( \min\{d \sqrt{L_T T}, d^{2/3} V_T^{1/3} T^{2/3} + d \sqrt{T} \})$ using the discretization technique.

\footnotetext{The dimensionality measure $\dot{\gamma}_T$ used in \textcite{deng_weighted_2022} is related to $\gamma_T$ but they use a discounted kernel matrix computed with an approximate feature mapping for computing $\dot{\gamma}$.}

\paragraph{Relation to Previous Work}
Table~\ref{table:contribution} compares the regret bound of our work to the \textit{corrected} regret bounds of previous works.
The regret bound of ADA-OPKB for non-stationary kernel bandits is tighter than previous work.
Applying to non-stationary linear bandits by using the linear kernel, ADA-OPKB nearly achieves the lower bound $\Omega(d^{\frac{2}{3}} V_T^{\frac{1}{3}} T^{\frac{2}{3}})$ \parencite{cheung_learning_2019}, solving an open problem of finding a nearly minimax optimal algorithm for non-stationary linear bandits.
The best regret bound before our work is by \textcite{cheung_hedging_2022} who discuss that an algorithm for adversarial linear bandits, e.g. Exp3 algorithm~\parencite{lattimore_bandit_2020}, equipped with periodic restarts (Restart-Adv) achieves $\widetilde{\mathcal{O}}(d^{\frac{2}{3}} T^{\frac{2}{3}} V_T^{\frac{1}{3}})$.
However, it requires the knowledge of $V_T$ to tune the frequency of restarts.
They also discuss a bandit-over-bandit reduction of Restart-Adv (Restart-Adv+BOB) that does not require the knowledge of $V_T$.
However, the reduction suffers an additional regret term of $d^{\frac{1}{2}} T^{\frac{3}{4}}$.

The dependence of $\gamma_T$ in the regret bound for kernel bandits is crucial since $\gamma_T$ can grow with $T$.
For example, $\gamma_T$ for the Mat\'ern kernel with smoothness parameter $\nu$ scales as $\widetilde{\Theta}(T^{\frac{d}{2\nu + d}})$ \parencite{vakili_information_2021}.
Previous works on non-stationary kernel bandits \parencite{zhou_no-regret_2021, deng_weighted_2022} show a regret bound of order $\gamma_T^{7/8}T^{3/4}$, which may not be sublinear in $T$.
For example, it is not sublinear in $T$ for Mat\'ern kernel when $\nu / d \leq 5 / 4$.
Our improved regret bound for ADA-OPKB is of order $\min\{\gamma_T^{1/3} T^{2/3}, \sqrt{\gamma_T T}\}$, which is sublinear in $T$ as long as $\gamma_T$ is sublinear in $T$.
As shown by \textcite{vakili_information_2021}, $\gamma_T$ is sublinear for a class of kernels of which eigenvalues decay polynomially or exponentially, which includes the Mat\'ern kernel and the squared exponential kernel. 

\section{Algorithms and Analyses}\label{section:opkb}

We first study stationary kernel bandits where the reward functions do not vary over time.

\subsection{OPKB: Optimization-based Algorithm for Stationary Kernel Bandits}

Central to the OPKB algorithm is the optimization problem (OP) designed to return a randomized strategy that balances exploration and exploitation.
OP uses an empirical suboptimality gap of each action computed based on the inverse propensity score (IPS) estimator~\parencite{camilleri_high-dimensional_2021}.

\begin{definition}
The inverse propensity score (IPS) estimator for the expected reward $r_t(x)$ with respect to $\varphi$ using the observed reward $y_t$ is defined as 
$$
\widehat{\mathcal{R}}_{\varphi, t}(x) \coloneqq \varphi(x)^T S_\varphi(P_t, \sigma / T)^{-1} \varphi(x_t) y_t
$$
for all $x \in \mathcal{X}$ where $P_t$ is the randomized strategy used at time $t$.
Averaging over an interval $\mathcal{I}$, we define
$\widehat{\mathcal{R}}_{\varphi, \mathcal{I}}(x) \coloneqq \frac{1}{\vert \mathcal{I} \vert} \sum_{t \in \mathcal{I}} \widehat{\mathcal{R}}_{\varphi, t}(x)$.
The empirical suboptimality gap of action $x$ from observations in $\mathcal{I}$ is defined as $\widehat{\Delta}_{\varphi, \mathcal{I}}(x) \coloneqq \max_{x' \in \mathcal{X}} \widehat{\mathcal{R}}_{\varphi, \mathcal{I}}(x') - \widehat{\mathcal{R}}_{\varphi, \mathcal{I}}(x)$.
\end{definition}

OP minimizes over $P \in \mathcal{P}_\mathcal{X}$ the objective function
\begin{equation}\label{eqn:op-objective-function}
\sum_{x \in \mathcal{X}} P(x) \widehat{\Delta}(x) - \frac{2}{\beta} \log \det S_\varphi(P, \sigma / T)
\end{equation}
where the first term is the weighted average of the empirical suboptimality gaps that encourages exploitation and the second term is a regularizer that encourages exploration.
That the second term encourages exploration can be seen by the property of the optimal design defined as follows.

\begin{definition}
Given a set of actions $\mathcal{A} \subseteq{\mathcal{X}}$ and a feature mapping $\varphi : \mathcal{X} \rightarrow \mathbb{R}^p$, we define $\pi_\varphi(\mathcal{A}) \coloneqq \argmax_{P \in \mathcal{P}_\mathcal{A}} \log\det S_\varphi(P, \sigma / T)$ and call it the optimal design on $\mathcal{A}$ with respect to $\varphi$.
\end{definition}

The optimal design is a generalization of the Bayesian $D$-optimal design for linear models that maximizes $\log \det (\sum_{x \in \mathcal{A}} P(x) x x^T + R)$, where $R$ is some regularizer.
The Bayesian $D$-optimal design is one of the exploration strategies used in the Bayesian experimental design literature~\parencite{chaloner_bayesian_1995}.
As shown in the following lemma, by playing our definition of the optimal design  $\pi_\varphi(\mathcal{A})$, we can uniformly bound the variance of the IPS estimators over all actions in $\mathcal{A}$.
See Appendix~\ref{section:proof-optimal-design} for proof.

\begin{lemma}\label{lemma:optimal-design}
Consider an optimal design $\pi_\varphi(\mathcal{A})$ with respect to a feature mapping $\varphi$ on a set of actions $\mathcal{A} \subseteq \mathcal{X}$.
If we play an action sampled from $\pi_\varphi(\mathcal{A})$ at time $t$ and observe $y_t$, then for all $x \in \mathcal{X}$, we have
$$
\Var(\widehat{\mathcal{R}}_{\varphi, t}(x)) \leq \Vert \varphi(x) \Vert^2_{S_\varphi(\pi_\varphi(\mathcal{A}), \sigma / T)^{-1}} \leq \gamma_{\varphi, T}.
$$
\end{lemma}

The full OP algorithm is presented below.
Note that due to the concavity of $\log\det(\cdot)$, the optimization problem used by OP and the optimal design can be solved efficiently, for example, by using the interior-point method in \textcite{vandenberghe_determinant_1998}.

\begin{algorithm}
\KwInput{$\varphi$, $\widehat{\Delta} = \{\widehat{\Delta}(x)\}_{x \in \mathcal{X}}$, $\alpha$, $\beta$, $T$}

Find a minimizer $P^\star \in \mathcal{P}_\mathcal{X}$ of (\ref{eqn:op-objective-function}). \\
Find $\mathcal{A} \leftarrow \{x \in \mathcal{X}: \widehat{\Delta}(x) \leq 2 \alpha \gamma_{\varphi, T} / \beta \}$. \label{alg:line:good-arm} \\
\KwReturn{The mixed strategy $Q = \frac{1}{2} P^\star + \frac{1}{2} \pi_\varphi(\mathcal{A})$}

\caption{$\textsc{OP}$: Optimization Problem}
\label{alg:optimization-problem}
\end{algorithm}

The parameter $\beta$ controls the balance between exploration and exploitation.
As stated in Lemma~\ref{lemma:optimization-problem}, the greater the $\beta$, the smaller the expected empirical regret $\sum_{x \in \mathcal{X}} Q(x) \widehat{\Delta}(x)$ and the greater the variance bound.
See Appendix~\ref{section:proof-optimization-problem} for the proof.
Note that OP mixes the minimizer $P^\star$ with the optimal design on the set $\mathcal{A}$ computed in Line~\ref{alg:line:good-arm}.
This step is required to get the bound~(\ref{eqn:improved-bound}), which is the key to bound the bias of the reward estimator for the regret analysis.

\begin{lemma}\label{lemma:optimization-problem}
The distribution $Q$ returned by the algorithm $\textsc{OP}(\varphi, \widehat{\Delta}, \alpha, \beta, T)$
satisfies
\begin{align}
  \sum_{x \in \mathcal{X}} Q(x) \widehat{\Delta}(x)
  &\leq
  \frac{(1 + \alpha)\gamma_{\varphi, T}}{\beta}, \\
  \Vert \varphi(x) \Vert_{S_\varphi(Q, \sigma / T)^{-1}}^2
  &\leq \beta \widehat{\Delta}(x) + 2 \gamma_{\varphi, T}, \,\, \forall x \in \mathcal{X}, \\
  \Vert \varphi(x) \Vert_{S_\varphi(Q, \sigma / T)^{-1}}^2
  &\leq \frac{\beta^2 \widehat{\Delta}(x)^2}{2\alpha \gamma_{\varphi, T}} + 2 \gamma_{\varphi, T}, \,\, \forall x \in \mathcal{X}. \label{eqn:improved-bound}
\end{align}
\end{lemma}

Now, we present the OPKB algorithm (Algorithm~\ref{alg:opkb}).
OPKB takes a feature mapping $\varphi$ as an input.
Assuming the knowledge of the kernel $k$ corresponding to the RKHS in which the reward function lies, we use any feature mapping $\varphi : \mathcal{X} \rightarrow \mathbb{R}^N$ equivalent to the feature mapping $\psi: \mathcal{X} \rightarrow \ell^2$ corresponding the kernel.
The choice of $\varphi$ among the feature mappings equivalent to $\psi$ does not affect the algorithm and the analysis.
See Appendix~\ref{section:feature-mapping-equivalence} for details.
OPKB runs in blocks of doubling sizes.
In the first block, it follows the optimal design for $E$ time steps.
Before starting a new block $j$, it computes the empirical suboptimality gaps using all past history, then runs OP to find the strategy $Q^{(j)}$ and mixes it with the optimal design.
The mixed strategy $P^{(j)}$ is run in block $j$.
Every block, OPKB increases the parameter $\beta$ by a factor of $\sqrt{2}$ when calling OP to increase the degree of exploitation.

\begin{algorithm}
\KwInput{feature map $\varphi$, horizon $T$, confidence $\delta \in (0, 1)$.}
\KwDef{$\mu_j = c_1 2^{-j / 2}$, $E = \lceil c_3 \gamma_{\varphi, T} \log(C_0 N / \delta) \rceil$, $\beta_j = c_2 \gamma_{\varphi, T} 2^{j / 2}$, $\alpha = c_4 \sigma / \log(C_0 N / \delta)$}
\KwInit{$t \leftarrow 1$, $P^{(0)} \leftarrow \pi_\varphi(\mathcal{X})$}
\For{$j = 0, 1, \dots$}{
  Set block $\mathcal{B}(j) \leftarrow [t, t + 2^j  E - 1]$ and cumulative block $\mathcal{C}(j) \leftarrow \cup_{k=0}^j \mathcal{B}(k)$. \\
  \If{$j \geq 1$}{
    Compute empirical gap $\widehat{\Delta} \leftarrow \{ \widehat{\Delta}_{\varphi, \mathcal{C}(j - 1)}(x) \}_{x \in \mathcal{X}}$. \\
    Find strategy $Q^{(j)} \leftarrow \textsc{OP}(\varphi, \widehat{\Delta}, \alpha, \beta_j, T)$. \\
    Set $P^{(j)} \leftarrow (1 - \mu_j) Q^{(j)} + \mu_j \pi_\varphi(\mathcal{X})$.
  }
  \While{$t \in \mathcal{B}(j)$}{
    Play $x_t \sim P^{(j)}$; receive $y_t$; increment $t \leftarrow t + 1$. \\
  }
}
\caption{$\textsc{OPKB}$}
\label{alg:opkb}
\end{algorithm}

\subsection{Analysis of OPKB}

For the analysis of OPKB, we use the following concentration bounds for the reward estimate $\widehat{\mathcal{R}}_{\varphi, \mathcal{C}(m)}(x)$ and the gap estimate $\widehat{\Delta}_{\varphi, \mathcal{C}(m)}$ shown under a more general setting of non-stationary kernel bandits.
The proof is based on a Freedman-style inequality on the martingale difference sequence $\{ \widehat{\mathcal{R}}_{\varphi, t}(x) - \mathbb{E}_t [ \widehat{\mathcal{R}}_{\varphi, t}(x)] \}_{t \in \mathcal{C}(j)}$.
See Appendix~\ref{section:reward-estimate-concentration-gap-proof} for the full proof.

\begin{lemma}\label{lemma:reward-estimate-concentration-gap}
With probability at least $1 - \delta$, when running the OPKB algorithm, we have for all block indices $j = 0, 1, \dots$ and actions $x \in \mathcal{X}$ that
\begin{align}
\vert \widehat{\mathcal{R}}_{\varphi, \mathcal{C}(j)}(x) - &\mathcal{R}_{\mathcal{C}(j)}(x) \vert
\leq \frac{1}{2} \Delta_{\mathcal{C}(j)}(x) + V_{\mathcal{C}(j)} + \frac{c_0 \mu_j}{4} \label{eqn:reward-estimate-concentration} \\
\Delta_{\mathcal{C}(j)}(x)
&\leq 2 \widehat{\Delta}_{\varphi, \mathcal{C}(j)}(x) + 4 V_{\mathcal{C}(j)} + c_0 \mu_j \label{eqn:gap-estimate-concentration1} \\
\widehat{\Delta}_{\varphi, \mathcal{C}(j)}(x)
&\leq 2 \Delta_{\mathcal{C}(j)}(x) + 4 V_{\mathcal{C}(j)} + c_0 \mu_j \label{eqn:gap-estimate-concentration2}
\end{align}
where $\mathcal{C}(j)$ is the interval from time 1 to the end of block $j$, $c_0$ is a universal constant, $\mathcal{R}_{\mathcal{I}}(x) \coloneqq \frac{1}{\vert \mathcal{I} \vert} \sum_{t \in \mathcal{I}} r_t(x)$ is the average reward in $\mathcal{I}$ and $\Delta_{\mathcal{I}}(x) \coloneqq \max_{x' \in \mathcal{X}} \mathcal{R}_{\mathcal{I}}(x') - \mathcal{R}_\mathcal{I}(x)$.
\end{lemma}

\begin{remark} \label{remark:sidestep}
Concentration bounds for regression-based reward estimates for the non-stationary LB and KB given by Lemma 2 in \textcite{zhao_non-stationary_2021} and Lemma 1 in \textcite{zhou_no-regret_2021} are analogous to (\ref{eqn:reward-estimate-concentration}).
However, their bounds have an additional factor of $\sqrt{d}$ and $\sqrt{\gamma_{\varphi, T}}$ respectively for the term $V_{\mathcal{C}(j)}$, leading to suboptimal regret bounds.
Their concentration bounds were believed to have a constant factor for the term $V_{\mathcal{C}(j)}$, but they had to be corrected due to an error found by \textcite{zhao_non-stationary_2021}.
The error is specific to regression-based reward estimates. See \textcite{zhao_non-stationary_2021} for details.
Our algorithm sidesteps the error by using IPS reward estimates instead of regression-based reward estimates.
The main motivation for using randomized strategies in our algorithm is to use IPS reward estimates, which can only be constructed when randomized strategies are used.
\end{remark}

\begin{remark} \label{remark:one-step-regret-bound}
Consider the stationary setting where $V_{\mathcal{C}(j)} = 0$.
The expected one step regret when following $P^{(j)}$ is
$$
\begin{aligned}
\sum_{x \in \mathcal{X}} &P^{(j)}(x) \Delta_{\mathcal{C}(j - 1)}(x) \\
&\leq \sum_{x \in \mathcal{X}} Q^{(j)}(x) \Delta_{\mathcal{C}(j - 1)}(x)
+ 2 \mu_j \sum_{x \in \mathcal{X}} \pi_\varphi(x) \\
&\leq 2 \sum_{x \in \mathcal{X}} Q^{(j)}(x) \widehat{\Delta}_{\varphi, \mathcal{C}(j - 1)}(x) + \mathcal{O}(\mu_j) \leq \mathcal{O}(\mu_j)
\end{aligned}
$$
where $\pi_\varphi$ is the optimal design on $\mathcal{X}$, the first inequality uses $\Delta_{\mathcal{C}(j - 1)} \leq 2$ and the last inequality uses Lemma~\ref{lemma:optimization-problem}.
\end{remark}

By the remark above, we can show the following theorem.

\begin{theorem}\label{theorem:opkb}
Under Assumption~\ref{assumption:rkhs} with stationary reward functions $r_t(\cdot) = r(\cdot)$ for all $t \in [T]$, the dynamic regret bound of OPKB using a feature mapping induced by the kernel $k$ is bounded with high probability by
$$
\textsc{Reg}_T \leq \widetilde{\mathcal{O}}\left(
\sqrt{\gamma_{T} T \log N}
\right).
$$
\end{theorem}

\begin{hproof}
By Remark~\ref{remark:one-step-regret-bound}, the expected regret of the block $\mathcal{B}(j)$ is $\mathcal{O}( \vert \mathcal{B}(j) \vert \sqrt{2^{-j}}) = \mathcal{O}(E \sqrt{2^j})$.
Summing over all blocks gives the bound $\mathcal{O}(\sqrt{\gamma_T T \log N})$ on the expected total regret.
See Appendix~\ref{section:analysis-opkb} for a full proof.
\end{hproof}

Our regret bound for OPKB is order-optimal \parencite{salgia_domain-shrinking_2021} and matches work by \textcite{salgia_domain-shrinking_2021,camilleri_high-dimensional_2021,li_gaussian_2022,valko_finite-time_2013}.
It is an improvement over \textcite{srinivas_gaussian_2010,chowdhury_kernelized_2017}.

\subsection{ADA-OPKB: Adapting OPKB to Non-Stationarity}\label{section:ada-opkb}

In this section, we propose an algorithm called ADA-OPKB for the non-stationary kernel bandit setting that does not require the knowledge of the non-stationarity budgets.

\begin{remark}
Before our paper, the most natural attempt for designing an algorithm for non-stationary KB is to use the MASTER reduction \parencite{wei_non-stationary_2021} on GPUCB \parencite{chowdhury_kernelized_2017}, a UCB-based algorithm for stationary kernel bandits.
This is because the MASTER reduction most naturally works for a UCB-based base algorithm. Also, the required analysis of GPUCB under non-stationary environment is available in the literature \parencite{zhou_no-regret_2021}.
However, as shown in Appendix~\ref{section:master-reduction}, the reduction of GPUCB gives worse dynamic regret bound compared to ADA-OPKB due to the suboptimal concentration bound of regression based reward estimates.
\end{remark}

ADA-OPKB adapts OPKB to non-stationarity by restarting upon detecting a significant change in reward functions.
The key is to use past strategies as change detectors.
Lemma~\ref{lemma:reward-estimate-concentration-gap} suggests that the strategy $P^{(j)}$ can detect changes in suboptimality gaps greater than $\sim \sqrt{2^{-j}}$ after running for $\sim 2^j$ time steps.
ADA-OPKB replays older strategies with small indices to detect large changes fast and more recent strategies to detect small changes after running for longer time intervals.
Algorithm~\ref{alg:ada-opkb} shows the full algorithm.
Highlighted lines indicate the difference from OPKB.

\begin{algorithm}[t]
\KwInput{feature map $\varphi$, horizon $T$, confidence $\delta \in (0, 1)$.}
\KwDef{$\mu_j = c_1 2^{-j / 2}$, $\beta_j = c_2 \gamma_{\varphi, T} 2^{j / 2}$, $E = \lceil c_3 \gamma_{\varphi, T} \log(N / \delta) \rceil$, $\alpha =  c_4 \sigma / \log(N / \delta)$}
\KwInit{$t \leftarrow 1$, epoch index $i \leftarrow 1$, $Q^{(0)} \leftarrow \pi_\varphi(\mathcal{X})$}
\For{$j = 0, 1, \dots$\label{alg:line:first} }{
  Set $\mathcal{B}(j) \leftarrow [t, t + 2^j  E - 1]$ and $\mathcal{C}(j) \leftarrow \cup_{k=0}^j \mathcal{B}(k)$. \\
  \If{$j \geq 1$}{
    Compute $\widehat{\Delta} \leftarrow \{ \widehat{\Delta}_{\varphi, \mathcal{C}(j - 1)}(x) \}_{x \in \mathcal{X}}$. \\
    Find strategy $Q^{(j)} \leftarrow \textsc{OP}(\varphi, \widehat{\Delta}, \alpha, \beta_j, T)$. \\
    Set $P^{(j)} \leftarrow (1 - \mu_j) Q^{(m_t)} + \mu_j \pi_\varphi(\mathcal{X})$. \\
  }
  \HiLi Generate replay schedule $\mathcal{S} \leftarrow \textsc{Schedule}(t, j)$. \\
  \While{$t \in \mathcal{B}(j)$}{
    \HiLi $m_t \leftarrow \min \{ m: (m, \mathcal{I}) \in \mathcal{S}\,\, \text{with}\,\,t \in \mathcal{I} \}$ \tcp*{smallest index of scheduled intervals}\label{alg:line:strategy-index}
    Play $x_t \sim P^{(m_t)}$; receive $y_t$; increment $t \leftarrow t + 1$. \\
    \HiLi If \textbf{Test} triggers a restart, increment $i$; go to Line~\ref{alg:line:first}.
  }
}
\KwTest{Trigger a restart if for any $(m, \mathcal{I}) \in \mathcal{S}$ with $\mathcal{I}$ ending at $t$ and $k < j$, the following holds
\begin{align}
\begin{split}
\widehat{\Delta}_{\varphi, \mathcal{I}}(x) - 4 \widehat{\Delta}_{\varphi, \mathcal{C}(k)}(x) &> 4 c_0 \mu_{m \wedge k} \quad \text{or} \\
\widehat{\Delta}_{\varphi, \mathcal{C}(k)}(x) - 4 \widehat{\Delta}_{\varphi, \mathcal{I}}(x) &> 4 c_0 \mu_{m \wedge k}. \label{eqn:change-detection}
\end{split}
\end{align}
}
\caption{$\textsc{ADA-OPKB}$: ADAptive Optimization Problem based Kernel Bandit Algorithm}
\label{alg:ada-opkb}
\end{algorithm}

Before starting a new block $j$, ADA-OPKB calls $\textsc{Schedule}$ (Algorithm~\ref{alg:replay-schedule}), similar to the scheduler in \textcite{wei_non-stationary_2021}), for determining when to use which of the strategies $P^{(0)}, \dots, P^{(j)}$.
The procedure generates a set of replay intervals denoted by $(m, \mathcal{I})$ where $m$ indicates the strategy index and $\mathcal{I}$ indicates the time interval scheduled for playing the strategy $P^{(m)}$.
A replay schedule of index $m$ has length $2^m E$ and there are $2^{j - m}$ slots in block $j$ available to be scheduled.
For each slot, the algorithm randomly schedule a replay of index $m$ with probability $\sqrt{2^{m - j}}$.
When multiple replay intervals are scheduled at a given time $t$, the algorithm selects the one with the smallest index.
The strategy used at time $t$ is denoted by $m_t$.
Upon completion of a replay interval $\mathcal{I}$, the change detection test (\ref{eqn:change-detection}) is run.
A restart is triggered if the test detects a significant change in reward functions.
The test is based on the comparison of the empirical gap $\widehat{\Delta}_{\varphi, \mathcal{I}}$ and $\widehat{\Delta}_{\varphi, \mathcal{C}(k)}$ where $\mathcal{C}(k)$ is any cumulative block prior to $\mathcal{I}$.

\begin{algorithm}
\KwInput{starting time $t$, block index $j$, base block size $E$}
\KwInit{$\mathcal{S} \leftarrow \{ (j, [t, t + 2^j E - 1]) \}$}
\For{$\tau = 0, \dots, 2^j E - 1$}{
  \For{$m = 0, \dots, j - 1$}{
    \If{$\tau$ is a multiple of $2^m E$}{
      With probability $\frac{\sqrt{2^m}}{\sqrt{2^j}}$, add $(m, [t + \tau, t + \tau + 2^m E - 1])$ to $\mathcal{S}$.
    }
  }
}
\KwReturn{$\mathcal{S}$}

\caption{$\textsc{Schedule}$}
\label{alg:replay-schedule}
\end{algorithm}

\subsection{Analysis of ADA-OPKB}

With the key lemmas proved for analyzing OPKB, we use ideas from \textcite{chen_new_2019} and \textcite{wei_non-stationary_2021} to analyze ADA-OPKB.
We provide a sketch of the proof below.
We suppress the dependency of the regret bound on $\gamma_T$ and $\log N$ for simplicity.
See Appendix~\ref{section:analysis-ada-opkb} for the full proof.

\paragraph{Step 1: Interval Regret}

Using a martingale concentration, we can bound the regret of an interval $\mathcal{J}$ inside a block $j$ as $\textsc{Reg}_\mathcal{J} \leq \mathcal{O}(\sum_{t \in \mathcal{J}} \mu_{m_t} + \vert \mathcal{J} \vert V_\mathcal{J} + \vert \mathcal{J} \vert \zeta_\mathcal{J})$ where $\zeta_{\mathcal{J}} \coloneqq \max_{x \in \mathcal{X}} ( \Delta_{\mathcal{J}}(x) - 8 \widehat{\Delta}_{\mathcal{C}(j - 1)}(x) )$ measures the change in average reward in $\mathcal{J}$ compared to the previous block $j - 1$.
See Appendix~\ref{subsection:interval-regret} for the proof.
Note that the interval regret is a sum of the expected one step regret assuming stationarity (Remark~\ref{remark:one-step-regret-bound}), the degree of non-stationarity within $\mathcal{J}$, and the magnitude of the change in reward function compared to the last block.

\paragraph{Step 2: Block Regret}

To bound the regret of a block $j$, we partition the block into nearly stationary intervals $\mathcal{J}_1, \dots, \mathcal{J}_\ell$ so that $V_{\mathcal{J}_i} \leq \mu_{\mathcal{J}_i}$ where $\mu_\mathcal{I} \coloneqq c \vert \mathcal{I} \vert^{-1/2}$.
Summing over the interval regret of $\mathcal{J}_k$ in Step 1 and applying Cauchy-Schwarz, we get $\textsc{Reg}_{\mathcal{B}(j)} \leq \mathcal{O}(\sum_{t \in \mathcal{B}(j)} \mu_{m_t} + \sqrt{\ell} \vert \mathcal{B}(j) \vert \mu_j + \sum_{k = 1}^\ell \vert \mathcal{J}_k \vert \zeta_{\mathcal{J}_k} \mathbb{I}\{ \zeta_{\mathcal{J}_k} > c' \mu_{\mathcal{J}_k}\})$.
The first term $\sum_{t \in \mathcal{B}(j)} \mu_{m_t}$ can be shown to be $\widetilde{\mathcal{O}}(\vert \mathcal{B}(j) \vert \mu_j)$, which suggests the replays of past strategies are not overdone (Lemma~\ref{lemma:replay-schedule}).
To bound the third term, we use the property of change detection test that when $\zeta_{\mathcal{J}_k}$ is above $c' \mu_{\mathcal{J}_k}$ then replaying a suitable strategy within $\mathcal{J}_k$ triggers a restart (Lemma~\ref{lemma:change-detection}).
We can show that the replays of past strategies are done frequently enough to terminate the block before the third term gets too large, leading to a bound $\widetilde{\mathcal{O}}(\sqrt{\ell} \vert \mathcal{B}(j) \vert \mu_j)$ (proof of Lemma~\ref{lemma:block-regret}).
Finally, we can greedily construct a partition with $\ell = \widetilde{\mathcal{O}}(\min\{ L_{\mathcal{B}(j)}, V_{\mathcal{B}(j)}^{2/3} \vert \mathcal{B}(j) \vert^{1/3} \})$ (Lemma~\ref{lemma:interval-partition}), which gives a block regret bound of $\widetilde{\mathcal{O}}(\min\{ \sqrt{2^j L_{\mathcal{B}(j)}}, V_{\mathcal{B}(j)}^{1/3} (2^j)^{2/3} \})$ (Lemma~\ref{lemma:block-regret}).

\paragraph{Step 3: Epoch Regret}

Since the block size is doubling, there can be at most $\mathcal{O}(\log_2 T)$ blocks in an epoch.
Summing up regret bounds of the blocks and applying Cauchy-Schwarz and H\"older's inequalities, we can bound the epoch regret by $\widetilde{\mathcal{O}}(\min \{ \sqrt{L_{\mathcal{E}_i} \vert \mathcal{E}_i \vert}, V_{\mathcal{E}_i}^{1/3} \vert \mathcal{E}_i \vert^{2/3} \})$ (Lemma~\ref{lemma:epoch-regret}).

\paragraph{Step 4: Total Regret}

By the property of the change detection test, restarts can be triggered only when the degree of non-stationarity is large enough (Lemma~\ref{lemma:no-false-alarm}).
Using this property, we can bound the number of epochs by $\widetilde{\mathcal{O}}(\min\{ L_T, V_T^{2/3} T^{1/3} \}$ (Lemma~\ref{lemma:epoch-bound}).
The epoch regret bound in Step 3 gives total regret bound of $\widetilde{\mathcal{O}}(\min \{ \sqrt{L_T T}, V_T^{1/3} T^{2/3} \})$ (Theorem~\ref{theorem:ada-opkb}).

\section{Dynamic Feature Mapping Using a Neural Network}\label{section:ada-opnn}

Recall that OPKB and ADA-OPKB use a fixed feature mapping induced by a kernel.
In this section, we present extensions of OPKB and ADA-OPKB called OPNN and ADA-OPNN respectively that use \textit{dynamic} feature mappings induced by a neural network trained using past history.

\subsection{Preliminaries and Notations}

\paragraph{Neural Network}

Following \textcite{zhou_neural_2020}, we use a fully connected neural network with width $m$ and depth $L$: $f(x; \bm{W}) = \sqrt{m} \bm{W}_L \sigma(\bm{W}_{L - 1} \sigma(\cdots \sigma(\bm{W}_1 x)\cdots ))$ where $\sigma(x) = \max\{x, 0\}$ is the ReLU activation function, $\bm{W}_1 \in \mathbb{R}^{m \times d}$, $\bm{W}_i \in \mathbb{R}^{m \times m}$ for $i = 2, \dots, L - 1$, $\bm{W}_L \in \mathbb{R}^{m \times 1}$ and $\bm{W} = [\text{vec}(\bm{W}_1)^T, \dots, \text{vec}(\bm{W}_L)^T]^T \in \mathbb{R}^p$ with $p = m + md + m^2(L - 1)$.
We denote by $g(x; \bm{W}) = \nabla_{\bm{W}} f(x; \bm{W}) \in \mathbb{R}^p$ the gradient of the neural network function.
We call $g(\cdot; \bm{W})$ the feature mapping induced by the neural network $f$ with parameter $\bm{W}$.
Each entry of the initial weights $\bm{W}^{(0)}$ of the network is sampled independently from $\mathcal{N}(0, 2 / m)$.

\paragraph{Neural Tangent Kernel}

By \textcite{jacot_neural_2018}, $\langle g(x; \bm{W}^{(0)}), g(x'; \bm{W}^{(0)}) \rangle$ converges in probability to $H(x, x')$ for all $x, x' \in \mathcal{X}$ where the deterministic kernel $H(\cdot, \cdot)$ is called the \textit{neural tangent kernel}.
We denote by $\bm{H} = \{ H(x, x') \}_{x, x' \in \mathcal{X}}$ the neural tangent kernel matrix.

\begin{algorithm}[h]
\KwInput{network width $m$, network depth $L$, time horizon $T$, confidence level $\delta \in (0, 1)$.}
\KwInit{$t \leftarrow 1$, initialize network weights $\bm{W}^{(0)}$, compute feature mapping $\varphi^{(0)}$ equivalent to $g(\cdot; \bm{W}^{(0)})$, find optimal design $P^{(0)} \leftarrow \pi_{\varphi^{(0)}}(\mathcal{X})$}
\KwDef{$\mu_j = c_1 2^{-j / 2}$, $\beta_j = c_2 \gamma_{\varphi^{(0)}, T} 2^{j / 2}$, $E = c_3 \gamma_{\varphi^{(0)}, T} \log(C_0 N / \delta)$, $\alpha = c_4 \sigma / \log(C_0 N / \delta)$}
\For{$j = 0, 1, \dots$}{
  Set $\mathcal{B}(j) \leftarrow [t, t + 2^j  E - 1]$ and $\mathcal{C}(j) \leftarrow \cup_{k=0}^j \mathcal{B}(k)$. \\
  \If{$j \geq 1$}{
    $\bm{W}^{(j)} \leftarrow \textsc{TrainNN}(\{(x_\tau, y_\tau)\}_{\tau \in \mathcal{C}(j - 1)}, \bm{W}^{(0)})$. \\
    Find a mapping $\varphi^{(j)}$ equivalent to $g(\cdot; \bm{W}^{(j)}) / \sqrt{m}$. \\
    Compute $\widehat{\Delta} \leftarrow \{ \widehat{\Delta}_{\varphi^{(j)}, \mathcal{C}(j - 1)}(x) \}_{x \in \mathcal{X}}$. \\
    Find strategy $Q^{(j)} \leftarrow \textsc{OP}(\varphi^{(j)}, \widehat{\Delta}, \alpha, \beta_j, T)$. \\
    Set $P^{(j)} \leftarrow (1 - \mu_j) Q^{(j)} + \mu_j \pi_{\varphi^{(j)}}(\mathcal{X})$.
  }
  \While{$t \in \mathcal{B}(j)$}{
    Play $x_t \sim P^{(j)}$. Receive $y_t$. Increment $t \leftarrow t + 1$. \\
  }
}
\caption{$\textsc{OPNN}$: Optimization Problem based algorithm using Neural Network}
\label{alg:opnn}
\end{algorithm}

For the analysis of OPNN and ADA-OPNN, we make the following assumptions.
The first assumption is on the invertibility of the neural tangent kernel matrix $\bm{H}$.

\begin{assumption}\label{assumption:invertibility}
For some $\lambda_0 > 0$, we have $\bm{H} \succcurlyeq \lambda_0 \bm{I}$.
\end{assumption}
This is a mild assumption commonly made when analyzing neural networks \parencite{du_gradient_2019, arora_exact_2019} and for analyzing neural bandit algorithms  \parencite{salgia_provably_2022, zhou_neural_2020, zhang_neural_2020, gu_batched_2021, kassraie_neural_2021}.
It is satisfied, for example, as long as no two actions in $\mathcal{X}$ are parallel (see Theorem 3.1 in \textcite{du_gradient_2019-1}).
The second assumption is on the regularity of the reward functions commonly made in the neural bandits literature \parencite{zhou_neural_2020, zhang_neural_2020, gu_batched_2021}.

\begin{assumption}\label{assumption:nn-reward}
We have $\sqrt{\bm{r}_t^T \bm{H}^{-1} \bm{r}_t} \leq B$ for all $t = 1, \dots, T$ where $\bm{r}_t = (r_t(a_1), \dots, r_t(a_N))$.
\end{assumption}

\subsection{OPNN and ADA-OPNN}

Unlike OPKB that uses a fixed feature mapping determined by a prespecified kernel, OPNN (Algorithm~\ref{alg:opnn}) uses the feature mapping induced by a neural network trained using past history.
For the initial block, OPNN uses the feature mapping induced by the initial weight $\bm{W}^{(0)}$.
Before starting a new block, OPNN trains the neural network with all past history using the procedure $\textsc{TrainNN}$ (Algorithm~\ref{alg:train-nn}) and recomputes the feature mapping using the newly trained weight.
The $\textsc{TrainNN}$ algorithm takes in training history and perform $J$ steps of gradient descent on the squared error loss regularized by L2 distance of the weight $\bm{W}$ from the initial weight $\bm{W}^{(0)}$.
Rest of the algorithm is the same as OPKB.

To adapt to non-stationarity, ADA-OPNN equips OPNN with change detection just as ADA-OPKB does with OPKB.
See Appendix~\ref{section:algorithm-detail} for the full algorithm of ADA-OPNN.

\begin{algorithm}
\KwInput{training history $\{(x_t, y_t)\}_{t \in \mathcal{I}}$, regularization parameter $\lambda$, step size $\eta$, number of gradient descent steps $J$, network width $m$, initial parameter $\bm{W}^{(0)}$}
Define $\mathcal{L}(\bm{W}) = \sum_{t \in \mathcal{I}} (f(x_t; \bm{W}) - y_t)^2 / 2 + m \lambda \Vert \bm{W} - \bm{W}^{(0)} \Vert_2^2 / 2$. \\
\For{$j = 0, \dots, J - 1$}{
  $\bm{W}^{(j + 1)} \leftarrow \bm{W}^{(j)} - \eta \nabla \mathcal{L}(\bm{W}^{(j)})$.
}
\KwReturn{$\bm{W}^{(J)}$.}
\caption{$\textsc{TrainNN}$: train neural network}
\label{alg:train-nn}
\end{algorithm}

\subsection{Analysis of OPNN and ADA-OPNN}

\textcite{jacot_neural_2018} show that the neural tangent kernel stays constant during training in the infinite network width limit.
Hence, in the infinite width limit, OPNN and ADA-OPNN are equivalent to OPKB and ADA-OPKB respectively that use the feature mapping corresponding to the kernel $H$.
We can expect that in the finite width regime, the regret bound for OPNN and ADA-OPNN are the same as that for OPKB and ADA-OPKB respectively as long as the network width is large enough.
Theorem~\ref{theorem:ada-opnn} and Theorem~\ref{theorem:opnn} confirm this.
See Appendix \ref{section:analysis-of-opnn} for the full proof.

\begin{remark}
The current NTK theory limits us to work in the infinite width regime where the feature mapping remains fixed.
However, we empirically show in Appendix~\ref{section:additional-experiments} that using the dynamic feature mapping induced by a finite width neural network is beneficial.
This finding is consistent with numerous empirical results demonstrated by \textcite{fort_deep_2020, lee_finite_2020} in the supervised learning setting.
We leave the analysis beyond the infinite width regime as future work.
\end{remark}

\begin{theorem}[Informal]\label{theorem:ada-opnn}
Under Assumption~\ref{assumption:invertibility} and Assumption~\ref{assumption:nn-reward}, the ADA-OPNN algorithm using a neural network of sufficiently large width achieves a dynamic regret bound of
$$
\begin{aligned}
\widetilde{\mathcal{O}}(\min\{&\sqrt{\gamma_T L_T T \log N}, \\
&(\gamma_T V_T \log N)^{1/3} T^{2/3} + \sqrt{\gamma_T T \log N} \})
\end{aligned}
$$
with high probability, where $\gamma_T$ is the maximum information gain corresponding to the neural tangent kernel $H$ of the neural network used in the algorithm.
\end{theorem}

\paragraph{Relation to Previous Work}

Our regret bound of ADA-OPNN becomes $\widetilde{\mathcal{O}}(\sqrt{\gamma_T T \log N})$ when adapted to the stationary setting, which is an improvement over previous work \parencite{zhou_neural_2020, gu_batched_2021, jia_learning_2022} by a factor of $\sqrt{\gamma_T}$ and is comparable to work by \textcite{kassraie_neural_2021}.

\section{Experiments}\label{section:experiment}

The most notable feature of our algorithms is that they can adapt to non-stationarity without prior knowledge of the degree of non-stationarity.
In this section, we illustrate this feature by comparing to previous work SW-GPUCB~\parencite{zhou_no-regret_2021} and WGPUCB~\parencite{deng_weighted_2022}, both of which require the knowledge of the degree of non-stationarity to tune parameters.
For the parameter tuning and the experiments, we used an internal cluster of nodes with 20-core 2.40 GHz CPU and Tesla V100 GPU.
The total amount of computing time was around 300 hours.

\paragraph{Experiment Design}
We run all algorithms in two environments: an environment with a single switch and the other with two switches.
We first tune the algorithms for the first environment.
Then, we run the tuned algorithms on the second environment to see how the algorithms adapt to the new non-stationarity.

\paragraph{Environments}

We run all simulations for $T = 10000$ rounds.
For each simulation, we randomly sample an action set of size $N = 100$ from the unit sphere in $\mathbb{R}^d$.
We follow \textcite{chowdhury_kernelized_2017} and sample the reward vector $\{ r(x) \}_{x \in \mathcal{X}}$ from the multivariate normal distribution $\mathcal{N}(0, \bm{K})$ where $\bm{K} = \{ k(x, x') \}_{x, x' \in \mathcal{X}}$ and $k$ is the radial basis function kernel with length scale $0.2$.
We scale the reward vector so that the maximum absolute reward is 0.8, 
We sample the noises $\eta_t$ from $\mathcal{N}(0, 0.1^2)$.
We run experiments on two environments: the first environment has a single switch at time 3000 and the second environment has switches at time 1500 and 5000.

\paragraph{Algorithm Tuning}

We tune SW-GPUCB, WGPUCB, ADA-OPKB, ADA-OPNN on the first environment with a single switch.
For SW-GPUCB, we do a grid search for $\lambda$ over the range $\{0.01, 0.02, 0.05, 0.1, \dots, 100\}$, the UCB scale parameter $v$ over $[0.001, 1]$, and the window size over $\{100, 200, 500, 1000, \dots, 10000\}$.
See Algorithm~\ref{alg:igp-ucb} for the definition of $\lambda$.
For WGPUCB, we do a grid search for $\lambda$ over the range $\{0.01, 0.02, 0.05, 0.1, \dots, 100\}$, the UCB scale parameter over $\{0.001, 0.002, 0.005, 0.01, \dots, 1\}$, and the discounting factor over $\{0.99, 0.995, 0.999, 0.9995, 0.9999 \}$.
See Algorithm~\ref{alg:igp-ucb} for the definition of $\lambda$.
For ADA-OPKB and ADA-OPNN, we do a grid search for $\sigma$ over $\{1, 2, 5, 10, 20, 50, 100, 200, 500, 1000\}$ and $c_0, c_1, c_2, c_3, c_4$ over $\{0.001, 0.002, 0.005, 0.01, \dots, 100\}$.
For ADA-OPNN, we do a grid search for the learning rate $\eta$ over $\{10^{-9}, 10^{-8}, 10^{-7} \}$, training steps $J$ over $\{100, 1000, 10000\}$ and regularization parameter $\lambda$ over $\{1, 10, 100, 1000\}$.
We use a neural network of depth $L = 3$ and width $m = 2048$.

\begin{remark} \label{remark:limitation}
Compared to SW-GPUCB and WGPUCB, ADA-OPKB and ADA-OPNN have many parameters to tune.
We leave designing a simpler algorithm with less parameters that does not require the knowledge of non-stationarity as future work.
\end{remark}

\paragraph{Results}

The cumulative regrets of SW-GPUCB, WGPUCB, ADA-OPKB and ADA-OPNN averaged over 25 random seeds are shown in Figure~\ref{figure:experiment}.
Error bars indicate standard errors of the means.
Plot (a) shows the performances of the algorithms tuned under the first environment (a single switch).
We remark that SW-GPUCB outperforms ADA-OPKB and ADA-OPNN in the initial stationary interval because ADA-OPKB and ADA-OPNN have overhead of running change detections.
We conjecture that ADA-OPNN performs worse than ADA-OPKB due to kernel mismatch: ADA-OPKB uses the kernel used by the nature for drawing reward functions while ADA-OPNN does not.

Plot (b) shows the performances of the algorithms on the second environment (switches at time 1500 and 5000).
SW-GPUCB optimally tuned for the single switch environment (window size 3000), performs worse than ADA-OPKB and ADA-OPNN in the new environment.
WGPUCB optimally tuned for the single switch environment (discounting factor of 0.9995) performs similarly to ADA-OPNN but is outperformed by ADA-OPKB.
This experiment highlights the fact that ADA-OPKB and ADA-OPNN can adapt to new non-stationarity better than SW-GPUCB and WGPUCB.

For an experiment that demonstrates the benefit of dynamically updating feature mapping for OPNN, and an experiment under a slowly varying environment, see Appendix~\ref{section:additional-experiments}.

\begin{figure}
    \centering
    \begin{subfigure}[b]{0.475\textwidth}
        \centering
        \includegraphics[width=\textwidth, height=1.2in]{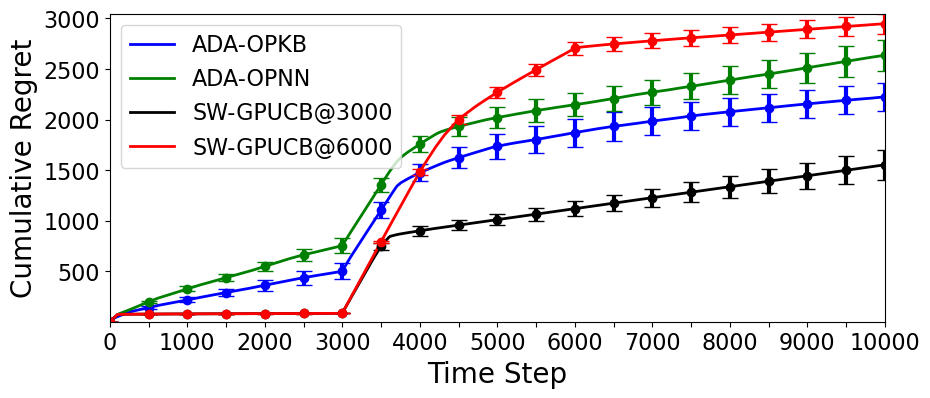}
        \caption[a]%
        {{\small Environment 1: single switch}}    
    \end{subfigure}
    \hfill
    \begin{subfigure}[b]{0.475\textwidth}  
        \centering 
        \includegraphics[width=\textwidth, height=1.2in]{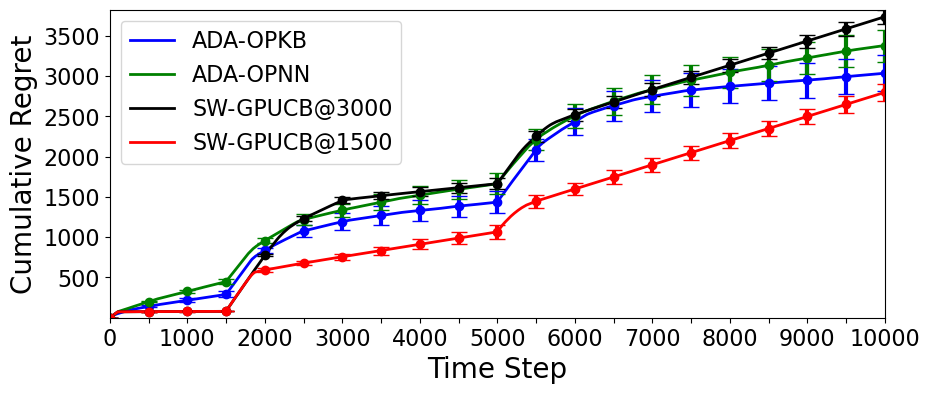}
        \caption[]%
        {{\small Environment 2: two switches}}    
    \end{subfigure}
    \caption{Cumulative regret comparison of algorithms in non-stationary environments}
    \label{figure:experiment}
\end{figure}

\section{Conclusion}

In this paper, we propose an algorithm for non-stationary kernel bandits that does not require the knowledge of non-stationary budgets, and show a simultaneous dynamic regret bound in terms of the budgets on the total variation and the number of changes in reward functions.
The dynamic regret bound is tighter than previous work on the non-stationary kernel bandit setting.
Also, our algorithm is nearly minimax optimal in the non-stationary linear bandit setting when run with a linear kernel.
We provide an extension of our algorithm using a neural network.
An interesting future work would be to adapt to a new non-stationary measure that tracks the number of times the identity of the best arm changes, which is a smaller measure than the number of changes in the reward functions.
We believe the reward estimate based change detection algorithm and its analysis in this paper is suitable for this extension.

\printbibliography

\newpage

\appendix

\section{Notation Table}

\begin{table*}[h]
\centering
\begin{tabular}{lll} 
 \toprule
 Notation & Definition & Explanation \\
 \midrule
 $S_\varphi(Q, \lambda)$ & $\sum_{x \in \mathcal{X}} Q(x) \varphi(x) \varphi(x)^T + \lambda I$ &  \\
 $\gamma_{\varphi, T}$ & $\max_{P \in \mathcal{P}_{\mathcal{X}}} \log \det S_\varphi(TP / \sigma, 1)$ & Information gain with respect to $\varphi$ \\
 $V_{[s, t]}$ & $\sum_{\tau = s}^{t - 1} \Vert r_{\tau + 1} - r_\tau \Vert_\infty$ & Total variation in interval $[s, t]$ \\
 $L_{[s, t]}$ & $\sum_{\tau = s}^{t - 1} \mathbb{I}\{r_{\tau + 1} \neq r_\tau\}$ & Number of arm switches in $[s, t]$ \\
 $\pi_\varphi(\mathcal{A})$ & $\argmax_{P \in \mathcal{P}_{\mathcal{A}}} \log \det S_\varphi(P, \sigma / T)$ &  Optimal design on $\mathcal{A}$ with respect to $\varphi$\\
 $\mathcal{R}_\mathcal{I}(x)$ & $\frac{1}{\vert \mathcal{I} \vert} \sum_{t \in \mathcal{I}} r_t(x)$ & Average reward of arm $x$ over interval $\mathcal{I}$ \\
 $\Delta_t(x)$ & $\max_{x' \in \mathcal{X}} r_t(x') - r_t(x)$ & Optimality gap of $x$ at time $t$ \\
 $\Delta_\mathcal{I}(x)$ & $\max_{x' \in \mathcal{X}} \mathcal{R}_\mathcal{I}(x') - \mathcal{R}_\mathcal{I}(x)$ & Average optimality gap over the interval $\mathcal{I}$ \\
 $\widehat{\mathcal{R}}_{\varphi, t}(x)$ & $\varphi(x)^T S_\varphi(P_t, \sigma / T)^{-1} \varphi(x_t) y_t$ & IPS estimator for $r_t(x)$ with respect to $\varphi$ \\
 $\widehat{\mathcal{R}}_{\varphi, \mathcal{I}} (x)$ & $\frac{1}{\vert \mathcal{I} \vert} \sum_{t \in \mathcal{I}} \widehat{\mathcal{R}}_{\varphi, t}(x)$ & IPS estimator for average reward of $x$ over the interval $\mathcal{I}$ \\
 $\widehat{\Delta}_{\varphi, \mathcal{I}}(x)$ & $\max_{x' \in \mathcal{X}} \widehat{\mathcal{R}}_{\varphi, \mathcal{I}}(x') - \widehat{\mathcal{R}}_{\varphi, \mathcal{I}}(x)$ & Estimated optimality gap of arm $x$ over the interval $\mathcal{I}$ \\
 \bottomrule
\end{tabular}
\end{table*}

\section{Omitted algorithms}\label{section:algorithm-detail}

The ADA-OPNN algorithm adapts the OPNN algorithm to the non-stationary environment by equipping change detection.

\begin{algorithm}
\KwInput{network width $m$, network depth $L$, time horizon $T$, confidence level $\delta \in (0, 1)$.}
\KwDef{$\mu_j = c_1 2^{-j / 2}$, $\beta_j = c_2 \gamma_{\varphi, T} 2^{j / 2}$, $E = \lceil c_3 \gamma_{\varphi, T} \log(C_1 N / \delta) \rceil$, $\alpha =  c_4 \sigma / \log(C_1 N / \delta)$}
\KwInit{time step $t \leftarrow 1$, epoch index $i \leftarrow 1$, initial strategy $Q^{(0)} \leftarrow \pi_\varphi(\mathcal{X})$}
\For{$j = 0, 1, \dots$ \label{alg:line:first-ada-opnn} }{
  Set block $\mathcal{B}(j) \leftarrow [t, t + 2^j  E - 1]$ and cumulative block $\mathcal{C}(j) \leftarrow \cup_{k=0}^j \mathcal{B}(k)$. \\
  \If{$j \geq 1$}{
    $\bm{W}^{(j)} \leftarrow \textsc{TrainNN}(\{(x_\tau, y_\tau)\}_{\tau \in \mathcal{C}(j - 1)}, \bm{W}^{(0)})$ \\
    Find a feature mapping $\varphi^{(j)}$ equivalent to $g(\cdot; \bm{W}^{(j)}) / \sqrt{m}$ \\
    Compute the empirical gap $\widehat{\Delta} \leftarrow \{ \widehat{\Delta}_{\varphi^{(j)}, \mathcal{C}(j - 1)}(x) \}_{x \in \mathcal{X}}$ using all past history in epoch $i$. \\
    Find strategy $Q^{(j)} \leftarrow \textsc{OP}(\varphi^{(j)}, \widehat{\Delta}, \alpha, \beta_j, T)$; Set $P^{(j)} \leftarrow (1 - \mu_j) Q^{(m_t)} + \mu_j \pi_{\varphi^{(j)}}(\mathcal{X})$. \\
  }
  Generate replay schedule $\mathcal{S} \leftarrow \textsc{Schedule}(t, j)$. \\
  \While{$t \in \mathcal{B}(j)$}{
    $m_t \leftarrow \min \{ m: (m, \mathcal{I}) \in \mathcal{S}\,\, \text{with}\,\,t \in \mathcal{I} \}$ \tcp*{smallest index of scheduled intervals}
    Record $P_t \leftarrow P^{(m_t)}$. Play $x_t \sim P_t$ and receive reward $y_t$; Increment $t \leftarrow t + 1$. \\
    If \textbf{Test} with $\varphi = \varphi^{(j)}$ triggers a restart then increment $i$ and go to Line~\ref{alg:line:first-ada-opnn}.
  }
}
\KwTest{Trigger a restart if for any $(m, \mathcal{I}) \in \mathcal{S}$ with $\mathcal{I}$ ending at $t$ and $k < j$, the following holds
$$
\widehat{\Delta}_{\varphi, \mathcal{I}}(x) - 4 \widehat{\Delta}_{\varphi, \mathcal{C}(k)}(x) > 4 c_0 \mu_{m \wedge k} \quad \text{or} \quad \widehat{\Delta}_{\varphi, \mathcal{C}(k)}(x) - 4 \widehat{\Delta}_{\varphi, \mathcal{I}}(x) > 4 c_0 \mu_{m \wedge k}.
$$
}
\caption{$\textsc{ADA-OPNN}$: ADAptive Optimization Problem based algorithm using Neural Network}
\label{alg:ada-opnn}
\end{algorithm}

\section{Maximum information gain}\label{section:maximum-information-gain}

In this section, we summarize the properties of the maximum information gain used in this paper.
The original definition of the maximum information gain by \textcite{srinivas_gaussian_2010} is
$$
\bar{\gamma}_T = \max_{x_1, \dots, x_T \in \mathcal{X}} \frac{1}{2} \log \det (\sigma^{-1} K_T + I_T)
$$
where $K_T = [ k(x_i, x_j) ]_{i, j \in [T]}$ and $I_T \in \mathbb{R}^{T \times T}$ is the identity matrix.
For ease of exposition, we drop the factor $\frac{1}{2}$ that appears in the original definition of $\bar\gamma_T$.
In this paper, we define the \textit{continuous} version of the maximum information gain $\gamma_T$ as follows
$$
\gamma_{\varphi, T} = \max_{P \in \mathcal{P}_\mathcal{X}} \log \det S_\varphi( \sigma^{-1} TP, I_N)
$$
where $\varphi : \mathcal{X} \rightarrow \mathbb{R}^N$ is a feature mapping corresponding to the kernel $k$ such that $k(x, x') = \langle \varphi(x), \varphi(x') \rangle$.
To see the connection of $\gamma_{\varphi, T}$ to the original definition $\bar\gamma_T$, note that $K_T = H K H^T$ where $K = [ k(a_i, a_j) ]_{i, j \in [N]}$ and $H \in \{0, 1\}^{T \times N}$ with $H_{ti} = \mathbb{I}\{ x_t = a_i \}$ is the history matrix that indicates whether the action $a_i$ is played at time $t$ for $t \in [T]$ and $i \in [N]$.
Using the notation $\Phi = [ \varphi(a_1) \cdots \varphi(a_N) ]^T \in \mathbb{R}^{N \times N}$ such that $K = \Phi \Phi^T$, we have by the Sylvester's determinant identity $\det (I + AB) = \det(I + BA)$ that
$$
\begin{aligned}
\log \det (\sigma^{-1} K_T + I_T) &= \log \det (\sigma^{-1} H \Phi \Phi^T H^T + I_T) \\
&= \log \det (\sigma^{-1} \Phi^T H^T H \Phi  + I_N) \\
&= \log \det (\sigma^{-1} \Phi^T D_N \Phi + I_N) \\
&= \log \det S_\varphi( \sigma^{-1} T P_N, I_N)
\end{aligned}
$$
where $D_N = H^T H = \text{diag}(n_1, \dots, n_N)$ with $n_i$ denoting how many times $a_i$ appears in the sequence $x_1, \dots, x_T$ and $P_N = D_N / T$ is the relative frequency of the actions.
Hence,
$$
\bar\gamma_T = \max_{P \in \mathcal{P}_{T, \mathcal{X}}} \frac{1}{2} \log \det S_\varphi( \sigma^{-1} T P, I_N)
$$
where the maximization is over $\mathcal{P}_{T, \mathcal{X}} \coloneqq \{ P \in \mathcal{P}_\mathcal{X} : P(a_i) = n_i / T \,\, \text{for all} \,\, i \in [N] \,\, \text{with} \,\, n_i \in \mathbb{Z} \}$.
It follows that our definition $\gamma_{\varphi, T}$ is a continuous version of the maximum information gain in the sense that it maximizes over $\mathcal{P}_\mathcal{X}$ instead of the discretized probability space $\mathcal{P}_{T, \mathcal{X}}$.

A direct consequence is that $\gamma_{\varphi, T} \geq \bar\gamma_T$.
To get an upper bound on $\gamma_{\varphi, T}$ we can use Theorem 3 in \textcite{vakili_information_2021} that shows an upper bound of $\bar\gamma_T$ in terms of the eigendecay of the kernel $k(\cdot, \cdot)$.
It can be seen that their proof can be easily adapted to the continuous version, which leads to upper bounds for common kernels in the following lemma.

\begin{lemma}[Theorem 3 in \textcite{vakili_information_2021}]
For the Mat\'ern-$\nu$ kernel and the SE kernel, the maximum information gain is upper bounded by
$$
\begin{aligned}
\gamma_{\varphi, T} &= \mathcal{O}\left(
T^{\frac{d}{2 \nu + d}} \log^{\frac{2\nu}{2 \nu + d}} (T)
\right), \quad \text{for Mat\'ern-}\nu\,\,\text{kernel} \\
\gamma_{\varphi, T} &= \mathcal{O}\left(
\log^{d+1}(T)
\right), \quad \text{for SE kernel}.
\end{aligned}
$$
\end{lemma}

Similarly, adapting the proof of Theorem 2 in \textcite{vakili_uniform_2021}, we get an upper bound on the maximum information gain for the neural tangent kernel of a ReLU network as follows.

\begin{lemma}
For the neural tangent kernel of a ReLU network, the maximum information gain is upper bounded by
$$
\gamma_{\varphi, T} = \mathcal{O}\left(
T^{\frac{d - 1}{d}} \log^{\frac{1}{d}}(T)
\right).
$$
\end{lemma}

For the linear kernel, we get the following upper bound on the maximum information gain.

\begin{lemma}\label{lemma:information-gain-linear}
For the identity feature mapping $\varphi(x) = x$ for all $x \in \mathcal{X} \subset \mathbb{R}^d$ corresponding to the linear kernel $k(x, x') = \langle x, x' \rangle$, we have $\gamma_{\varphi, T} \leq \mathcal{O}(d \log T)$.
\end{lemma}
\begin{proof}
Using the identity $\det(A) \leq (\Tr(A) / d)^d$ for a positive semi-definite matrix $A \in \mathbb{R}^{d \times d}$, which can be seen by the AM-GM inequality on the eigenvalues of $A$, we have
$$
\begin{aligned}
\log \det S_\varphi(\sigma^{-1} TP, 1)
&\leq 
d \log (\Tr (S_\varphi(\sigma^{-1} TP, 1)) / d) \\
&=
d \log \left(\frac{1}{d} \Tr \left( \frac{T}{\sigma} \sum_{x \in \mathcal{X}} P(x) x x^T + I_d \right) \right) \\
&=
d \log \left( \frac{1}{d}\left(\frac{T}{\sigma} \sum_{x \in \mathcal{X}} P(x) \Vert x \Vert_2^2 + d \right) \right) \\
&\leq
d \log \left(
\frac{T}{\sigma d} + 1
\right)
= \mathcal{O}(d \log T)
\end{aligned}
$$
where the second inequality follows by the assumption that $\Vert x \Vert_2 \leq 1$.
Taking the maximum over $P \in \mathcal{P}_\mathcal{X}$ completes the proof.
\end{proof}

\begin{lemma}\label{lemma:variance-bounded-by-information-gain}
For any feature mapping $\varphi$ and any $P \in \mathcal{P}_\mathcal{X}$, we have
$$
\sum_{x \in \mathcal{X}} P(x) \Vert \varphi(x) \Vert^2_{S_\varphi(P, \sigma / T)^{-1}} \leq \gamma_{\varphi, T}.
$$
\end{lemma}

\begin{proof}
We can rewrite the left hand side as
$$
\begin{aligned}
\sum_{x \in \mathcal{X}} P(x) \Vert \varphi(x) \Vert^2_{S_\varphi(P, \sigma / T)^{-1}}
&=
\Tr \left(
\sum_{x \in \mathcal{X}} S_\varphi(P, \sigma / T)^{-1} P(x) \varphi(x) \varphi(x)^T
\right) \\
&=
\Tr \left(
S_\varphi(P, \sigma / T)^{-1} (S_\varphi(P, \sigma / T) - (\sigma / T) I_N)
\right) \\
&\leq
\log \det S_\varphi(P, \sigma / T) - \log \det (\sigma / T) I_N \\
&=
\log \det S_\varphi((T / \sigma) P, 1)
\leq \gamma_{\varphi, T}
\end{aligned}
$$
where the first inequality uses the identity $\Tr(A^{-1}(A - B)) \leq \log \det A - \log \det B$ for $A \succcurlyeq B \succcurlyeq 0$ (Lemma 12 in \textcite{hazan_logarithmic_2007}).
\end{proof}

\section{Analysis of OPKB}\label{section:analysis-opkb}

In this section, we prove the high probability dynamic regret bound of the OPKB algorithm under the stationary kernel bandit setting stated below.

\subsection{Constants and notations} \label{subsection:constants}

We use the following parameters in this section (and in Section~\ref{section:analysis-ada-opkb}) for ease of exposition of the proof: $c_0 = 40 + 16 \sqrt{\alpha}$, $c_1 = \frac{1}{2}$, $c_2 = \frac{1}{10 + 4 \sqrt{\alpha}}$, $c_3 = 4$, $c_4 = \frac{1}{4}$ so that $\mu_j = 2^{-(j + 2) / 2}$, $\beta_j = \frac{\gamma_{\varphi, T}}{10 + 4 \sqrt{\alpha}} 2^{j / 2}$, $C_0 = 8 T \log_2 T$, $E = \lceil 4 \gamma_{\varphi, T} \log(C_0 N / \delta) \rceil$, $\alpha = \sigma / (4 \log(C_0 N / \delta))$.
We define $\xi_j = \frac{\mu_j}{4 \gamma_{\varphi, T}}$.
We frequently use the identities
$$
c_0 \beta_j \mu_j = 2 \gamma_{\varphi, T}, \quad
\xi_j \gamma_{\varphi, T} = \frac{\mu_j}{4}, \quad
\mu_j \beta_j = \frac{2 \gamma_{\varphi, T}}{c_0} \leq \frac{\gamma_{\varphi, T}}{20}, \quad
\xi_j \beta_j = \frac{1}{2 c_0} \leq \frac{1}{80}.
$$
We denote by $\mathcal{R}_\mathcal{I}(x) = \frac{1}{\vert \mathcal{I} \vert} \sum_{t \in \mathcal{I}} r_t(x)$ the average reward of action $x$ in interval $\mathcal{I}$.
We define $\Delta_{\mathcal{I}}(x) = \max_{x' \in \mathcal{X}} \mathcal{R}_\mathcal{I}(x') - \mathcal{R}_\mathcal{I}(x)$.

\subsection{Proof of Lemma~\ref{lemma:optimal-design}} \label{section:proof-optimal-design}

\begin{refproof}[Lemma~\ref{lemma:optimal-design}]
For ease of exposition, we write $\pi^\star = \pi_\varphi(\mathcal{A})$.
Recall that the optimal design $\pi^\star$ is a maximizer of $\log \det S_\varphi(P, \sigma / T)$ subject to $P(x) \geq 0$ for all $x \in \mathcal{A}$ and $\sum_{x \in \mathcal{A}} P(x) = 1$.
Introducing Lagrange multipliers $\lambda_x$ for the conditions $P(x) \geq 0$ for all $x \in \mathcal{X}$ and $\lambda$ for $\sum_{x \in \mathcal{A}} P(x) = 1$, the KKT optimality conditions give
\begin{align}
\tag{Stationarity}
\Vert \varphi(x) \Vert^2_{S_\varphi(\pi^\star, \sigma / T)^{-1}}
+
\lambda_x - \lambda &= 0, \quad \text{for all }\, x \in \mathcal{A} \\
\tag{Dual feasibility}
\lambda_x &\geq 0, \quad \text{for all }\, x \in \mathcal{A} \\
\tag{Complementary slackness}
\pi^\star(x) \lambda_x &= 0, \quad \text{for all }\, x \in \mathcal{A}
\end{align}
where we use the fact that $\frac{\partial}{\partial P(x)} \log \det S_\varphi(P, \sigma / T) = \Vert \varphi(x) \Vert^2_{S_\varphi(P, \sigma / T)^{-1}}$.
Multiplying $\pi^\star(x)$ to the stationarity condition and summing over $x \in \mathcal{A}$, we get
$$
\begin{aligned}
0
&= \sum_{x \in \mathcal{A}} \pi^\star(x) \Vert \varphi(x) \Vert^2_{S_\varphi(\pi^\star, \sigma / T)^{-1}}
+ \sum_{x \in \mathcal{A}} \pi^\star(x) \lambda_x
- \lambda \sum_{x \in \mathcal{A}} \pi^\star(x) \\
&=
\sum_{x \in \mathcal{A}} \pi^\star(x) \Vert \varphi(x) \Vert^2_{S_\varphi(\pi^\star, \sigma / T)^{-1}}
- \lambda
\end{aligned}
$$
where the second equality uses the complementary slackness conditions.
Hence,
$$
\lambda =
\sum_{x \in \mathcal{A}} \pi^\star(x) \Vert \varphi(x) \Vert^2_{S_\varphi(\pi^\star, \sigma / T)^{-1}}
\leq
\max_{P \in \mathcal{P}_\mathcal{X}}
\sum_{x \in \mathcal{X}} P(x) \Vert \varphi(x) \Vert^2_{S_\varphi(P, \sigma / T)^{-1}}
= \gamma_{\varphi, T}.
$$
Using this result $\lambda \leq \gamma_{\varphi, T}$ to the stationarity conditions and using the dual feasibility conditions $\lambda_x \geq 0$, we get $\Vert \varphi(x) \Vert^2_{S_\varphi(\pi^\star, \sigma / T)^{-1}} = \lambda - \lambda_x \leq \lambda \leq \gamma_{\varphi, T}$ for all $x \in \mathcal{X}$ as desired.
For the proof of $\Var(\widehat{\mathcal{R}}_{\varphi, t}(x)) \leq \Vert \varphi(x) \Vert^2_{S_\varphi(\pi_\varphi(\mathcal{A}), \sigma / T)^{-1}}$, refer to the proof of Lemma~\ref{lemma:one-point-estimator}.
\end{refproof}

\subsection{Proof of Lemma~\ref{lemma:optimization-problem}} \label{section:proof-optimization-problem}

\begin{refproof}[Lemma~\ref{lemma:optimization-problem}]
Recall that the strategy returned by the algorithm $OP(\varphi, \widehat{\Delta}, \alpha, \beta, T)$ is $Q = \frac{1}{2} P^\star + \frac{1}{2} \pi^\star$ where $P^\star$ is the minimizer of $J(P) = \sum_{x \in \mathcal{X}} P(x) \widehat{\Delta}(x) - \frac{2}{\beta} \log \det S_\varphi(P, \sigma / T)$ among $\mathcal{P}_\mathcal{X}$ and we write $\pi^\star = \pi_\varphi(\mathcal{A})$ where $\mathcal{A} = \{ x \in \mathcal{X} : \widehat\Delta(x) \leq 2 \alpha \gamma_{\varphi, T} / \beta \}$.
Since the empirical gap estimates satisfy $\widehat{\Delta}(x) \geq 0$ for all $x \in \mathcal{X}$ and there exists $\hat{x} \in \mathcal{X}$ with $\widehat{\Delta}(\hat{x}) = 0$, we can check that $P^\star$ is also a minimizer among the set of sub-distributions $\widetilde{\mathcal{P}}_\mathcal{X} = \{ P \in \mathbb{R}^\mathcal{X} : P(x) \geq 0 \,\, \text{for all} \,\, x \in \mathcal{X}, \sum_{x \in \mathcal{X}} P(x) \leq 1 \}$.
This can be seen by noting that for any sub-distribution $\widetilde{P}$, the proper distribution $P$ obtained by increasing the weight of the empirically best action $\hat{x}$ satisfies $J(\widetilde{P}) \geq J(P)$.
Introducing Lagrange multipliers $\lambda_x$ for the conditions $P(x) \geq 0$ for all $x \in \mathcal{X}$ and $\lambda$ for $\sum_{x \in \mathcal{X}} P(x) \leq 1$, the KKT optimality conditions give
\begin{align}
\tag{Stationarity}
\widehat{\Delta}(x) - \frac{2}{\beta} \Vert \varphi(x) \Vert^2_{S_\varphi(P^\star, \sigma / T)^{-1}}
- \lambda_x + \lambda &= 0, \quad \text{for all }\, x \in \mathcal{X} \\
\tag{Dual feasibility}
\lambda_x &\geq 0, \quad \text{for all }\, x \in \mathcal{X} \\
\tag{Dual feasibility}
\lambda &\geq 0 \\
\tag{Complementary slackness}
P^\star(x) \lambda_x &= 0, \quad \text{for all }\, x \in \mathcal{X}.
\end{align}
Multiplying $P^\star(x)$ to the stationarity conditions and summing over $x \in \mathcal{X}$, we get
\begin{align}
  0 &=
  \sum_{x \in \mathcal{X}} P^\star(x) \widehat{\Delta}(x)
  - \frac{2}{\beta} \sum_{x \in \mathcal{X}}
  P^\star(x) \Vert \varphi(x) \Vert^2_{S_\varphi(P^\star, \sigma / T)^{-1}}
  - \sum_{x \in \mathcal{X}} P^\star(x) \lambda_x
  + \lambda \sum_{x \in \mathcal{X}} P^\star(x) \nonumber \\
  &=
  \sum_{x \in \mathcal{X}} P^\star(x) \widehat{\Delta}(x)
  - \frac{2}{\beta} \sum_{x \in \mathcal{X}}
  P^\star(x) \Vert \varphi(x) \Vert^2_{S_\varphi(P^\star, \sigma / T)^{-1}}
  + \lambda \label{eqn:opkb-stationarity-condition}
\end{align}
where the second equality uses the complementary slackness conditions.
Rearranging and using the dual feasibility condition $\lambda \geq 0$, we get
$$
\sum_{x \in \mathcal{X}} P^\star(x) \widehat{\Delta}(x)
= \frac{2}{\beta} \sum_{x \in \mathcal{X}} P^\star(x) \Vert \varphi(x) \Vert^2_{S_\varphi(P^\star, \sigma / T)^{-1}} - \lambda \\
\leq \frac{2 \gamma_{\varphi, T}}{\beta}.
$$
It follows that $Q = \frac{1}{2} P^\star + \frac{1}{2} \pi^\star$ satisfies
$$
\begin{aligned}
\sum_{x \in \mathcal{X}} Q(x) \widehat{\Delta}(x)
&= \frac{1}{2} \sum_{x \in \mathcal{X}} P^\star(x) \widehat{\Delta}(x)
+ \frac{1}{2} \sum_{x \in \mathcal{A}} \pi^\star(x) \widehat{\Delta}(x) \\
&\leq \frac{1}{2} \frac{2 \gamma_{\varphi, T}}{\beta} + \frac{1}{2} \frac{2 \alpha \gamma_{\varphi, T}}{\beta}
= \frac{(1 + \alpha)\gamma_{\varphi, T}}{\beta}
\end{aligned}
$$
where the inequality uses the fact that $\widehat{\Delta}(x) \leq 2 \alpha \gamma_{\varphi, T} / \beta$ for $x \in \mathcal{A}$ by the definition of $\mathcal{A}$.
This proves the first inequality of the lemma.
Also, since the empirical gaps satisfy $\widehat{\Delta}(x) \geq 0$ for all $x \in \mathcal{X}$, rearranging (\ref{eqn:opkb-stationarity-condition}) gives
$$
\lambda = \frac{2}{\beta} \sum_{x \in \mathcal{X}} P^\star(x) \Vert \varphi(x) \Vert^2_{S_\varphi(P^\star, \sigma / T)^{-1}}
- \sum_{x \in \mathcal{X}} P^\star(x) \widehat{\Delta}(x)
\leq \frac{2 \gamma_{\varphi, T}}{\beta}.
$$
Hence, by the stationarity condition, we have for each $x \in \mathcal{X}$ that
$$
\Vert \varphi(x) \Vert^2_{S_\varphi(P^\star, \sigma / T)^{-1}}
= \frac{\beta \widehat{\Delta}(x)}{2} - \frac{\beta \lambda_x}{2} + \frac{\beta \lambda}{2}
\leq \frac{\beta \widehat{\Delta}(x)}{2} + \gamma_{\varphi, T}
$$
where we use the dual feasibility condition $\lambda_x \geq 0$.
Using the fact that $S_\varphi(Q, \sigma / T) \succcurlyeq \frac{1}{2} S_\varphi(P^\star, \sigma / T)$ gives the second inequality of the lemma.
Finally, for the third inequality of the lemma, we argue for the cases $x \in \mathcal{A}$ and $x \notin \mathcal{A}$ separately.
If $x \in \mathcal{A}$, then using $S_\varphi(Q, \sigma / T) \succcurlyeq \frac{1}{2} S_\varphi(\pi^\star, \sigma / T)$, we get
$\Vert \varphi(x) \Vert^2_{S_\varphi(Q, \sigma / T)^{-1}} \leq 2 \Vert \varphi(x) \Vert^2_{S_\varphi(\pi^\star, \sigma / T)^{-1}} \leq 2 \gamma_{\varphi, T} \leq \frac{\beta^2 \widehat{\Delta}^2(x)}{2 \alpha \gamma_{\varphi, T}} + 2 \gamma_{\varphi, T}$.
If $x \notin \mathcal{A}$, then we have $\widehat{\Delta}(x) > 2 \alpha \gamma_{\varphi, T} / \beta$ by the definition of $\mathcal{A}$.
Hence, $1 < \frac{\beta \widehat{\Delta}(x)}{2 \alpha \gamma_{\varphi, T}}$ and the second inequality of the lemma gives $\Vert \varphi(x) \Vert^2_{S_\varphi(Q, \sigma / T)^{-1}} \leq \beta \widehat{\Delta}(x) + 2 \gamma_{\varphi, T} \leq \frac{\beta^2 \widehat{\Delta}^2(x)}{2 \alpha \gamma_{\varphi, T}} + 2 \gamma_{\varphi, T}$, as desired.
\end{refproof}

\subsection{Concentration bound for reward estimates} \label{section:concentration-reward-estimates}

In this subsection, we prove the following concentration bound for reward estimates.

\begin{lemma}\label{lemma:reward-estimate-concentration}
Let $\mathcal{I} \subseteq [1, T]$ be a time interval.
Let $m_t$ be the strategy index used by OPKB at time $t$.
Let $j$ be the maximum strategy index used in $\mathcal{I}$ such that $m_t \leq j$ for all $t \in \mathcal{I}$. Then, with probability at least $1 - \frac{2\delta}{C}$, we have
$$
\vert \widehat{\mathcal{R}}_{\varphi, \mathcal{I}}(x) - \mathcal{R}_\mathcal{I}(x) \vert
\leq
\frac{\xi_j}{\vert \mathcal{I} \vert} \sum_{t \in \mathcal{I}} \Vert \varphi(x) \Vert_{S_\varphi(P_t, \sigma / T)^{-1}}^2 + \frac{\log(CN/\delta)}{\xi_j \vert \mathcal{I} \vert}
+ \frac{\sqrt{\sigma / T}}{\vert \mathcal{I} \vert} \sum_{t \in \mathcal{I}} \Vert \varphi(x) \Vert_{S_\varphi(P_t, \sigma / T)^{-1}}
$$
for all $x \in \mathcal{X}$ where $\xi_j = \mu_j / (4 \gamma_{\varphi, T})$.
\end{lemma}

The proof relies on the following Freedman-style martingale inequality.
See Theorem~1 in \cite{beygelzimer_contextual_2011} for the proof of this inequality.

\begin{lemma}[Freedman]\label{lemma:freedman}
Let $X_1, \dots, X_n \in \mathbb{R}$ be a martingale difference sequence
with respect to a filtration $\mathcal{F}_0, \mathcal{F}_1, \dots$.
Assume $X_i \leq R$ a.s. for all $i$.
Then for any $\delta \in (0, 1)$ and
$\xi \in [0, 1 / R]$, we have with probability at least $1 - \delta$ that
$$
\sum_{i=1}^n X_i \leq \xi V + \frac{\log (1 / \delta)}{\xi},
$$
where $V = \sum_{i=1}^n \mathbb{E}[ X_i^2 \mid \mathcal{F}_{i-1} ]$.
\end{lemma}

To apply the Freedman inequality, we analyze the distribution of the IPS estimator $\widehat{\mathcal{R}}_{\varphi, t}(x)$ in the following lemma.

\begin{lemma}\label{lemma:one-point-estimator}
Suppose the reward function $r_t(\cdot)$ lies in a RKHS with a feature mapping $\psi : \mathcal{X} \rightarrow \ell^2$.
Let $\varphi : \mathcal{X} \rightarrow \mathbb{R}^N$ be a feature mapping equivalent to $\psi$.
Let $m_t$ be the strategy index used at time $t$ and $P_t = P^{(m_t)}$ be the strategy used at time $t$.
Then, the IPS estimator $\widehat{\mathcal{R}}_{\varphi, t}(x) = \varphi(x)^T S_\varphi(P_t, \sigma / T)^{-1} \varphi(x_t)^T y_t$ satisfies
$$
\begin{aligned}
\vert \widehat{\mathcal{R}}_{\varphi, t}(x) \vert &\leq \frac{\gamma_{\varphi, T}}{\mu_{m_t}} \\
\vert \mathbb{E}_t[\widehat{\mathcal{R}}_{\varphi, t}(x)] - r_{t}(x) \vert &\leq \sqrt{\sigma / T} \Vert \varphi(x) \Vert_{S_\varphi(P_t, \sigma / T)^{-1}} \\
\Var_t [ \widehat{\mathcal{R}}_{\varphi, t}(x)] &\leq \Vert \varphi(x) \Vert^2_{S_\varphi(P_t, \sigma / T)^{-1}}
\end{aligned}
$$
where $\mathbb{E}_t$ and $\Var_t$ are the conditional expectation and the conditional variance given the history before time $t$ respectively.
\end{lemma}

\begin{proof}
  The first claim follows by
  $$
    \vert \widehat{\mathcal{R}}_{\varphi, t}(x) \vert
    = \left\vert \varphi(x)^T S(P_t, \sigma / T)^{-1} \varphi(x_t) y_t \right\vert
    \leq
    \Vert \varphi(x) \Vert_{S(P_t, \sigma / T)^{-1}}
    \Vert \varphi(x_t) \Vert_{S(P_t, \sigma / T)^{-1}}
    \leq \frac{\gamma_{\varphi, T}}{\mu_{m_t}}
  $$
  where the first inequality uses the assumption
  $\vert y_t \vert \leq 1$ and the Cauchy-Schwarz inequality, and
  the second inequality uses
  $
  S_\varphi(P_t, \sigma / T)
  = (1 - \mu_{m_t}) S(Q^{(m_t)}, \sigma / T)
  + \mu_{m_t} S_\varphi(\pi_{\varphi}(\mathcal{X}), \sigma / T)
  \succcurlyeq
  \mu_{m_t} S(\pi_\varphi(\mathcal{X}), \sigma / T)
  $
  and Lemma~\ref{lemma:optimal-design}.

  To show the second claim, let $\theta_t \in \ell^2$ be the parameter such that $r_t(x) = \langle \psi(x), \theta_t \rangle$ for all $x \in \mathcal{X}$.
  Since $P_t$ is completely determined given history up to $t - 1$, we have
  $$
  \begin{aligned}
    \mathbb{E}_t[\widehat{\mathcal{R}}_{\varphi, t}(x)]
    &= \mathbb{E}_t[
      \psi(x)^T
        S_\psi(P_t, \sigma / T)^{-1}
      \psi(x_t)
      (\psi(x_t)^T \theta_t + \eta_t)
    ] \\
    &= \psi(x)^T S_\psi(P_t, \sigma / T)^{-1} \mathbb{E}_t[\psi(x_t) \psi(x_t)^T] \theta_t \\
    &= \psi(x)^T S_\psi(P_t, \sigma / T)^{-1} (S_\psi(P_t, \sigma / T) - (\sigma / T) I) \theta_t \\
    &= r_t(x) - (\sigma / T) \psi(x)^T S_\psi(P_t, \sigma / T)^{-1} \theta_t
  \end{aligned}
  $$
  where the first equality is by Lemma~\ref{lemma:useful-identity} and the third equality uses the fact that the strategy $P_t$ is deterministic given the history up to time $t$.
  The second claim follows by the bound
  $$
  \begin{aligned}
    (\sigma / T) \left\vert \psi(x)^T S_\psi(P_t, \sigma / T)^{-1} \theta_t \right\vert
    &\leq
    (\sigma / T)
    \left\Vert \psi(x) \right\Vert_{S_\psi(P_t, \sigma / T)^{-1}}
    \left\Vert \theta_t \right\Vert_{S_\psi(P_t, \sigma / T)^{-1}} \\
    &\leq
    \sqrt{\sigma / T} \left\Vert \varphi(x) \right\Vert_{S_\varphi(P_t, \sigma / T)^{-1}}
  \end{aligned}
  $$
  where the first inequality is by the Cauchy-Schwarz inequality and the last inequality uses $S_\psi(P_t, \sigma / T)^{-1} \preccurlyeq (T / \sigma) I$, the assumption that $\Vert \theta_t \Vert_2 \leq 1$ and Lemma~\ref{lemma:useful-identity}.

  Finally, the third claim follows by
  \begin{align*}
    \Var_t [ \widehat{\mathcal{R}}_{\varphi, t}(x)]
    &\leq \mathbb{E}_t [\{\varphi(x)^T S_\varphi(P_t, \sigma / T)^{-1} \varphi(x_t) \}^2 y_t^2 ] \\
    &\leq \varphi(x)^T S_\varphi(P_t, \sigma / T)^{-1} \mathbb{E}_t [\varphi(x_t) \varphi(x_t)^T ] S_\varphi(P_t, \sigma / T)^{-1} \varphi(x) \\
    &= \varphi(x)^T S_\varphi(P_t, \sigma / T)^{-1} S_\varphi(P_t, 0) S_\varphi(P_t, \sigma / T)^{-1} \varphi(x) \\
    &\leq \Vert \varphi(x) \Vert_{S_\varphi(P_t, \sigma / T)^{-1}}^2.
  \end{align*}
  where the second inequality uses the assumption $\vert y_t \vert \leq 1$ and the last inequality uses $S_\varphi(P_t, 0) \preccurlyeq S_\varphi(P_t, \sigma / T)$.
\end{proof}

We are now ready to prove Lemma~\ref{lemma:reward-estimate-concentration}.

\begin{refproof}[Lemma~\ref{lemma:reward-estimate-concentration}]
Fix an action $x \in \mathcal{X}$ and consider a martingale difference sequence $\{z_{t, x} \}_{t \in \mathcal{I}}$ where $z_{t, x} = \widehat{\mathcal{R}}_{\varphi, t}(x) - \mathbb{E}_t[ \widehat{\mathcal{R}}_{\varphi, t}(x)]$.
We can bound $z_{t, x}$ for all $t \in \mathcal{I}$ by
$$
z_{t, x}
\leq
\vert \widehat{\mathcal{R}}_{\varphi, t}(x) \vert + \vert \mathbb{E}_t[\widehat{\mathcal{R}}_{\varphi, t}(x)] \vert
\leq
\vert \widehat{\mathcal{R}}_{\varphi, t}(x) \vert + \mathbb{E}_t[\vert \widehat{\mathcal{R}}_{\varphi, t}(x) \vert ]
\leq
\frac{2 \gamma_{\varphi, T}}{\mu_j}
$$
where the last inequality uses Lemma~\ref{lemma:one-point-estimator} and $m_t \leq j$.
Also, by Lemma~\ref{lemma:one-point-estimator}, we have 
$$
\Var_t[z_{t, x}] = \Var_t[\widehat{\mathcal{R}}_{\varphi, t}(x)] \leq \Vert \varphi(x) \Vert^2_{S_\varphi(P_t, \sigma / T)^{-1}}.
$$
Using the Freedman inequality (Lemma~\ref{lemma:freedman}) on $\{ z_{t, x} \}_{t \in \mathcal{I}}$ with $\xi = \frac{\mu_j}{4 \gamma_{\varphi, T}} = \xi_j$, we get with probability at least $1 - \frac{\delta}{CN}$ that
$$
\begin{aligned}
\widehat{\mathcal{R}}_{\varphi, \mathcal{I}}(x)
- \mathcal{R}_\mathcal{I}&(x)
= \frac{1}{\vert \mathcal{I} \vert} \sum_{t \in \mathcal{I}} (z_{t, x} + \mathbb{E}_t[\widehat{\mathcal{R}}_{\varphi, t}(x)] - \mathcal{R}_{\mathcal{I}}(x)) \\
&\leq
\frac{\xi_j}{\vert \mathcal{I} \vert} \sum_{t \in \mathcal{I}} \Vert \varphi(x) \Vert^2_{S_\varphi(P_t, \sigma / T)^{-1}} + \frac{\log(CN / \delta)}{\xi_j \vert \mathcal{I} \vert} + \frac{\sqrt{\sigma / T}}{\vert \mathcal{I} \vert} \sum_{t \in \mathcal{I}} \Vert \varphi(x) \Vert_{S_\varphi(P_t, \sigma / T)^{-1}}
\end{aligned}
$$
where we use Lemma~\ref{lemma:one-point-estimator} to bound the bias term $\mathbb{E}_t[\widehat{\mathcal{R}}_{\varphi, t}(x)] - \mathcal{R}_{\mathcal{I}}(x)$.
A union bound over all $x \in \mathcal{X}$ and the reverse case $\mathcal{R}_\mathcal{I}(x) - \widehat{\mathcal{R}}_{\varphi, \mathcal{I}}(x)$ completes the proof.
\end{refproof}

Choosing $C = C_0 = 8 T \log_2 T$, we get by a union bound that for all intervals of sizes $E, 2E, 2^2 E, \dots$ and $(2^2 - 1)E, (2^3 - 1)E, \dots$, the concentration bound in Lemma~\ref{lemma:reward-estimate-concentration} holds with probability at least $1 - \delta$.
For ease of exposition, we define the following event.

\begin{definition}[$\textsc{Event}_1$] \label{definition:event1}
Denote by $\textsc{Event}_1$ the event that
$$
\vert \widehat{\mathcal{R}}_{\varphi, \mathcal{I}}(x) - \mathcal{R}_\mathcal{I}(x) \vert
\leq
\frac{\xi_j}{\vert \mathcal{I} \vert} \sum_{t \in \mathcal{I}} \Vert \varphi(x) \Vert_{S_\varphi(P_t, \sigma / T)^{-1}}^2 + \frac{\log(C_0 N/\delta)}{\xi_j \vert \mathcal{I} \vert}
+ \frac{\sqrt{\sigma / T}}{\vert \mathcal{I} \vert} \sum_{t \in \mathcal{I}} \Vert \varphi(x) \Vert_{S_\varphi(P_t, \sigma / T)^{-1}}
$$
holds for all intervals $\mathcal{I} \subset [T]$ of sizes $2^j E$ for all $j = 0, 1, \dots$ and $(2^j - 1)E$ for all $j = 1, 2, \dots$.
\end{definition}
By the previous argument, $\textsc{Event}_1$ holds with probability at least $1 - \delta$.

\subsection{Proof of Lemma~\ref{lemma:reward-estimate-concentration-gap}} \label{section:reward-estimate-concentration-gap-proof}

The following lemma bounds the optimality gaps of an action in two intervals by the total variation of the reward function throughout an interval that spans the two intervals.
The proof is adapted from Lemma~13 by \textcite{luo_efficient_2018} and Lemma~8 by \textcite{chen_new_2019}.

\begin{lemma}\label{lemma:gaps}
  For any interval $\mathcal{I}$, any of its sub-intervals $\mathcal{I}_1, \mathcal{I}_2 \subseteq \mathcal{I}$ and any $x \in \mathcal{X}$, we have
  $$
  \vert \Delta_{\mathcal{I}_1}(x) - \Delta_{\mathcal{I}_2}(x) \vert
  \leq 2 V_\mathcal{I}.
  $$
\end{lemma}
\begin{proof}
For all $x \in \mathcal{X}$, we have
$$
\begin{aligned}
\vert \mathcal{R}_{\mathcal{I}_1}(x) - \mathcal{R}_{\mathcal{I}_2}(x) \vert
&=
\left\vert
\frac{1}{\vert \mathcal{I}_1 \vert} \sum_{s \in \mathcal{I}_1} r_s(x) - \frac{1}{\vert \mathcal{I}_2 \vert} \sum_{t \in \mathcal{I}_2} r_t(x)
\right\vert \\
&=
\frac{1}{\vert \mathcal{I}_1 \vert \vert \mathcal{I}_2 \vert}
\left\vert
\sum_{s \in \mathcal{I}_1} \sum_{t \in \mathcal{I}_2}
\left(
r_s(x) - r_t(x)
\right)
\right\vert \\
&\leq
\frac{1}{\vert \mathcal{I}_1 \vert \vert \mathcal{I}_2 \vert}
\sum_{s \in \mathcal{I}_1} \sum_{t \in \mathcal{I}_2}
\left\vert
r_s(x) - r_t(x)
\right\vert
\leq
V_{\mathcal{I}}
\end{aligned}
$$
where the last inequality follows since
$\vert r_s(x) - r_t(x) \vert \leq \sum_{\tau = s}^{t - 1} \left\vert r_{\tau + 1}(x) - r_{\tau}(x) \right\vert \leq V_\mathcal{I}$.
Hence,
$$
- V_{\mathcal{I}}
\leq
\mathcal{R}_{\mathcal{I}_1}(x) - \mathcal{R}_{\mathcal{I}_2}(x),
\mathcal{R}_{\mathcal{I}_1}(x_{\mathcal{I}_1}^\star) -
\mathcal{R}_{\mathcal{I}_2}(x_{\mathcal{I}_1}^\star),
\mathcal{R}_{\mathcal{I}_1}(x_{\mathcal{I}_2}^\star) -
\mathcal{R}_{\mathcal{I}_2}(x_{\mathcal{I}_2}^\star)
\leq
V_{\mathcal{I}}
$$
where we use the notation $x_\mathcal{I}^\star = \argmax_{x' \in \mathcal{X}} \mathcal{R}_{\mathcal{I}}(x')$.
It follows that
$$
- V_{\mathcal{I}}
\leq
\mathcal{R}_{\mathcal{I}_1}(x^\star_{\mathcal{I}_2})
- \mathcal{R}_{\mathcal{I}_2}(x^\star_{\mathcal{I}_2})
\leq
\mathcal{R}_{\mathcal{I}_1}(x^\star_{\mathcal{I}_1})
- \mathcal{R}_{\mathcal{I}_2}(x^\star_{\mathcal{I}_2})
\leq
\mathcal{R}_{\mathcal{I}_1}(x^\star_{\mathcal{I}_1})
- \mathcal{R}_{\mathcal{I}_2}(x^\star_{\mathcal{I}_1})
\leq
V_{\mathcal{I}}
$$
where we use the optimality of $x_{\mathcal{I}_1}^\star$ and $x_{\mathcal{I}_2}^\star$.
Hence, for all $x \in \mathcal{X}$,
$$
\begin{aligned}
\vert \Delta_{\mathcal{I}_1}(x) - \Delta_{\mathcal{I}_2}(x) \vert
&=
\vert
\mathcal{R}_{\mathcal{I}_1}(x^\star_{\mathcal{I}_1})
- \mathcal{R}_{\mathcal{I}_1}(x)
- \mathcal{R}_{\mathcal{I}_2}(x^\star_{\mathcal{I}_2})
+ \mathcal{R}_{\mathcal{I}_2}(x)
\vert \\
&\leq
\vert
\mathcal{R}_{\mathcal{I}_1}(x^\star_{\mathcal{I}_1})
- \mathcal{R}_{\mathcal{I}_2}(x^\star_{\mathcal{I}_2})
\vert
+
\vert \mathcal{R}_{\mathcal{I}_1}(x) - \mathcal{R}_{\mathcal{I}_2}(x) \vert \leq 2 V_{\mathcal{I}}.
\end{aligned}
$$
\end{proof}

Now, we are ready to prove Lemma~\ref{lemma:reward-estimate-concentration-gap}.

\begin{refproof}[Lemma~\ref{lemma:reward-estimate-concentration-gap}] \label{proof:reward-estimate-concentration-gap}
Assume that the event \hyperref[definition:event1]{$\textsc{Event}_1$} holds.
We prove by induction on the block index $j$.
For the base case $j = 0$, note that the strategy used in block $\mathcal{B}(0)$ is $\pi_\varphi(\mathcal{X})$.
Under the event $\textsc{Event}_1$, using the result $\Vert \varphi(x) \Vert^2_{S_\varphi(\pi_\varphi(\mathcal{X}), \sigma / T)^{-1}} \leq \gamma_{\varphi, T}$ from Lemma~\ref{lemma:optimal-design} gives
$$
\vert \widehat{\mathcal{R}}_{\varphi, \mathcal{B}(0)}(x) - \mathcal{R}_{\mathcal{B}(0)}(x) \vert
\leq
\xi_0 \gamma_{\varphi, T}
+ \frac{\log(C_0 N / \delta)}{\xi_0 \vert \mathcal{B}(0) \vert}
+ \sqrt{\frac{\sigma \gamma_{\varphi, T}}{T}}
\leq \frac{c_0}{4} \mu_0
$$
where the last inequality follows by $\xi_0 = \frac{1}{8 \gamma_{\varphi, T}}$, $\vert \mathcal{B}(0) \vert = E \geq 4\gamma_{\varphi, T} \log(C_0 N / \delta)$ and $\sqrt{\sigma \gamma_{\varphi, T} / T} \leq 2\sqrt{\alpha} \mu_0$.
This proves the base case for the bound (\ref{eqn:reward-estimate-concentration}).

Now, suppose the bound (\ref{eqn:reward-estimate-concentration}) holds for the block indices 0, 1, \dots, $j$.
Then, for any $m = 0, \dots, j$, using the notations $x^\star = \argmax_{x \in \mathcal{X}} \mathcal{R}_{\mathcal{C}(m)} (x)$ and $\hat{x} = \argmax_{x \in \mathcal{X}} \widehat{\mathcal{R}}_{\varphi, \mathcal{C}(m)}(x)$, we have
$$
\begin{aligned}
\Delta_{\mathcal{C}(m)}(x) - \widehat{\Delta}_{\varphi, \mathcal{C}(m)}(x)
&=
\mathcal{R}_{\mathcal{C}(m)}(x^\star) - \mathcal{R}_{\mathcal{C}(m)}(x)
- \widehat{\mathcal{R}}_{\varphi, \mathcal{C}(m)} (\hat{x}) + \widehat{\mathcal{R}}_{\varphi, \mathcal{C}(m)}(x) \\
&\leq
\mathcal{R}_{\mathcal{C}(m)}(x^\star) - \mathcal{R}_{\mathcal{C}(m)}(x)
- \widehat{\mathcal{R}}_{\varphi, \mathcal{C}(m)} (x^\star) + \widehat{\mathcal{R}}_{\varphi, \mathcal{C}(m)}(x) \\
&\leq
\frac{1}{2} \Delta_{\mathcal{C}(m)}(x) + 2 V_{\mathcal{C}(m)} + \frac{c_0}{2} \mu_m
\end{aligned}
$$
where the first inequality uses the optimality of $\hat{x}$, and the second inequality uses the induction hypothesis and the fact that $\Delta_{\mathcal{C}(m)}(x^\star) = 0$.
Rearranging gives the bound (\ref{eqn:gap-estimate-concentration1}) for the blocks $0, \dots, j$.
Similarly, for $m = 0, \dots, j$, we have
$$
\begin{aligned}
\widehat{\Delta}_{\varphi, \mathcal{C}(m)}(x) - \Delta_{\mathcal{C}(m)}(x)
&\leq
\widehat{\mathcal{R}}_{\varphi, \mathcal{C}(m)}(\hat{x}) - \widehat{\mathcal{R}}_{\varphi, \mathcal{C}(m)}(x) - \mathcal{R}_{\mathcal{C}(m)}(\hat{x}) + \mathcal{R}_{\mathcal{C}(m)}(x) \\
&\leq
\frac{1}{2} \Delta_{\mathcal{C}(m)}(\hat{x}) + \frac{1}{2} \Delta_{\mathcal{C}(m)}(x) + 2 V_{\mathcal{C}(m)} + \frac{c_0}{2} \mu_m \\
&\leq
\frac{1}{2} \Delta_{\mathcal{C}(m)}(x) + 4 V_{\mathcal{C}(m)} + c_0 \mu_m
\end{aligned}
$$
where the first inequality uses the optimality of $x^\star$, the second inequality uses the induction hypothesis and the last inequality uses the bound (\ref{eqn:gap-estimate-concentration1}) we showed and the optimality of $\hat{x}$ to bound $\Delta_{\mathcal{C}(m)}(\hat{x}) \leq 2 \widehat{\Delta}_{\varphi, \mathcal{C}(j)}(x) + 4 V_{\mathcal{C}(j)} + c_0 \mu_j = 4V_{\mathcal{C}(j)} + c_0 \mu_j$.
Rearranging gives the bound (\ref{eqn:gap-estimate-concentration2}) for the blocks $0, \dots, j$.

Now, for the block index $j + 1$, $\textsc{Event}_1$ gives
\begin{align}
\vert \widehat{\mathcal{R}}_{\varphi, \mathcal{C}(j + 1)}(x) - \mathcal{R}_{\mathcal{C}(j + 1)}(x) \vert
\leq
&\frac{\xi_{j + 1}}{\vert \mathcal{C}(j + 1) \vert} \sum_{t \in \mathcal{C}(j + 1)} \Vert \varphi(x) \Vert_{S_\varphi(P_t, \sigma / T)^{-1}}^2 + \frac{\log(CN/\delta)}{\xi_{j + 1} \vert \mathcal{C}(j + 1) \vert} \nonumber \\
&+ \frac{\sqrt{\sigma / T}}{\vert \mathcal{C}(j + 1) \vert} \sum_{t \in \mathcal{C}(j + 1)} \Vert \varphi(x) \Vert_{S_\varphi(P_t, \sigma / T)^{-1}}. \label{eqn:reward-bound-intermediate}
\end{align}
To bound the first term, we use Lemma~\ref{lemma:optimization-problem} and the bound (\ref{eqn:gap-estimate-concentration2}) we showed for blocks $0, \dots, j$ to get
\begin{align}
\xi_{j + 1} \Vert \varphi(x) \Vert^2_{S_\varphi(P_t, \sigma / T)^{-1}}
&\leq
2 \xi_{j + 1} \Vert \varphi(x) \Vert^2_{S_\varphi(Q^{(m_t)}, \sigma / T)^{-1}} \nonumber \\
&\leq
2 \xi_{j + 1} (\beta_{m_t} \widehat{\Delta}_{\varphi, \mathcal{C}(m_t - 1)}(x) + 2 \gamma_{\varphi, T}) \nonumber \\
&\leq
2 \xi_{j + 1} (\beta_{m_t} (2 \Delta_{\mathcal{C}(m_t - 1)}(x) + 4 V_{\mathcal{C}(m_t - 1)} + c_0 \mu_{m_t - 1})  + 2 \gamma_{\varphi, T}) \nonumber \\
&\leq
\frac{1}{20} \Delta_{\mathcal{C}(m_t - 1)}(x) + \frac{1}{10} V_{\mathcal{C}(m_t - 1)} + \frac{3}{2} \mu_{j + 1} \label{eqn:xi_varphi2_bound} \\
&\leq
\frac{1}{20} \Delta_{\mathcal{C}(j + 1)}(x) + \frac{1}{5} V_{\mathcal{C}(j + 1)} + 2 \mu_{j + 1} \nonumber
\end{align}
where the second to last inequality follows by a simple calculation using identities in Section~\ref{subsection:constants} and the fact that $m_t \leq j + 1$ for $t \in \mathcal{C}(j + 1)$ and the last inequality follows by Lemma~\ref{lemma:gaps}.

The second term can be bounded by
\begin{equation} \label{eqn:log_by_xi_bound}
\frac{\log(CN / \delta)}{\xi_{j + 1} \vert \mathcal{C}(j + 1) \vert}
= \frac{4 \gamma_{\varphi, T}\log(CN / \delta)}{\mu_{j + 1} E \cdot 2^{j + 1}}
\leq
\frac{1}{\mu_{j + 1} 2^{j + 1}}
= 4\mu_{j + 1}.
\end{equation}

The third term can be bounded using Lemma~\ref{lemma:optimization-problem} and the bound (\ref{eqn:gap-estimate-concentration2}):
\begin{align}
\sqrt{\sigma / T} \Vert \varphi(x) \Vert_{S_\varphi(P_t, \sigma / T)^{-1}}
&\leq
\sqrt{2\sigma / T} \Vert \varphi(x) \Vert_{S_\varphi(Q^{(m_t)}, \sigma / T)^{-1}} \nonumber \\
&\leq
\frac{\sqrt{\gamma}}{\sqrt{\alpha \gamma_{\varphi, T} T}} \beta_{m_t} \widehat{\Delta}_{\varphi, \mathcal{C}(m_t - 1)}(x) + \frac{2 \sqrt{\sigma \gamma_{\varphi, T}}}{\sqrt{T}} \nonumber \\
&\leq
\frac{2 \mu_{j + 1}}{\gamma_{\varphi, T}} \beta_{m_t}(2 \Delta_{\mathcal{C}(m_t - 1)}(x) + 4 V_{\mathcal{C}(m_t - 1)} + c_0 \mu_{m_t - 1}) + 4 \sqrt{\alpha} \mu_{j + 1} \nonumber \\
&\leq
\frac{1}{5} \Delta_{\mathcal{C}(m_t - 1)}(x) + \frac{2}{5} V_{\mathcal{C}(m_t - 1)} + (4 + 4 \sqrt{\alpha}) \mu_{j + 1} \label{eqn:gamma_varphi_bound}\\
&\leq
\frac{1}{5} \Delta_{\mathcal{C}(j + 1)}(x) + \frac{4}{5} V_{\mathcal{C}(j + 1)} + (4 + 4 \sqrt{\alpha}) \mu_{j + 1} \nonumber
\end{align}
where the second inequality uses $\sqrt{a + b} \leq \sqrt{a} + \sqrt{b}$ and the third inequality uses $\sqrt{\sigma \gamma_{\varphi, T} / T} \leq 2 \sqrt{\alpha} \mu_j$ for any block index $j$ and the second to last inequality follows by a simple calculation and the last inequality follows by Lemma~\ref{lemma:gaps}.

Using these three bounds, we can further bound (\ref{eqn:reward-bound-intermediate}) by $\vert \widehat{\mathcal{R}}_{\varphi, \mathcal{C}(j + 1)}(x) - \mathcal{R}_\mathcal{C}(j + 1)(x) \vert \leq \frac{1}{2} \Delta_{\mathcal{C}(j + 1)}(x) + V_{\mathcal{C}(j + 1)} + \frac{c_0}{4} \mu_{j + 1}$, which proves the bound (\ref{eqn:reward-estimate-concentration}) for the block $j + 1$.
By induction, the proof is complete.

\end{refproof}

\subsection{Proof of Theorem~\ref{theorem:opkb}}\label{section:proof:opkb}

\begin{refproof}[Theorem~\ref{theorem:opkb}]
We bound the regret of each block $\mathcal{B}(j)$ separately.
Using the Azuma-Hoeffding inequality on a martingale difference sequence $\{ \mathbb{E}_t[r(x_t)] - r(x_t) \}_{t \in \mathbb{B}(j)}$, we get
$$
\textsc{Reg}_{\mathcal{B}(j)} = \sum_{t \in \mathcal{B}(j)} (r(x^\star) - r(x_t)) \leq \sum_{t \in \mathcal{B}(j)} (r(x^\star) - \mathbb{E}_t [ r(x_t) ]) + \widetilde{\mathcal{O}}(\sqrt{2^j E})
$$
where we use $r(\cdot)$ to denote the stationary reward function and $x^\star = \argmax_{x \in \mathcal{X}} r(x)$.
Since $P_t = (1 - \mu_j) Q^{(j)} + \mu_j \pi_\varphi(\mathcal{X})$ for $t \in \mathcal{B}(j)$,
using Lemma~\ref{lemma:reward-estimate-concentration-gap} with $V_{\mathcal{C}(j)} = 0$, we get with high probability that
$$
r(x^\star) - \mathbb{E}_t[r(x_t)] = \sum_{x \in \mathcal{X}} P_t(x) \Delta_{\mathcal{C}(j - 1)}(x)
\leq
2 \sum_{x \in \mathcal{X}} Q^{(j)}(x) \widehat{\Delta}_{\varphi,  \mathcal{C}(j - 1)}(x) + \mathcal{O}(\mu_j)
\leq
\mathcal{O}(\mu_j)
$$
where the last inequality uses Lemma~\ref{lemma:optimization-problem} and $1 / \beta_j = \mathcal{O}(\mu_j)$.
Summing over $t \in \mathcal{B}(j)$, we get $\textsc{Reg}_{\mathcal{B}(j)} \leq \widetilde{\mathcal{O}}(E \sqrt{2^j})$.
Summing over $j$ and applying Cauchy-Schwarz, we get $\textsc{Reg}_T = \widetilde{\mathcal{O}}(E \sqrt{T / E}) = \widetilde{\mathcal{O}}(\sqrt{\gamma_T T \log N})$.
\end{refproof}

\subsection{Subgaussian case} \label{subsection-subgaussian}

For the analysis with subgaussian noises, we can use the following modified Freedman-style inequality.

\begin{lemma}\label{lemma:modified-freedman}
Let $X_1, \dots, X_n \in \mathbb{R}$ be a martingale difference sequence
with respect to a filtration $\mathcal{F}_0, \mathcal{F}_1, \dots$.
Assume $X_i$ are $\sigma$-subguassian.
Then for any $\delta \in (0, 1)$ and
$\xi \in [0, 1 / \sqrt{2 \sigma^2 \log(n / \delta)}]$, we have with probability at least $1 - 2\delta$ that
$$
\sum_{i=1}^n X_i \leq \xi V + \frac{\log (1 / \delta)}{\xi},
$$
where $V = \sum_{i=1}^n \mathbb{E}[ X_i^2 \mid \mathcal{F}_{i-1} ]$.
\end{lemma}

\begin{proof}
The proof closely follows the proof of the original Freedman-style inequality by \textcite{beygelzimer_contextual_2011}.
Since $X_1$, \dots, $X_n$ are $\sigma$-subguassian, we have $X_t \leq B = \sqrt{2 \sigma^2 \log (n / \delta)}$ for all $t = 1, \dots, n$ with probability at least $1 - \delta$.
Define $\widetilde{X}_t = \min \{ X_t, B \}$ for $i = 1, \dots, n$.
Then,
\begin{equation}\label{eqn:prev-ineq}
\mathbb{E}_t [ \exp( \xi \widetilde{X}_t )] \leq \mathbb{E}_t [ 1 + \xi \widetilde{X}_t + \xi^2 \widetilde{X}_t^2]
\leq 1 + \xi^2 \mathbb{E}_t [ \widetilde{X}_t^2 ]
\leq \exp(\xi^2 \mathbb{E}_t [\widetilde{X}_t^2])
\leq \exp(\xi^2 \mathbb{E}_t [X_t^2])
\end{equation}
where the first inequality uses the fact that $\xi \leq 1 / B$ and the identity $e^z \leq 1 + z + z^2$ for $z \leq 1$.
Define $Z_0 = 1$ and $Z_t = Z_{t -1} \exp(\xi \widetilde{X}_t - \xi^2 \mathbb{E}_t[X_t^2])$.
Then,
$$
\mathbb{E}_t[Z_t] = Z_{t - 1} \exp(-\xi^2 \mathbb{E}_t [ X_t^2 ]) \mathbb{E}_t [\exp(\xi \widetilde{X}_t)] \leq 1
$$
where the last inequality holds by (\ref{eqn:prev-ineq}).
Hence, we have $\mathbb{E}[Z_n] \leq 1$ and by Markov inequality, $P(Z_n \geq 1 / \delta) \leq \delta$.
Note that by recursive definition, we have $Z_n = \exp( \xi \sum_{t = 1}^n \widetilde{X}_t - \xi^2 \sum_{t = 1}^n \mathbb{E}_t X_t^2)$.
Hence, $\sum_{t = 1}^n \widetilde{X}_t \leq \xi \sum_{t = 1}^n \mathbb{E}_t X_t^2 + \log(1 / \delta) / \xi$ with probability at least $1 - \delta$.

By the previous argument that $X_t \leq B$ for all $t = 1, \dots, n$ with probability at least $1 - \delta$, we have $\sum X_t^2 = \sum \widetilde{X}_t^2$ with probability at least $1 - \delta$.
By a union bound, we have $\sum X_t^2 = \sum \widetilde{X}_t^2 \leq \xi V + \log(1 / \delta) / \xi$ with probability at least $1 - 2\delta$ as desired.
\end{proof}

\section{MASTER reduction of GPUCB}\label{section:master-reduction}

\textcite{wei_non-stationary_2021} introduce the MASTER reduction that converts a base algorithm into an algorithm that adapts to non-stationarity.
They prove that if a base algorithm satisfies Condition~\ref{condition:wei-2021} for a constant $\omega$, then the converted algorithm satisfies the dynamic regret bound displayed in Theorem~\ref{theorem:wei-2021} without prior knowledge of the non-stationarity budgets.

\begin{condition}[Adapted from Assumption 1' in \textcite{wei_non-stationary_2021}]\label{condition:wei-2021}
For any $t = 1, \dots, T$, as long as $\omega V_{[1, t]} \leq \rho(t)$, the base algorithm can produce $\tilde{f}_t$ using history up to $t - 1$ that satisfies
$$
\tilde{f}_t \geq \min_{\tau \in [1, t]} \max_{x \in \mathcal{X}} r_t(x) - \omega V_{[1, t]}
\quad
\text{and}
\quad
\frac{1}{t} \sum_{\tau = 1}^t (\tilde{f}_\tau - y_\tau) \leq c \rho(t) + c \omega V_{[1, t]}
$$
with probability at least $1 - \frac{\delta}{T}$ where $\rho(t) \geq \frac{1}{\sqrt{t}}$, $t \rho(t)$ is non-decreasing in $t$, $\omega$ is some function of the parameters, and $c$ is a universal constant.
\end{condition}

\begin{theorem}[Adapted from Theorem 2 in \textcite{wei_non-stationary_2021}]\label{theorem:wei-2021}
If a base algorithm satisfies Condition~\ref{condition:wei-2021} with $t \rho(t) = g_1 \sqrt{t} + g_2$, then the algorithm obtained by the MASTER reduction guarantees with high probability that
$$
\textsc{Reg}_T = \widetilde{\mathcal{O}}\left(\min\left\{ (g_1 + g_1^{-1}g_2) \sqrt{L_T T}, (g_1^{2/3} + g_2 g_1^{-4/3}) \omega^{1/3} V_T^{1/3} T^{2/3} + (g_1 + g_1^{-1} g_2) \sqrt{T} \right\}\right).
$$
\end{theorem}

Now, we show that the GPUCB algorithm \parencite{chowdhury_kernelized_2017} satisfies Condition~\ref{condition:wei-2021}, and provide the resulting dynamic regret bounds.

The GPUCB algorithm (Algorithm~\ref{alg:igp-ucb}) is a UCB-based algorithm for stationary kernel bandits introduced by \textcite{chowdhury_kernelized_2017}.
They use a surrogate prior model $GP(0, k(\cdot, \cdot))$ on $f$ and use the posterior distribution $GP(\mu_t(\cdot), k_t(\cdot, \cdot))$ given observed rewards up to time $t$ for designing the upper confidence bounds of reward estimates.
It can be shown that
$$
\mu_t(x) = \varphi(x)^T \Phi^T(\Phi \Phi^T + \lambda I)^{-1} y_{1:t}, \quad
k_t(x, x') = k(x, x') - \varphi(x)^T \Phi^T(\Phi \Phi^T + \lambda I)^{-1} \Phi \varphi(x')
$$
where $\varphi$ is a feature mapping induced by the kernel $k$, $\Phi = [\varphi(x_1) \cdots \varphi(x_t)]^T$ and $y_{1:t} = (y_1, \dots, y_t)$.

\begin{algorithm}
\KwInput{kernel $k$, confidence level $\delta \in (0, 1)$, regularization parameter $\lambda$}
\For{$t = 1, \dots, T$}{
Set $\beta_t \leftarrow 1 + \sqrt{2 (\gamma_{t - 1} + 1 + \log(1 / \delta))}$ and  $\sigma_t^2 \leftarrow k_t(x, x)$ \\
Play $x_t = \argmax_{x \in \mathcal{X}} \mu_{t - 1}(x) + \beta_t \sigma_{t - 1}(x)$ and receive reward $y_t$.
}
\caption{$\textsc{GPUCB}$ \cite{chowdhury_kernelized_2017}}
\label{alg:igp-ucb}
\end{algorithm}

The following lemma shows that GPUCB satisfies Condition~\ref{condition:wei-2021}.

\begin{lemma}
The GPUCB algorithm satisfies Condition~\ref{condition:wei-2021} with $\tilde{f}_t = \max_{x \in \mathcal{X}} (\mu_{t - 1}(x) + \beta_t \sigma_{t - 1}(x)$, $\rho(t) = \beta_t \sqrt{\gamma_{t - 1} \log(T / \delta) / t}$ and $\omega = \gamma_T \sqrt{\log(T / \delta)}$.
\end{lemma}
\begin{proof}
Let $W_t \coloneqq \sum_{s = 1}^t \varphi(x_s) \varphi(x_s)^T + \lambda I$.
It can be shown that $\sigma_t(x) = \sqrt{k_t(x, x)} = \sqrt{\lambda} \Vert \varphi(x) \Vert_{W_t^{-1}}$.
Following the proof of Lemma 1 in \textcite{zhou_no-regret_2021}, we get
$$
\vert r_t(x) - \mu_{t - 1}(x) \vert \leq
\left\vert \varphi(x)^T W_{t - 1}^{-1} \sum_{s = 1}^{t - 1} \varphi(x_s) \varphi(x_s)^T(\theta_t - \theta_s) \right\vert + \beta_t \Vert \varphi(x) \Vert_{W_{t - 1}^{-1}}
$$
Following the corrected version of the analysis for the reduction of OFUL in \textcite{wei_non-stationary_2021}, we get
$$
\begin{aligned}
\left\vert \varphi(x)^T W_{t - 1}^{-1} \sum_{s = 1}^{t - 1} \varphi(x_s) \varphi(x_s)^T(\theta_t - \theta_s) \right\vert
&\leq
\sum_{s = 1}^{t - 1}
\vert \varphi(x)^T W_{t - 1}^{-1} \varphi(x_s) \vert
\vert \varphi(x_s)^T(\theta_t - \theta_s) \vert \\
&\leq
V_{[1, t]} \Vert \varphi(x) \Vert_{W_{t - 1}^{-1}} \sum_{s = 1}^{t - 1} \Vert \varphi(x_s) \Vert_{W_{t - 1}^{-1}} \\
&\leq
V_{[1, t]} \Vert \varphi(x) \Vert_{W_{t - 1}^{-1}} \sqrt{(t - 1) \sum_{s = 1}^{t - 1} \Vert \varphi(x_s) \Vert_{W_{t - 1}^{-1}}^2} \\
&\leq
V_{[1, t]} \Vert \varphi(x) \Vert_{W_{t - 1}^{-1}} \sqrt{t \gamma_{t - 1}}
\end{aligned}
$$
where the second inequality is by Cauchy-Schwarz, $\vert \theta_t - \theta_s \vert \leq V_{[1, t]}$ and the assumption $\Vert \varphi(x_s) \Vert \leq 1$.
The third inequality is by Cauchy-Schwarz.
The last inequality is by
\begin{equation} \label{eqn:variance_info_gain_bound}
\sum_{s = 1}^{t - 1} \Vert \varphi(x_s) \Vert^2_{W_{t - 1}^{-1}} = \sum_{s = 1}^{t - 1} U(x_s) \Vert \varphi(x_s) \Vert^2_{S_\varphi(U, \sigma / (t - 1))^{-1}}
\leq \gamma_{t - 1}
\end{equation}
where $U$ is the uniform distribution on $\{x_1, \dots, x_{t - 1} \}$ and the inequality is by Lemma~\ref{lemma:variance-bounded-by-information-gain}.
Hence,
$$
\vert r_t(x) - \mu_{t - 1}(x) \vert \leq (V_{[1, t]} \sqrt{t \gamma_{t - 1}} + \beta_t) \Vert \varphi(x) \Vert_{W_{t - 1}^{-1}}
\leq 2 \beta_t \Vert \varphi(x) \Vert_{W_{t - 1}^{-1}}
= 2 \beta_t \sigma_{t - 1}(x) / \sqrt{\lambda}
$$
where the last inequality uses $V_{[1, t]} \leq \rho(t) / \omega \leq \beta_t / \sqrt{t \gamma_T}$.
Thus,
$$
\begin{aligned}
\sum_{\tau = 1}^t (\tilde{f}_\tau - y_\tau)
&=
\sum_{\tau = 1}^t (\tilde{f}_\tau - r_\tau(x_\tau)) + \sum_{\tau = 1}^t (r_\tau(x_\tau) - y_\tau) \\
&=
\sum_{\tau = 1}^t (\mu_{\tau - 1}(x_\tau) - r_\tau(x_\tau)) + \sum_{\tau = 1}^t \beta_\tau \sigma_{\tau - 1}(x_\tau) + \mathcal{O}(\sqrt{t \log(T / \delta)}) \\
&=
\mathcal{O}( \sum_{\tau = 1}^t \beta_\tau \sigma_{\tau - 1}(x_\tau) + \sqrt{t \log(T / \delta)}) \\
&=
\mathcal{O}(\beta_t \sqrt{t \gamma_T \log(T / \delta)})
\end{aligned}
$$
where the second equality uses the fact that $\tilde{f}_\tau = \mu_{\tau - 1}(x_\tau) + \beta_\tau \sigma_{\tau - 1}(x_\tau)$ due to the optimism principle of the algorithm.
The last equality uses
$$
\sum_{\tau = 1}^t \beta_\tau \sigma_{\tau - 1}(x_\tau) \leq
\beta_t \sum_{\tau = 1}^t \sigma_{\tau - 1}(x_\tau) \leq \mathcal{O}(\beta_t \sqrt{t \gamma_t})
$$
where the last inequality uses Lemma 4 in \textcite{chowdhury_kernelized_2017}.
This verifies the second condition in Condition~\ref{condition:wei-2021}.
Also,
$$
\tilde{f}_t = \max_{x \in \mathcal{X}} (\mu_{t - 1}(x) + \beta_t \sigma_{t - 1}(x))
\geq \max_{x \in \mathcal{X}} r_t(x) \geq \min_{\tau \in [1, t]} \max_{x \in \mathcal{X}} r_\tau(x)
$$
where the first inequality uses Theorem 2 in \textcite{chowdhury_kernelized_2017}.
This shows the first condition, completing the proof.
\end{proof}

The previous lemma allows invoking the MASTER reduction for GPUCB, which gives a dynamic regret bound of
$$
\textsc{Reg}_T \leq \widetilde{\mathcal{O}}(\min\{ \gamma_T \sqrt{L_T T}, \gamma_T V^{1/3} T^{2/3} + \gamma_T \sqrt{T} \}).
$$

\section{Analysis of ADA-OPKB}\label{section:analysis-ada-opkb}

For ease of exposition, we use the same set of parameters listed in Section~\ref{subsection:constants}.

\subsection{Change detection} \label{subsection:change-detection}

In this subsection, we prove properties of the change detection rules used in ADA-OPKB.

\begin{lemma}\label{lemma:gap-estimate-concentration-replay}
Assume the event \hyperref[definition:event1]{$\textsc{Event}_1$} holds.
Then, we have for any $x \in \mathcal{X}$ and replay interval $(m, \mathcal{I})$ that
$$
\begin{aligned}
\Delta_{\mathcal{I}}(x) &\leq 2 \widehat{\Delta}_{\varphi, \mathcal{I}}(x) + c_0 \mu_m + 4 V_{[\tau_i, t]} \\
\widehat{\Delta}_{\varphi, \mathcal{I}}(x) &\leq 2 \Delta_{\mathcal{I}}(x) + c_0 \mu_m + 4 V_{[\tau_i, t]}
\end{aligned}
$$
where $\tau_i$ is the starting time of the epoch $i$ in which $\mathcal{I}$ is scheduled and $t$ is the end of the interval $\mathcal{I}$.
\end{lemma}

\begin{proof}
Consider a replay interval $(m, \mathcal{I})$ scheduled in a block $\mathcal{B}(j)$ in epoch $i$ and let $\tau_i$ be the starting time of the epoch $i$ and $t$ be the end time of $\mathcal{I}$.
Following the calculation in (\ref{eqn:xi_varphi2_bound}) in the proof of Lemma~\ref{lemma:reward-estimate-concentration-gap}, we get
$$
\begin{aligned}
\xi_m \Vert \varphi(x) \Vert^2_{S_\varphi(P_t, \sigma / T)^{-1}}
&\leq
\frac{1}{20} \Delta_{\mathcal{C}(m_t - 1)}(x) + \frac{1}{10} V_{\mathcal{C}(m_t - 1)} + \frac{3}{2} \mu_m \\
&\leq
\frac{1}{20} \Delta_{\mathcal{I}}(x) + \frac{1}{5} V_{[\tau_i, t]} + 2 \mu_m \\
\end{aligned}
$$
where the second inequality uses Lemma~\ref{lemma:gaps} and the fact that both $\mathcal{C}(m_t - 1)$ and $\mathcal{I}$ lie in $[\tau_i, t]$.
Likewise, following the calculation in (\ref{eqn:gamma_varphi_bound}) in the proof of Lemma~\ref{lemma:reward-estimate-concentration-gap} and using Lemma~\ref{lemma:gaps}, we get
$$
\begin{aligned}
\sqrt{\sigma / T} \Vert \varphi(x) \Vert_{S_\varphi(P_t, \sigma / T)^{-1}}
&\leq
\frac{1}{5} \Delta_{\mathcal{C}(m_t - 1)}(x) + \frac{2}{5} V_{\mathcal{C}(m_t - 1)} + (4 + 4 \sqrt{\alpha}) \mu_m \\
&\leq
\frac{1}{5} \Delta_{\mathcal{I}}(x) + \frac{4}{5} V_{[\tau_i, t]} + (4 + 4 \sqrt{\alpha}) \mu_m.
\end{aligned}
$$
Note that $m$ is the maximum strategy index used in $\mathcal{I}$ due to the index selection logic in Line~\ref{alg:line:strategy-index} in Algorithm~\ref{alg:ada-opkb}.
Hence, under the event \hyperref[definition:event1]{$\textsc{Event}_1$}, the two bounds above and the bound~(\ref{eqn:log_by_xi_bound}) give
\begin{align}
\vert \widehat{\mathcal{R}}_{\varphi, \mathcal{I}}(x) &- \mathcal{R}_\mathcal{I}(x) \vert \nonumber \\
&\leq
\frac{\xi_m}{\vert \mathcal{I} \vert} \sum_{t \in \mathcal{I}} \Vert \varphi(x) \Vert_{S_\varphi(P_t, \sigma / T)^{-1}}^2 + \frac{\log(CN/\delta)}{\xi_m \vert \mathcal{I} \vert}
+ \frac{\sqrt{\sigma / T}}{\vert \mathcal{I} \vert} \sum_{t \in \mathcal{I}} \Vert \varphi(x) \Vert_{S_\varphi(P_t, \sigma / T)^{-1}} \nonumber \\
&\leq
\frac{1}{2} \Delta_\mathcal{I}(x) + V_{[\tau_i, t]} + \frac{c_0}{4} \mu_m. \label{eqn:replay-reward-bound}
\end{align}
Denoting $\hat{x} = \argmax_{x' \in \mathcal{X}} \widehat{\mathcal{R}}_{\varphi, \mathcal{I}}(x')$ and $x^\star = \argmax_{x' \in \mathcal{X}} \mathcal{R}_\mathcal{I}(x)$, we have
$$
\begin{aligned}
\Delta_\mathcal{I}(x) - \widehat{\Delta}_{\varphi, \mathcal{I}}(x)
&=
\mathcal{R}_\mathcal{I}(x^\star) - \mathcal{R}_\mathcal{I}(x)
- \widehat{\mathcal{R}}_{\varphi, \mathcal{I}} (\hat{x}) + \widehat{\mathcal{R}}_{\varphi, \mathcal{I}}(x) \\
&\leq
\mathcal{R}_{\mathcal{I}}(x^\star) - \mathcal{R}_{\mathcal{I}}(x)
- \widehat{\mathcal{R}}_{\varphi, \mathcal{I}} (x^\star) + \widehat{\mathcal{R}}_{\varphi, \mathcal{I}}(x) \\
&\leq
\frac{1}{2} \Delta_{\mathcal{I}}(x) + 2 V_{[\tau_i, t]} + \frac{c_0}{2} \mu_m
\end{aligned}
$$
where the first inequality uses the optimality of $\hat{x}$ and the second inequality uses the bound~(\ref{eqn:replay-reward-bound}) and $\Delta_\mathcal{I}(x^\star) = 0$.
Rearranging proves the first inequality of the lemma.
The second inequality can be shown by
$$
\begin{aligned}
\widehat{\Delta}_{\varphi, \mathcal{I}}(x) - \Delta_{\mathcal{I}}(x)
&\leq
\widehat{\mathcal{R}}_{\varphi, \mathcal{I}}(\hat{x}) - \widehat{\mathcal{R}}_{\varphi, \mathcal{I}}(x) - \mathcal{R}_\mathcal{I}(\hat{x}) + \mathcal{R}_\mathcal{I}(x) \\
&\leq
\frac{1}{2} \Delta_\mathcal{I}(\hat{x}) + \frac{1}{2} \Delta_\mathcal{I}(x) + 2 V_{[\tau_i, t]} + \frac{c_0}{2} \mu_m \\
&\leq
\frac{1}{2} \Delta_\mathcal{I}(x) + 4 V_{[\tau_i, t]} + c_0 \mu_m
\end{aligned}
$$
where the first inequality uses the optimality of $x^\star$, the second inequality uses the bound~(\ref{eqn:replay-reward-bound}) and the last inequality uses the first inequality of the lemma.
Rearranging proves the second inequality of the lemma.

\end{proof}

\begin{lemma}\label{lemma:gap-estimate-concentration-replay-no-restart}
Let $(m, \mathcal{I})$ be a replay interval scheduled in $\mathcal{S}$ for block $j$ in some epoch $i$.
If no restart is triggered by this replay interval when performing the change detection test at the end of $\mathcal{I}$, we have with probability at least $1 - \delta$ for all $x \in \mathcal{X}$ that
$$
\begin{aligned}
\widehat{\Delta}_{\varphi, \mathcal{I}}(x) &\leq 2 \Delta_{\mathcal{I}}(x) + 4 c_0 \mu_m \\
\Delta_{\mathcal{I}}(x) &\leq 2 \widehat{\Delta}_{\varphi, \mathcal{I}}(x) + 4 c_0 \mu_m
\end{aligned}
$$
\end{lemma}

\begin{proof}
Suppose no restart is triggered by the test~(\ref{eqn:change-detection}) for $(m, \mathcal{I})$.
Then, $\widehat{\Delta}_{\varphi, \mathcal{I}}(x) - 4 \widehat{\Delta}_{\varphi, \mathcal{C}(k)}(x) \leq 4 c_0 \mu_{m \wedge k}$ and $\widehat{\Delta}_{\varphi, \mathcal{C}(k)}(x) - 4 \widehat{\Delta}_{\varphi, \mathcal{I}}(x) \leq 4 c_0 \mu_{m \wedge k}$ for all $k = 0, \dots, j - 1$.
Hence, for $t \in \mathcal{I}$, we have
$$
\begin{aligned}
\xi_m \Vert \varphi(x) \Vert^2_{S_\varphi(P_t, \sigma / T)^{-1}}
&\leq
2 \xi_m \Vert \varphi(x) \Vert^2_{S_\varphi(Q^{(m_t)}, \sigma / T)^{-1}} \\
&\leq
2 \xi_m (\beta_{m_t} \widehat{\Delta}_{\varphi, \mathcal{C}(m_t - 1)}(x) + 2 \gamma_{\varphi, T}) \\
&\leq
2 \xi_m (\beta_{m_t} (4 \widehat{\Delta}_{\varphi, \mathcal{I}}(x) + 4 c_0 \mu_{(m_t - 1) \wedge m}) + 2 \gamma_{\varphi, T}) \\
&\leq
8 \xi_m \beta_m \widehat{\Delta}_{\varphi, \mathcal{I}}(x) + 8 \sqrt{2} c_0 \xi_m \beta_{m_t} \mu_{m_t} + 4 \xi_m \gamma_{\varphi, T} \\
&\leq
\frac{1}{10} \widehat{\Delta}_{\varphi, \mathcal{I}}(x) + 10 \mu_m
\end{aligned}
$$
where the first inequality uses $\frac{1}{2} S_\varphi(Q^{(m_t)}, \sigma / T) \preccurlyeq S_\varphi(P_t, \sigma / T)$, the second inequality uses Lemma~\ref{lemma:optimization-problem}, the fourth inequality uses $m_t \leq m$ and $\mu_{m_t - 1} = \sqrt{2} \mu_{m_t}$.
The last inequality holds by simple calculation.
Similarly,
$$
\begin{aligned}
\sqrt{\sigma / T} \Vert \varphi(x) \Vert_{S_\varphi(P_t, \sigma / T)^{-1}}
&\leq
\sqrt{2\sigma / T} \Vert \varphi(x) \Vert_{S_\varphi(Q^{(m_t)}, \sigma / T)^{-1}} \\
&\leq
\frac{\sqrt{\sigma}}{\sqrt{\alpha \gamma_{\varphi, T} T}} \beta_{m_t} \widehat{\Delta}_{\varphi, \mathcal{C}(m_t - 1)}(x) + \frac{2 \sqrt{\sigma \gamma_{\varphi, T}}}{\sqrt{T}} \\
&\leq
\frac{2 \mu_m}{\gamma_{\varphi, T}} \beta_{m_t} (4 \widehat{\Delta}_{\varphi, \mathcal{I}}(x) + 4 c_0 \mu_{(m_t - 1) \wedge m}) + 4 \sqrt{\alpha} \mu_m \\
&\leq
\frac{8 \mu_m \beta_m}{\gamma_{\varphi, T}} \widehat{\Delta}_{\varphi, \mathcal{I}}(x) + \frac{8 \sqrt{2} c_0 \mu_m \beta_{m_t} \mu_{m_t}}{\gamma_{\varphi, T}} + 4 \sqrt{\alpha} \mu_m \\
&\leq
\frac{2}{5} \widehat{\Delta}_{\varphi, \mathcal{I}}(x) + 24 \mu_m + 4 \sqrt{\alpha} \mu_m.
\end{aligned}
$$
where the second inequality uses Lemma~\ref{lemma:optimization-problem} and $\sqrt{a + b} \leq \sqrt{a} + \sqrt{b}$, the third inequality uses $\sqrt{\sigma \gamma_{\varphi, T} / T} \leq 2 \sqrt{\alpha} \mu_j$

Under the event \hyperref[definition:event1]{$\textsc{Event}_1$}, the two bounds above and the bound~(\ref{eqn:log_by_xi_bound}) give
\begin{align}
\vert \widehat{\mathcal{R}}_{\varphi, \mathcal{I}}(x) &- \mathcal{R}_\mathcal{I}(x) \vert \nonumber \\
&\leq
\frac{\xi_m}{\vert \mathcal{I} \vert} \sum_{t \in \mathcal{I}} \Vert \varphi(x) \Vert_{S_\varphi(P_t, \sigma / T)^{-1}}^2 + \frac{\log(CN/\delta)}{\xi_m \vert \mathcal{I} \vert}
+ \frac{\sqrt{\sigma / T}}{\vert \mathcal{I} \vert} \sum_{t \in \mathcal{I}} \Vert \varphi(x) \Vert_{S_\varphi(P_t, \sigma / T)^{-1}} \nonumber \\
&\leq
\frac{1}{2} \widehat{\Delta}_{\varphi, \mathcal{I}}(x) + 38 \mu_m + 4 \sqrt{\alpha} \mu_m
\leq
\frac{1}{2} \widehat{\Delta}_{\varphi, \mathcal{I}}(x) + c_0 \mu_m
\label{eqn:norestart-reward-bound}
\end{align}
Denoting $\hat{x} = \argmax_{x' \in \mathcal{X}} \widehat{\mathcal{R}}_{\varphi, \mathcal{I}}(x')$ and $x^\star = \argmax_{x' \in \mathcal{X}} \mathcal{R}_\mathcal{I}(x)$, we have
$$
\widehat{\Delta}_{\varphi, \mathcal{I}}(x) - \Delta_{\mathcal{I}}(x)
\leq
\widehat{\mathcal{R}}_{\varphi, \mathcal{I}}(\hat{x}) - \widehat{\mathcal{R}}_{\varphi, \mathcal{I}}(x) - \mathcal{R}_\mathcal{I}(\hat{x}) + \mathcal{R}_\mathcal{I}(x)
\leq
\frac{1}{2} \widehat{\Delta}_{\varphi, \mathcal{I}}(x) + 2 c_0 \mu_m
$$
where the first inequality uses the optimality of $x^\star$ and the second inequality uses $\widehat{\Delta}_{\varphi, \mathcal{I}}(\hat{x}) = 0$.
Rearranging gives the first inequality of the lemma.
Using this result, we get
$$
\begin{aligned}
\Delta_\mathcal{I}(x) - \widehat{\Delta}_{\varphi, \mathcal{I}}(x)
&=
\mathcal{R}_\mathcal{I}(x^\star) - \mathcal{R}_\mathcal{I}(x)
- \widehat{\mathcal{R}}_{\varphi, \mathcal{I}} (\hat{x}) + \widehat{\mathcal{R}}_{\varphi, \mathcal{I}}(x) \\
&\leq
\mathcal{R}_{\mathcal{I}}(x^\star) - \mathcal{R}_{\mathcal{I}}(x)
- \widehat{\mathcal{R}}_{\varphi, \mathcal{I}} (x^\star) + \widehat{\mathcal{R}}_{\varphi, \mathcal{I}}(x) \\
&\leq
\frac{1}{2} \widehat{\Delta}_{\varphi, \mathcal{I}}(x^\star) + \frac{1}{2} \widehat{\Delta}_{\varphi, \mathcal{I}}(x) + 2 c_0 \mu_m \\
&\leq
\frac{1}{2} (2\Delta_\mathcal{I}(x^\star) + 4 c_0 \mu_m) + \frac{1}{2} \widehat{\Delta}_{\varphi, \mathcal{I}}(x) + 2 c_0 \mu_m \\
&=
\frac{1}{2} \widehat{\Delta}_{\varphi, \mathcal{I}}(x) + 4 c_0 \mu_m
\end{aligned}
$$
where the first inequality uses the optimality of $\hat{x}$ and the second inequality uses the bound~(\ref{eqn:norestart-reward-bound}) and the last equality uses $\Delta_\mathcal{I}(x^\star) = 0$.
Rearranging gives the second inequality of the lemma.

\end{proof}
For the rest of the analysis, we define $\mu_\mathcal{I} \coloneqq c_1 (\vert \mathcal{I} \vert / E)^{-1/2}$ so that $\mu_j = \mu_{\mathcal{B}(j)}$.

\begin{lemma}\label{lemma:no-false-alarm}
Assume the event \hyperref[definition:event1]{$\textsc{Event}_1$} holds.
Consider an epoch $i$ that starts at time $\tau_i$.
If $V_{[\tau_i, t]} \leq \mu_{[\tau_i, t]}$ holds for some time $t \geq \tau_i$, then no restart is triggered in $[\tau_i, t]$.
\end{lemma}

\begin{proof}
It is enough to show that none of the end of replay intervals that lie within $[\tau_i, t]$ trigger a restart when running the change detection test~(\ref{eqn:change-detection}).
Suppose $\mathcal{S}$ is the replay schedule in a block $j$.
Suppose $s$ is the end of a replay interval $(m, \mathcal{I}) \in \mathcal{S}$ with $\mathcal{I} \subseteq [\tau_i, t]$.
Then, by Lemma~\ref{lemma:reward-estimate-concentration-gap} and Lemma~\ref{lemma:gap-estimate-concentration-replay} (which hold under $\textsc{Event}_1$), we have for any $k < j$ that
$$
\begin{aligned}
\widehat{\Delta}_{\varphi, \mathcal{I}}(x) &\leq 2 \Delta_\mathcal{I}(x) + c_0 \mu_m + 4 V_{[\tau_i, s]} \\
&\leq
2 \Delta_{\mathcal{C}(k)}(x) + c_0 \mu_m + 8 V_{[\tau_i, s]} \\
&\leq
4 \widehat{\Delta}_{\varphi, \mathcal{C}(k)}(x) + 8 V_{\mathcal{C}(k)} + 2 c_0 \mu_k + c_0 \mu_m + 8 V_{[\tau_i, s]} \\
&\leq
4 \widehat{\Delta}_{\varphi, \mathcal{C}(k)}(x) + 3 c_0 \mu_{m \wedge k} + 16 V_{[\tau_i, t]} \\
&\leq
4 \widehat{\Delta}_{\varphi, \mathcal{C}(k)}(x) + 4 c_0 \mu_{m \wedge k}
\end{aligned}
$$
where the second inequality uses Lemma~\ref{lemma:gaps} and the last inequality uses $V_{[\tau_i, t]} \leq \mu_{[\tau_i, t]} \leq \mu_m \leq \mu_{m \wedge k}$.
Similarly, we have
$$
\begin{aligned}
\widehat{\Delta}_{\varphi, \mathcal{C}(k)}(x)
&\leq 2 \Delta_{\mathcal{C}(k)}(x) + c_0 \mu_k + 4 V_{\mathcal{C}(k)} \\
&\leq
2 \Delta_\mathcal{I}(x) + c_0 \mu_k + 8 V_{[\tau_i, t]} \\
&\leq
4 \widehat{\Delta}_{\varphi, \mathcal{I}}(x) + 8 V_{[\tau_i, s]} + 2 c_0 \mu_m + c_0 \mu_k + 8 V_{[\tau_i, t]} \\
&\leq
4 \widehat{\Delta}_{\varphi, \mathcal{I}}(x) + 3 c_0 \mu_{m \wedge k} + 16 V_{[\tau_i, t]} \\
&\leq
4 \widehat{\Delta}_{\varphi, \mathcal{I}}(x) + 4 c_0 \mu_{m \wedge k}.
\end{aligned}
$$
Hence, no restart is triggered by the replay interval $(m, \mathcal{I})$.
Since this holds for any $(m, \mathcal{I}) \in \mathcal{S}$, proof is complete.
\end{proof}

\begin{definition}[Excess regret]
Let $\mathcal{J}$ be an interval, not necessarily a replay interval, that lies in a block $\mathcal{B}(j)$ with $j \geq 1$ in an epoch $i$.
We define the \textit{excess regret} of $\mathcal{J}$ with respect to a feature mapping $\varphi$ as
$$
\zeta_{\varphi, \mathcal{J}} = \max_{x \in \mathcal{X}} \left(
\Delta_{\mathcal{J}}(x) - 8 \widehat{\Delta}_{\varphi, \mathcal{C}(j - 1)}(x)
\right).
$$
\end{definition}

\begin{lemma}\label{lemma:change-detection}
Assume \hyperref[definition:event1]{$\textsc{Event}_1$} holds.
Let $\mathcal{J}$ be an interval that lies within a block $\mathcal{B}(j)$ with $V_\mathcal{J} \leq \mu_\mathcal{J}$ and $\zeta_{\varphi, \mathcal{J}} > D_1 \mu_\mathcal{J}$ where $D_1 = 25 c_0$.
Then, there exists an index $m^\star \in \{0, \dots, j\}$ such that $D_1 \mu_{m^\star + 1} < \zeta_{\varphi, \mathcal{J}} \leq D_1 \mu_{m^\star}$ and $2^{m^\star} E < \vert \mathcal{J} \vert$.
Moreover, any replay interval $\mathcal{I}$ of index $m^\star$ with $\mathcal{I} \subseteq \mathcal{J}$ triggers a restart.
\end{lemma}

\begin{proof}
We show that there exists $m^\star$ such that $D_1 \mu_{m^\star + 1} < \zeta_{\varphi, \mathcal{J}} \leq D_1 \mu_{m^\star}$.
By the definition of the excess regret, we have $\zeta_{\varphi, \mathcal{J}} \leq \max_{x \in \mathcal{X}} \Delta_{\mathcal{J}}(x) \leq 2 \leq D_1 \mu_0$.
Also, by the assumption that $\zeta_{\varphi, \mathcal{J}} > D_1 \mu_\mathcal{J} \geq D_1 \mu_j$ where the last inequality follows since $\mathcal{J} \subseteq \mathcal{B}(j)$, we have $D_1 \mu_j < \zeta_{\varphi, \mathcal{J}} \leq D_1 \mu_0$.
It follows that there exists $m^\star \in \{0, \dots, j\}$ such that $D_1 \mu_{m^\star + 1} < \zeta_{\varphi, \mathcal{J}} \leq D_1 \mu_{m^\star}$.
Also, such $m^\star$ satisfies $D_1 \mu_\mathcal{J} < \zeta_{\varphi, \mathcal{J}} \leq D_1 \mu_{m^\star}$ and it follows that $\vert \mathcal{J} \vert > 2^{m^\star} E$ as desired.

Now, we show that any replay interval $\mathcal{I} \subseteq \mathcal{J}$ of index $m^\star$ determined above triggers a restart.
We argue by contradiction.
Suppose that no restart is triggered after running a replay interval $\mathcal{I} \subseteq \mathcal{J}$ of index $m^\star$.
By the definition of the excess regret, there exists $x' \in \mathcal{X}$ such that $\zeta_{\varphi, \mathcal{J}} = \Delta_\mathcal{J}(x') - 8 \widehat{\Delta}_{\varphi, \mathcal{C}(j - 1)}(x')$.
Hence, by Lemma~\ref{lemma:gaps}, we have
$$
\begin{aligned}
\Delta_\mathcal{I}(x')
&\geq \Delta_\mathcal{J}(x') - 2 V_\mathcal{J} \\
&\geq 8 \widehat{\Delta}_{\varphi, \mathcal{C}(j - 1)}(x') + \zeta_{\varphi, \mathcal{J}} - 2 \mu_\mathcal{J} \\
&> 8 \widehat{\Delta}_{\varphi, \mathcal{C}(j - 1)}(x') + D_1 \mu_{m^\star + 1} - 2 \mu_\mathcal{I}.
\end{aligned}
$$
Moreover, by Lemma~\ref{lemma:gap-estimate-concentration-replay-no-restart}, we have $\Delta_{\mathcal{I}}(x') \leq 2 \widehat{\Delta}_{\varphi, \mathcal{I}}(x') + 4 c_0 \mu_{m^\star}$ under $\textsc{Event}_1$.
Rearranging the lower bound and the upper bound of $\Delta_\mathcal{I}(x')$ we just found, we get
$$
\widehat{\Delta}_{\varphi, \mathcal{I}}(x') > 4 \widehat{\Delta}_{\varphi, \mathcal{C}(j - 1)}(x') + \frac{D_1}{2} \mu_{m^\star + 1} - 2 c_0 \mu_{m^\star} - \mu_\mathcal{I}
\geq 4 \widehat{\Delta}_{\varphi, \mathcal{C}(j - 1)}(x') + 4 c_0 \mu_{m^\star}
$$
which must have triggered a restart by the test (\ref{eqn:change-detection}).
This contradicts the assumption that no restart is triggered, completing the proof.
\end{proof}

\subsection{Replay schedule}

In this subsection, we analyze the behavior of the replay schedule.
Consider a replay schedule $\mathcal{S}$ for a block $\mathcal{B}(j)$ in an epoch $i$.
The following lemma shows that the sum of the errors $\mu_{m_t}$ over the block $\mathcal{B}(j)$ when following the schedule $\mathcal{S}$ is similar to the sum of the errors when using the latest strategy over the entire block.

\begin{lemma}\label{lemma:replay-schedule}
With probability at least $1 - \delta$, for any block $\mathcal{B}(j)$ in any epoch $i$ defined by ADA-OPKB, we have
$$
\sum_{t \in \mathcal{B}(j)} \mu_{m_t} = \widetilde{\mathcal{O}}(\vert \mathcal{B}(j) \vert \mu_j) = \widetilde{\mathcal{O}}(\sqrt{2^j} \gamma_T \log N).
$$
\end{lemma}

\begin{proof}
Consider a block $\mathcal{B}(j)$ in an epoch $i$ and its replay schedule $\mathcal{S}$.
Then,
\begin{equation}
\sum_{t \in \mathcal{B}(j)} \mu_{m_t}
= \mathcal{O}\left(\sum_{t \in \mathcal{B}(j)} 2^{-m_t / 2} \right)
= \mathcal{O}\left(\sum_{m = 0}^j 2^{-m / 2} \sum_{t \in \mathcal{B}(j)} \mathbb{I}\{ m_t = m \} \right). \label{eqn:replay-schedule-intermediate}
\end{equation}
Note that the sum $\sum_{t \in \mathcal{B}(j)} \mathbb{I}\{m_t = m \}$ counts the number of times the replay index $m$ is chosen when following the schedule $\mathcal{S}$.
This sum is bounded by the sum of lengths of all replay intervals of index $m$ in $\mathcal{S}$.
Since a replay interval of index $m$ has length $2^m E$, the maximum possible number of replay intervals of index $m$ is $\vert \mathcal{B}(j) \vert / (2^m E) = 2^{j - m}$.
Denote by $Z_k^{(m)}$, $k = 1, \dots, 2^{j - m}$ a Bernoulli random variable that indicates whether the $k$-th candidate replay interval of index $m$ is scheduled in $\mathcal{S}$.
By the replay scheduling algorithm (Algorithm~\ref{alg:replay-schedule}) used by ADA-OPKB, $Z_k^{(m)}$ are independent with success probability $p = \sqrt{2^{m - j}}$.
Hence, with probability at least $1 - \frac{\delta}{4T (\log_2 T)^2}$, we have
$$
\sum_{t \in \mathcal{B}(j)} \mathbb{I}\{ m_t = m \}
\leq
(2^m E) \sum_{k = 1}^{2^{j - m}} Z_k^{(m)}
\leq
\widetilde{\mathcal{O}}(E \sqrt{2^{j + m}})
$$
where we use the Hoeffding's inequality to bound
$$
\sum_{k = 1}^{2^{j - m}} Z_k^{(m)}
\leq 2^{j - m} p + \sqrt{\frac{2^{j - m} \log(T (\log_2 T)^2 / \delta)}{2}}
= \widetilde{\mathcal{O}}(\sqrt{2^{j - m}})
$$
with probability at least $1 - \frac{\delta}{T (\log_2 T)^2}$.
Applying a union bound over the possible choices of replay index $m$, we can further bound (\ref{eqn:replay-schedule-intermediate}) with probability at least $1 - \frac{\delta}{T \log_2 T}$ by
$$
\sum_{t \in \mathcal{B}(j)} \mu_{m_t}
=
\mathcal{O}\left(\sum_{m = 0}^j 2^{-m / 2} \sum_{t \in \mathcal{B}(j)} \mathbb{I}\{ m_t = m \} \right)
\leq
\widetilde{\mathcal{O}} \left(
j E \sqrt{2^j}
\right)
\leq
\widetilde{\mathcal{O}} \left(
\sqrt{2^j \gamma_T \log N}
\right)
$$
where we use the fact that the block index $j$ is bounded by $\log_2 (T / E)$.
Applying a union bound over all possible choices of the starting time of $\mathcal{B}(j)$ and the block index $j$ completes the proof.
\end{proof}

\subsection{Regret of an interval}\label{subsection:interval-regret}

\begin{lemma}
With probability at least $1 - \delta$, for all intervals $\mathcal{J} \subseteq [T]$, we have
\begin{equation}\label{eqn:reward-azuma-hoeffding}
\sum_{t \in \mathcal{J}} (r_t(x_t^\star) - r_t(x_t))
\leq \sum_{t \in \mathcal{J}} (r_t(x_t^\star) - \mathbb{E}_t [r_t(x_t)]) + \sqrt{8 \vert \mathcal{J} \vert \log(T^2 / \delta)}
\end{equation}
where $x_t^\star = \argmax_{x \in \mathcal{X}} r_t(x)$.
\end{lemma}
\begin{proof}
The result follows by applying the Azuma-Hoeffding inequality on the martingale difference sequence $\{ \mathbb{E}_t[r_t(x_t)] - r_t(x_t) \}_{t \in \mathcal{J}}$, using the fact that $\vert \mathbb{E}_t [ r_t(x_t)] - r_t(x_t) \vert \leq 2$.
\end{proof}

\begin{definition}[$\textsc{Event}_2$]
Define $\textsc{Event}_2$ as the event that the bound (\ref{eqn:reward-azuma-hoeffding}) holds for all intervals $\mathcal{J} \subseteq [T]$.
\end{definition}

\begin{lemma}\label{lemma:interval-regret}
With probability at least $1 - \delta$, for any interval $\mathcal{J}$ that lies in any block $\mathcal{B}(j)$ with $j \geq 1$ in any epoch $i$, the regret in any sub-interval $\mathcal{J}' \subseteq \mathcal{J}$ is bounded by
$$
\textsc{Reg}_{\mathcal{J}'} \leq \mathcal{O}\left(\sum_{t \in \mathcal{J}'} \mu_{m_t} + \vert \mathcal{J}' \vert \mu_{\mathcal{J}'} + \vert \mathcal{J}' \vert V_\mathcal{J} + \vert \mathcal{J}' \vert \zeta_{\varphi, \mathcal{J}} \mathbb{I} \{ \zeta_{\varphi, \mathcal{J}} > D_1 \mu_\mathcal{J} \} \right)
$$
where $D_1 = 25c_0$.
\end{lemma}

\begin{proof}
Fix epoch $i$ and consider an interval $\mathcal{J}$ that lies within a block $j$.
Under $\textsc{Event}_2$, we have
$$
\begin{aligned}
\sum_{t \in \mathcal{J}'} \left(
  r_t(x_t^\star) - r_t(x_t)
\right)
&\leq
\sum_{t \in \mathcal{J}'} \left(
  r_t(x_t^\star) - \mathbb{E}_t[r_t(x_t)]
\right)
+ \sqrt{8 \vert \mathcal{J}' \vert \log(4T^2 / \delta)} \\
&=
\sum_{t \in \mathcal{J}'}
\sum_{x \in \mathcal{X}}
P_t(x) \Delta_t(x)
+ \sqrt{8 \vert \mathcal{J}' \vert \log(4T^2 / \delta)} \\
&\leq
\mathcal{O} \left(
\sum_{t \in \mathcal{J}'}
\sum_{x \in \mathcal{X}}
Q^{(m_t)}(x) \Delta_t(x)
+ \sum_{t \in \mathcal{J}'} \mu_{m_t}
+ \vert \mathcal{J}' \vert \mu_{\mathcal{J}'}
\right)
\end{aligned}
$$
where the last inequality uses $\mu_{\mathcal{J}'} = \mathcal{O}(1 / \sqrt{\vert \mathcal{J}' \vert / E}) = \mathcal{O}(1 / \sqrt{\vert \mathcal{J}' \vert / \log(T / \delta)})$, $P_t = (1 - \mu_{m_t}) Q^{(m_t)} + \mu_{m_t} \pi_\mathcal{X}$ and $\vert \Delta_t(x) \vert \leq 2$.
The first term in the bound above can be bounded by
$$
\begin{aligned}
\sum_{x \in \mathcal{X}}
Q^{(m_t)}(x) \Delta_t(x)
&\leq
\sum_{x \in \mathcal{X}} Q^{(m_t)}(x) \Delta_\mathcal{J}(x) + 2 V_\mathcal{J} \\
&\leq
\sum_{x \in \mathcal{X}} Q^{(m_t)}(x) \left(
  8 \widehat{\Delta}_{\varphi, \mathcal{C}(j - 1)}(x) + \zeta_{\varphi, \mathcal{J}}
\right) + 2 V_\mathcal{J} \\
&\leq
8 \sum_{x \in \mathcal{X}} Q^{(m_t)}(x) \left(
  4 \widehat{\Delta}_{\varphi, \mathcal{C}(m_t - 1)}(x)
  + 4 c \mu_{m_t - 1}
\right)
+ \zeta_{\varphi, \mathcal{J}}
+ 2 V_\mathcal{J} \\
&\leq
\mathcal{O} \left(
\frac{(1 + \alpha) \gamma_{\varphi, T}}{\beta_{m_t - 1}}
+ \mu_{m_t - 1} + \zeta_{\varphi, \mathcal{J}} + V_\mathcal{J}
\right)
\leq
\mathcal{O} \left(
\mu_{m_t} + \zeta_{\varphi, \mathcal{J}} + V_\mathcal{J}
\right)
\end{aligned}
$$
where the first inequality uses Lemma~\ref{lemma:gaps}, the second inequality uses the definition of $\zeta_{\varphi, \mathcal{J}}$ and the third inequality uses the fact that no restart is triggered by the block $\mathcal{B}(j - 1)$.
The second to last inequality uses Lemma~\ref{lemma:optimization-problem}.
We can further bound the regret as
$$
\begin{aligned}
\sum_{t \in \mathcal{J}'} \left(
  r_t(x_t^\star) - r_t(x_t)
\right)
&\leq
\widetilde{\mathcal{O}}\left(
\sum_{t \in \mathcal{J}'} \mu_{m_t} + \vert \mathcal{J}' \vert \zeta_{\varphi, \mathcal{J}} + \vert \mathcal{J}' \vert V_\mathcal{J} + \vert \mathcal{J}' \vert \mu_{\mathcal{J}'}
\right).
\end{aligned}
$$
Noting that $\vert \mathcal{J}' \vert \zeta_{\varphi, \mathcal{J}} \leq \vert \mathcal{J}' \vert \zeta_{\varphi, \mathcal{J}} \mathbb{I}\{ \zeta_{\varphi, \mathcal{J}} > D_1 \mu_\mathcal{J} \} + D_1 \mu_\mathcal{J} \vert \mathcal{J}' \vert$ completes the proof.
\end{proof}

\subsection{Regret of a block}

In this section, we fix a block $\mathcal{J}$ in an epoch $i$ and bound its regret.
The strategy is to partition the block into nearly-stationary intervals to use the interval regret bound we found in Lemma~\ref{lemma:interval-regret}, and argue that the change detection test does not allow the non-stationarity to accumulate without being detected.
First, we show that given an arbitrary interval $\mathcal{J}$, we can partition it into nearly-stationary intervals $\mathcal{J}_1, \dots, \mathcal{J}_\ell$ while controlling the size of the partition $\ell$.
For ease of exposition, we write $\gamma_T = \gamma_{\varphi, T}$.

\begin{lemma}\label{lemma:interval-partition}
Given an interval $\mathcal{J}$, we can partition it into a set of intervals $\{ \mathcal{J}_1, \dots, \mathcal{J}_\ell \}$ such that $V_{\mathcal{J}_k} \leq \mu_{\mathcal{J}_k}$ for all $k = 1, \dots, \ell$ and
$$
\ell \leq \min\left\{
L_\mathcal{J}, \left(\frac{1}{2}\gamma_T \log(C_1 N / \delta)\right)^{-1/3} V_\mathcal{J}^{2/3} \vert \mathcal{J} \vert^{1/3} + 1
\right\}.
$$
\end{lemma}

\begin{proof}
Following the same procedures described in the proof of Lemma 5 by \textcite{chen_new_2019} and the proof of Lemma 19 by \textcite{wei_non-stationary_2021}, we partition $\mathcal{J}$ by taking intervals consecutively from the beginning of $\mathcal{J}$ in a greedy manner.
Specifically, given that the first $k - 1$ intervals we took are $\mathcal{J}_1 = [s_1, e_1], \dots, \mathcal{J}_{k - 1} = [s_{k - 1}, e_{k - 1}]$, we take the next interval $\mathcal{J}_k = [s_k, e_k]$ with $s_k = e_{k - 1} + 1$ (or set $s_k$ to the beginning of $\mathcal{J}$ if $k = 1$) that satisfies $V_{[s_k, e_k]} \leq \mu_{[s_k, e_k]}$ and $V_{[s_k, e_k]} > \mu_{[s_k, e_k + 1]}$.
In other words, $\mathcal{J}_k$ is the maximal interval that immediately follows $\mathcal{J}_{k - 1}$ and satisfies $V_{[s_k, e_k]} \leq \mu_{[s_k, e_k]}$.
We repeat this greedy procedure until the end of $\mathcal{J}$ is reached.

We first show that the number of intervals $\ell$ in the partition obtained by the procedure must satisfy $\ell \leq L_\mathcal{J}$.
To see this, consider the partition $\{ \mathcal{I}_1, \dots, \mathcal{I}_{L_\mathcal{J}} \}$ of $\mathcal{J}$ where each $\mathcal{I}_k$, $k = 1, \dots, L_\mathcal{J}$ are stationary, that is $V_{\mathcal{I}_k} = 0$.
Then, each interval $\mathcal{J}_k$ must contain at least one end point
of a stationary interval.
Otherwise, $\mathcal{J}_k$ must end within some stationary interval $\mathcal{I}_{k'}$ and does not contain the end point of the stationary interval.
This contradicts with the greedy procedure because when the procedure constructs $\mathcal{J}_k$, it must have taken time steps at least until the end point of the stationary interval $\mathcal{I}_{k'}$ since doing so does not affect $V_{\mathcal{J}_k}$.
Also, each end point of the stationary interval is contained in exactly one of $\mathcal{J}_1, \dots, \mathcal{J}_{\ell}$.
Hence, there is a surjection from $\{ \mathcal{I}_1, \dots, \mathcal{I}_{L_\mathcal{J}} \}$ to $\{ \mathcal{J}_1, \dots, \mathcal{J}_{\ell} \}$ and it follows that $\ell \leq L_\mathcal{J}$.

Now, we show that $\ell \leq (\frac{1}{2}\gamma_T \log(C_1 N / \delta))^{-1/3} V_\mathcal{J}^{2/3} \vert \mathcal{J} \vert^{1/3} + 1$.
Recall that for any interval $\mathcal{I}$, $\mu_{\mathcal{I}}$ is defined as $\mu_{\mathcal{I}} = \frac{1}{2} \sqrt{E} \vert \mathcal{I} \vert^{-1/2}$ where $E = \lceil 4 \gamma_T \log(C_1 N / \delta) \rceil$.
Hence,
$$
V_\mathcal{J}
\geq \sum_{k = 1}^{\ell - 1} V_{[s_k, e_k]}
> \sum_{k = 1}^{\ell - 1}
\mu_{[s_k, e_k + 1]}
=
\frac{\sqrt{E}}{2} \sum_{k = 1}^{\ell - 1} (\vert \mathcal{J}_k \vert + 1)^{-1/2}
\geq
\sqrt{\frac{1}{2} \gamma_T \log(C_1 N / \delta)}
\sum_{k = 1}^{\ell - 1} \vert \mathcal{J}_k \vert^{-1/2}
$$
where the second inequality follows by the greedy procedure and the last inequality follows since $(x + 1)^{-1/2} \geq (2x)^{-1/2}  = \frac{1}{\sqrt{2}} x^{-1/2}$ for all $x \geq 1$.
By the H\"older's inequality, we have
$$
\ell - 1
\leq
\left(
\sum_{k = 1}^{\ell - 1}
\left\vert \mathcal{J}_k \right\vert^{-1/2}
\right)^{2/3}
\left(
\sum_{k = 1}^{\ell - 1}
\left\vert \mathcal{J}_k \right\vert
\right)^{1/3}
\leq
\left(\frac{1}{2} \gamma_T \log(C_1 N / \delta)\right)^{-1/3}
V_\mathcal{J}^{2/3}
\vert \mathcal{J} \vert^{1/3}
$$
and the desired bound for $\ell$ follows.
This completes the proof.
\end{proof}

Note that the block $\mathcal{B}(j)$ defined by ADA-OPKB spans exactly $2^j \cdot E$ time steps whether the block runs past the time horizon or a restart is triggered before the block ends.
Denote by $\mathcal{B}'(j)$ the actual block run as part of epoch $i$ before a restart is triggered or the time horizon is reached.

\begin{lemma}\label{lemma:block-regret}
Consider a block $\mathcal{B}(j)$ in an epoch $i$ defined by ADA-OPKB.
Let $\mathcal{B}'(j)$ be the actual block run as part of epoch $i$ before a restart is triggered or the time horizon is reached.
With probability at least $1 - 2 \delta$, we have
$$
\textsc{Reg}_{\mathcal{B}'(j)} = \widetilde{\mathcal{O}}\left(
\min \left\{
(\gamma_T \log N) \sqrt{2^j L_{\mathcal{B}'(j)}} ,
(\gamma_T \log N) V_{\mathcal{B}'(j)}^{1/3} (2^j)^{2/3}
+ (\gamma_T \log N) \sqrt{2^j}
\right\}
\right).
$$
\end{lemma}
\begin{proof}
For ease of exposition, we suppress the subscript $\varphi$ and write $\gamma_T$ and $\zeta_\mathcal{J}$ instead of $\gamma_{\varphi, T}$ and $\zeta_{\varphi, \mathcal{J}}$.
Assume \hyperref[definition:event1]{$\textsc{Event}_1$} holds.
Using the procedure described in Lemma~\ref{lemma:interval-partition}, we partition $\mathcal{B}(j)$ into $\mathcal{J}_1, \dots, \mathcal{J}_\ell$ such that $V_{\mathcal{J}_k} \leq \mu_{\mathcal{J}_k}$ for all $k = 1, \dots, \ell$.
Let $\mathcal{J}'_1, \dots, \mathcal{J}'_{\ell'}$ be the non-empty intervals $\mathcal{J}'_k = \mathcal{J}_k \cap \mathcal{B}'(j)$ that partition $\mathcal{B}'(j)$.
Using the interval regret bound in Lemma~\ref{lemma:interval-regret} with the fact that $\mathcal{J}'_k \subseteq \mathcal{J}_k$ and using $V_{\mathcal{J}_k} \leq \mu_{\mathcal{J}_k} \leq \mu_{\mathcal{J}'_k}$, we get
\begin{equation} \label{eqn:block-regret-bound}
\textsc{Reg}_{\mathcal{B}'(j)}
= \sum_{k = 1}^{\ell'} \textsc{Reg}_{\mathcal{J}'_k}
\leq
\mathcal{O}\left(
\sum_{t \in \mathcal{B}'(j)} \mu_{m_t} + \sum_{k = 1}^{\ell'} \vert \mathcal{J}'_k \vert \mu_{\mathcal{J}'_k} + \sum_{k = 1}^{\ell'} \vert \mathcal{J}'_k \vert \zeta_{\mathcal{J}_k} \mathbb{I}\{ \zeta_{\mathcal{J}_k} > D_1 \mu_{\mathcal{J}_k} \}
\right).
\end{equation}
The first term can be bounded using Lemma~\ref{lemma:replay-schedule} by $\sum_{t \in \mathcal{B}'(j)} \mu_{m_t} \leq \widetilde{\mathcal{O}} (\sqrt{2^j} \gamma_T \log N)$ with probability at least $1 - \delta$.

The second term can be bounded using $\mu_{\mathcal{I}} = \mathcal{O}(\sqrt{E / \vert \mathcal{I} \vert})$ as
$$
\sum_{k = 1}^{\ell'} \vert \mathcal{J}'_k \vert \mu_{\mathcal{J}'_k}
\leq
\mathcal{O}\left(
\sum_{k = 1}^{\ell'} \sqrt{\vert \mathcal{J}_k' \vert E}
\right)
\leq
\mathcal{O}\left(
\sqrt{\ell' \vert \mathcal{B}'(j) \vert E}
\right)
\leq
\mathcal{O}\left(
E \sqrt{\ell' 2^j}
\right)
$$
where the second inequality uses Cauchy-Schwarz and the last inequality uses $\vert \mathcal{B}'(j) \vert \leq \vert \mathcal{B}(j) \vert$.

The third term is bounded using Lemma~\ref{lemma:change-detection} which shows that there exists $m_k^\star \in \{0, \dots, j\}$ with
\begin{equation} \label{eqn:m_star}
2^{m^\star_k} < \vert \mathcal{J}_k \vert / E \quad\text{and}\quad D_1 \mu_{m_k^\star + 1} < \zeta_{\mathcal{J}_k} \leq D_1 \mu_{m_k^\star}
\end{equation}
such that running a replay interval of index $m^\star_k$ inside $\mathcal{J}_k$ triggers a restart.
Denote by $n_k^{(m)}$ the number of replay intervals of index $m$ that can be scheduled completely inside $\mathcal{J}'_k$.
Then,
\begin{equation}\label{eqn:n_k}
n_k^{(m)} \geq (\vert \mathcal{J}'_k \vert - 3 \cdot 2^m E)_+ / (2^m E)
\end{equation}
for all $m = 0, \dots, j$ where $(\cdot)_+ = \max\{0, \cdot\}$.
Hence,
$$
\begin{aligned}
\vert \mathcal{J}'_k \vert \zeta_{\mathcal{J}_k} \mathbb{I} \{ \zeta_{\mathcal{J}_k} > D_1 \mu_{\mathcal{J}_k} \}
&\leq
3 \cdot 2^{m_k^\star} E D_1 \mu_{m_k^\star} + (\vert \mathcal{J}'_k \vert - 3 \cdot 2^{m_k^\star} E)_+ D_1 \mu_{m_k^\star} \mathbb{I}\{ \zeta_{\mathcal{J}_k} > D_1 \mu_{\mathcal{J}_k} \} \\
&\leq
\mathcal{O}\left(
\sqrt{E \vert \mathcal{J}_k \vert} + E \sqrt{2^{m^\star_k}} n_k^{(m_k^\star)} \mathbb{I}\{ \zeta_{\mathcal{J}_k} > D_1 \mu_{\mathcal{J}_k} \}
\right)
\end{aligned}
$$
where the first inequality uses $\zeta_{\mathcal{J}_k} \leq D_1 \mu_{m_k^\star}$ stated in (\ref{eqn:m_star}) and the second inequality uses $2^{m_k^\star} \mu_{m_k^\star} = \mathcal{O}(\sqrt{2^{m_k^\star}}) \leq \mathcal{O}(\sqrt{\vert \mathcal{J}_k \vert / E})$ which follows by (\ref{eqn:m_star}) and the lower bound of $n_k^{(m)}$ shown in (\ref{eqn:n_k}).
Summing over $k = 1, \dots, \ell'$ and writing the event $\{ \zeta_{\mathcal{J}_k} > D_1 \mu_{\mathcal{J}_k} \}$ as $A_k$ for convenience, we get
\begin{align}
\sum_{k = 1}^{\ell'} \vert \mathcal{J}'_k \vert \zeta_{\mathcal{J}_k} \mathbb{I}\{ A_k \}
&\leq
\mathcal{O}\left(
\sum_{k = 1}^{\ell'} \sqrt{E \vert \mathcal{J}_k \vert}
+ E \sum_{k = 1}^{\ell'} \sqrt{2^{m_k^\star}} n_k^{(m_k^\star)} \mathbb{I}\{ A_k \}
\right) \nonumber \\
&\leq
\mathcal{O} \left(
E \sqrt{\ell' 2^j}
+ E \sum_{m = 0}^j \sqrt{2^m} \sum_{k = 1}^{\ell'} n_k^{(m)} \mathbb{I}\{A_k, m_k^\star = m\}
\right) \label{eqn:block-regret-intermediate}
\end{align}
where the second inequality uses Cauchy-Schwarz and $\sum_{k = 1}^{\ell'} \vert \mathcal{J}_k \vert \leq \vert \mathcal{B}(j) \vert = 2^j E$.
Denoting by $Z_{k, l}^{(m)}$, $l = 1, \dots, n_k^{(m)}$ the Bernoulli random variable that indicates whether the $l$-th replay interval among the $n_k^{(m)}$ candidate replay intervals within $\mathcal{J}'_k$ is scheduled, we have
$$
\begin{aligned}
\sum_{k = 1}^{\ell'} n_k^{(m)} \mathbb{I}\{ A_k, m_k^\star = m \}
&=
\sum_{k = 1}^{\ell'} n_k^{(m)} \mathbb{I}\{
A_k, m_k^\star = m, Z^{(m)}_{k, 1} = 0, \dots, Z^{(m)}_{k, n_k^{(m)}} = 0
\} \\
&\leq
\sum_{k = 1}^{\ell'} n_k^{(m)} \mathbb{I}\{ Z_{k, 1}^{(m)} = 0, \dots, Z_{k, n_k^{(m)}}^{(m)} = 0 \}
\leq
\widetilde{\mathcal{O}}(\sqrt{2^{j - m}})
\end{aligned}
$$
where the first equality follows under $\textsc{Event}_1$ by Lemma~\ref{lemma:change-detection} since if any of $Z_{k, l}^{(m)} = 1$ for $m = m_k^\star$, a restart must have been triggered before reaching the end of $\mathcal{J}'_k$.
The last inequality follows since the second to last term is a geometric random variable with trials $Z_{1, 1}^{(m)}, \dots, Z_{1, n_1^{(m)}}^{(m)}, \dots, Z_{\ell', n_{\ell'}^{(m)}}^{(m)}$ with success probability $\sqrt{2^{m - j}}$, which is bounded with probability at least $1 - \delta$ by $\widetilde{\mathcal{O}}(\sqrt{2^{j - m}})$.
We can further bound the third term (\ref{eqn:block-regret-intermediate}) by
$
\sum_{k = 1}^{\ell'} \vert \mathcal{J}_k' \vert \zeta_{\mathcal{J}_k} \mathbb{I}\{A_k \}
\leq
\widetilde{\mathcal{O}} \left(
E \sqrt{\ell' 2^j}
\right)
$
where we use $j \leq \log_2 (T / E)$.
Summing the three bounds we found for the terms in (\ref{eqn:block-regret-bound}), we get $\textsc{Reg}_{\mathcal{B}'(j)} = \widetilde{\mathcal{O}}(E \sqrt{\ell' 2^j})$.
Bounding $\ell'$ using Lemma~\ref{lemma:interval-partition} completes the proof.
\end{proof}

\subsection{Proof of Theorem~\ref{theorem:ada-opkb}}

\begin{lemma}\label{lemma:epoch-bound}
Assume \hyperref[definition:event1]{$\textsc{Event}_1$} holds. The number of epochs $H$ when running ADA-OPKB is bounded by
$$
H \leq \min \left\{
L_T, \left( \frac{1}{2} \gamma_T \log(C_1 N / \delta) \right)^{-1/3} V_T^{2/3} T^{1/3} + 1
\right\}.
$$
\end{lemma}

\begin{proof}
Let $\{ \mathcal{J}_k \}_{k = 1}^\ell$ with $\ell \leq \min\left\{ L_T, (\frac{1}{2} \gamma_T \log(C_0 N / \delta))^{-1/3} V_T^{2/3} T^{1/3} + 1 \right\}$ be a partition of $[T]$ where $V_{\mathcal{J}_k} \leq \mu_{\mathcal{J}_k}$ for all $k = 1, \dots, \ell$.
Such a partition exists by Lemma~\ref{lemma:interval-partition}.
Let $\mathcal{E}_1, \dots, \mathcal{E}_H$ be all the intervals spanned by the epochs in $[1, T]$.
Note that if an epoch $i$ starts inside an interval $\mathcal{J}_k$, then the epoch must continue at least until the end of $\mathcal{J}_k$ since the total variation in $\mathcal{E}_i \cap \mathcal{J}_k$ is upper bounded by $V_{\mathcal{E}_i \cap \mathcal{J}_k} \leq V_{\mathcal{J}_k} \leq \mu_{\mathcal{J}_k} \leq \mu_{\mathcal{E}_i \cap \mathcal{J}_k}$ and no restart is triggered under $\textsc{Event}_1$ in $\mathcal{E}_i \cap \mathcal{J}_k$ due to Lemma~\ref{lemma:no-false-alarm}.
Hence, each $\mathcal{E}_i$ contains the end point of at least one interval $\mathcal{J}_k$.
Also, trivially,
the end point of each $\mathcal{J}_k$
is contained in exactly one epoch.
Hence, there is a surjection from
$\{ \mathcal{J}_1, \dots, \mathcal{J}_\ell \}$
to $\{ \mathcal{E}_1, \dots, \mathcal{E}_H \}$
and it follows that $H \leq \ell$.
This completes the proof.
\end{proof}

\begin{lemma}\label{lemma:epoch-regret}
Given an epoch $i$ in ADA-OPKB, let $\mathcal{E}_i$ be the interval spanned by the epoch.
Then, with high probability, we have
$$
\textsc{Reg}_{\mathcal{E}_i} = \widetilde{\mathcal{O}}\left(
\min \left\{
\sqrt{\gamma_T L_{\mathcal{E}_i} \vert \mathcal{E}_i \vert \log N},
(\gamma_T V_{\mathcal{E}_i} \log N)^{1/3} \vert \mathcal{E}_i \vert^{2/3} + \sqrt{\gamma_T \log N \vert \mathcal{E}_i \vert}
\right\}
\right).
$$
\end{lemma}

\begin{proof}
Let $J_i$ be the index of the last block in epoch $i$.
Then $\mathcal{E}_i = \cup_{j = 0}^{J_i} \mathcal{B}'(j)$ where $\mathcal{B}'(j) = \mathcal{B}(j) \cap \mathcal{E}_i$ and $\mathcal{B}(j)$ is the $j$-th block defined by ADA-OPKB in epoch $i$.
Since $\vert \mathcal{E}_i \vert = \sum_{j = 0}^{J_i} \vert \mathcal{B}'(j) \vert = E(1 + 2 + \cdots + 2^{J_i - 1}) + \vert \mathcal{B}'(J_i) \vert \geq (2^{J_i} - 1) E$, we have $2^{J_i} - 1 \leq \vert \mathcal{E}_i \vert / E$.
Hence, using the regret bound of $\mathcal{B}'(j)$ in terms of $L_{\mathcal{B}'(j)}$ provided by Lemma~\ref{lemma:block-regret} and $\textsc{Reg}_{\mathcal{E}_i} = \sum_{j = 0}^{J_i} \textsc{Reg}_{\mathcal{B}'(j)}$, we have with high probability that
$$
\textsc{Reg}_{\mathcal{E}_i}
\leq
\widetilde{\mathcal{O}}\left(
E \sum_{j = 0}^{J_i} \sqrt{2^j L_{\mathcal{B}'(j)}}
\right)
\leq
\widetilde{\mathcal{O}}\left(
E \sqrt{2^{J_i} - 1} \sqrt{L_{\mathcal{E}_i} + J_i}
\right)
\leq
\widetilde{\mathcal{O}}\left(
\sqrt{E \vert \mathcal{E}_i \vert L_{\mathcal{E}_i}}
\right)
$$
where the second inequality uses Cauchy-Schwarz and the fact that $\sum_{j = 0}^{J_i} L_{\mathcal{B}'(j)} \leq L_{\mathcal{E}_i} + J_i$, and the last inequality uses $J_i \leq \log_2 (T / E)$.
This shows the first bound of the lemma.

To show the second bound in terms of $V_{\mathcal{E}_i}$, we use the regret bound of $\mathcal{B}'(j)$ in terms of $V_{\mathcal{B}'(j)}$ provided by Lemma~\ref{lemma:block-regret} to get with high probability that
$$
\begin{aligned}
\textsc{Reg}_{\mathcal{E}_i}
&\leq
\widetilde{\mathcal{O}} \left(
E \sum_{j = 0}^{J_i} V_{\mathcal{B}'(j)}^{1/3} (2^j)^{2/3}
\right)
+ \widetilde{\mathcal{O}} \left(
E \sum_{j = 0}^{J_i} \sqrt{2^j}
\right) \\
&\leq
\widetilde{\mathcal{O}} \left(
E V_{\mathcal{E}_i}^{1/3} (2^{J_i} - 1)^{2/3}
\right)
+
\widetilde{\mathcal{O}} \left(
E \sqrt{J_i (2^J_i - 1)}
\right)
\leq
\widetilde{\mathcal{O}} \left(
E^{1/3} V_{\mathcal{E}_i}^{1/3} \vert \mathcal{E}_i \vert^{2/3}
+ \sqrt{E \vert \mathcal{E}_i \vert}
\right)
\end{aligned}
$$
where the second inequality uses the H\"older's inequality and the Cauchy-Schwarz inequality, and the third inequality uses the bound $2^{J_i} - 1 \leq \vert \mathcal{E}_i \vert / E$ and $J_i \leq \log_2 (T / E)$.
This completes the proof.
\end{proof}

Now, we are ready to prove Theorem~\ref{theorem:ada-opkb}.
To bound the total dynamic regret, we bound the sum of the epoch regret bounds and use the bound on the number of epochs as shown below.

\begin{refproof}[Theorem~\ref{theorem:ada-opkb}]
Using the epoch regret bound in Lemma~\ref{lemma:epoch-regret} and the bound on the number of epochs $H$ in Lemma~\ref{lemma:epoch-bound}, we can bound $\textsc{Reg}_T = \sum_{i = 1}^H \textsc{Reg}_{\mathcal{E}_i}$ as follows.
First, using the epoch regret bound in terms of $L_{\mathcal{E}_i}$, we get with high probability that
$$
\textsc{Reg}_T
\leq
\widetilde{\mathcal{O}}\left(
\sqrt{E} \sum_{i = 1}^H \sqrt{L_{\mathcal{E}_i} \vert \mathcal{E}_i \vert}
\right)
\leq
\widetilde{\mathcal{O}}\left(
\sqrt{E} \sqrt{L_T + H} \sqrt{T}
\right)
\leq
\widetilde{\mathcal{O}}\left(
\sqrt{E L_T T}
\right)
$$
where the second inequality uses Cauchy-Schwarz and the last inequality uses the bound $H \leq L_T$.

Now, using the epoch regret bound in terms of $V_{\mathcal{E}_i}$, we get
$$
\textsc{Reg}_T
\leq
\widetilde{\mathcal{O}}\left(
E^{1/3} \sum_{i = 1}^H V_{\mathcal{E}_i}^{1/3} \vert \mathcal{E}_i \vert^{2/3}
+ \sqrt{E} \sum_{i = 1}^H \sqrt{\vert \mathcal{E}_i \vert}
\right)
\leq
\widetilde{\mathcal{O}}\left(
E^{1/3} V_T^{1/3} T^{2/3} + \sqrt{E H T}
\right)
$$
where the second inequality uses the H\"older's inequality and the Cauchy-Schwarz inequality.
Further bounding by $H \leq \mathcal{O}(1 + E^{-1/3} V_T^{2/3} T^{1/3})$ completes the proof.
\end{refproof}

\section{Analysis of OPNN}\label{section:analysis-of-opnn}

In this section, we prove the following theorem that states a regret bound for the OPNN algorithm under the general stationary bandit setting.

\begin{theorem}[c.f. Theorem~\ref{theorem:opkb}]\label{theorem:opnn}
Consider the general stationary bandit setting described in Section~\ref{section:problem-statement}.
Assume Assumption~\ref{assumption1} and Assumption~\ref{assumption2} hold.
If we run the OPNN algorithm using a neural network with width $m$ and depth $L$, the dynamic regret is bounded by
$$
\textsc{Reg}_T \leq \widetilde{\mathcal{O}}\left(
\sqrt{\gamma_T T \log N}
\right)
$$
with probability at least $1 - \delta$ where $\gamma_T$ is the maximum information gain with respect to the neural tangent kernel of the neural network as long as $m \geq \text{poly}(T, L, N, \lambda^{-1}, \lambda_0^{-1}, \log(1 / \delta))$.
\end{theorem}

The key insight for the analysis of OPNN is that in the infinite network width regime, OPNN is equivalent to OPKB with the neural tangent kernel $\bm{H}$ defined as follows.

\begin{definition}[\textcite{jacot_neural_2018}, \textcite{arora_exact_2019}]
Consider a fully connected neural network of depth $L$ with the ReLU activation function $\sigma$.
For all $a_i, a_j \in \mathcal{X}$, define covariance matrices $\Sigma^{(l)}$ and derivative covariance matrices $\dot{\Sigma}^{(l)}$ for $l = 0, \dots, L$ recursively as follows:
$$
\begin{aligned}
\Sigma^{(0)}_{ij} &= \langle a_i, a_j \rangle, \quad
\bm{A}^{(l)}_{ij} = \begin{pmatrix}
\Sigma^{(l - 1)}_{ii} & \Sigma^{(l - 1)}_{ii} \\
\Sigma^{(l - 1)}_{ji} & \Sigma^{(l - 1)}_{jj}
\end{pmatrix} \\
\Sigma^{(l)}_{ij} &= 2 \mathbb{E}_{(u, v) \sim N(0, \bm{A}^{(l)}_{ij})} [ \sigma(u), \sigma(v) ] \\
\dot{\Sigma}^{(l)}_{ij} &= 2 \mathbb{E}_{(u, v) \sim N(0, \bm{A}^{(l)}_{ij})} [ \dot{\sigma}(u), \dot{\sigma}(v) ]
\end{aligned}
$$
where $\dot{\sigma}$ is the derivative of the activation function.
The neural tangent kernel $\bm{H}$ for the network is defined as
$$
\bm{H}_{ij} = \sum_{l = 1}^L \left(
\Sigma^{(l - 1)}_{ij} \cdot \prod_{l' = l}^L \dot{\Sigma}^{(l')}_{ij}
\right).
$$
\end{definition}

For the analysis, we make the following technical assumptions.

\begin{assumption} \label{assumption1}
The neural tangent kernel matrix is positive definite with $\bm{H} \succcurlyeq \lambda_0 \bm{I}$ for some $\lambda_0 > 0$.
\end{assumption}

The assumption that the neural tangent kernel matrix is positive definite is a mild assumption commonly made when analyzing neural networks \parencite{du_gradient_2019, arora_exact_2019}.
The assumption is satisfied, for example, as long as the actions are normalized to $\Vert a_i \Vert_2 = 1$ for all $i \in [N]$ and no two actions in $\mathcal{X}$ are parallel \parencite{du_gradient_2019-1}.

We impose regularity assumption on the reward functions as follows.

\begin{assumption} \label{assumption2}
For all $t \in [T]$, we have $\sqrt{2 \bm{r}_t \bm{H}^{-1} \bm{r}_t} \leq B$ for some constant $B$ where $\bm{r}_t = (r_t(a_1), \dots, r_t(a_N))$ is the vector of reward function values at time $t$.
We assume that the learner knows the upper bound $B$ and scales the problem so that $\sqrt{2 \bm{r}_t \bm{H}^{-1} \bm{r}_t} \leq 1$ for all $t \in [T]$.
\end{assumption}

This assumption is common in the neural bandits literature \parencite{zhou_neural_2020, zhang_neural_2020, gu_batched_2021}.
As discussed by \textcite{zhou_neural_2020}, if $\bm{r}_t$ lies in the RKHS $\mathcal{H}$ induced by the neural tangent kernel, the quantity $\sqrt{\bm{r}_t \bm{H}^{-1} \bm{r}_t}$ is upper bounded by the RKHS norm $\Vert \bm{r}_t \Vert_\mathcal{H}$.
In this sense, the upper bound on $\sqrt{2 \bm{r}_t \bm{H}^{-1} \bm{r}_t}$ imposes regularity on the reward functions.

\subsection{NTK theory from previous work}

We first review results related to the neural tangent kernel in previous work.
The lemmas provided in this subsection are adapted from \textcite{zhou_neural_2020} which uses results in \textcite{allen-zhu_convergence_2019} and \textcite{arora_exact_2019}.

\begin{lemma}[Lemma B.5 by \textcite{zhou_neural_2020}]\label{lemma:gradient-deviation}
With high probability, we have
$$
\Vert g(x; \bm{W}) - g(x; \bm{W}^{(0)}) \Vert_2 \leq \mathcal{O}\left(
\sqrt{\log m} T^{1/6} m^{-1/6} \lambda^{-1/6} L^3 \Vert g(x; \bm{W}^{(0)}) \Vert_2
\right)
$$
for all $\Vert \bm{W} - \bm{W}^{(0)} \Vert_2 \leq 2 \sqrt{T / (m \lambda)}$ as long as $m \geq \text{poly}(T, L, \lambda^{-1})$.
\end{lemma}

\begin{lemma}[Lemma B.6 by \textcite{zhou_neural_2020}]\label{lemma:gradient-bound}
With high probability, we have
$$
\Vert g(x; \bm{W}) \Vert_2 \leq \mathcal{O}(\sqrt{mL})
$$
for all $\Vert \bm{W} - \bm{W}^{(0)} \Vert_2 \leq 2 \sqrt{T / (m\lambda)}$ and $x \in \mathcal{X}$ as long as $m \geq \text{poly}(T, L, \lambda^{-1})$.
\end{lemma}

\begin{lemma}[Lemma 5.2 by \textcite{zhou_neural_2020}]
Let $W$ be a parameter trained by $\textsc{TrainNN}$ (Algorithm~\ref{alg:train-nn}).
Then, with probability at least $1 - \delta$, we have
$$
\Vert \bm{W} - \bm{W}^{(0)} \Vert_2 \leq \mathcal{O}(\sqrt{T / (m \lambda)})
$$
as long as $m \geq \text{poly}(T, L, \lambda^{-1}, \log(1 / \delta))$.
\end{lemma}

\begin{lemma}[Lemma 5.1 by \textcite{zhou_neural_2020}]\label{lemma:nn-realizability}
With probability at least $1 - \delta$, there exists $\bm{W}_t^\star \in \mathbb{R}^p$ such that
$$
r_t(x) = \langle g(x; \bm{W}^{(0)}), \bm{W}_t^\star - \bm{W}^{(0)} \rangle
\quad \text{and} \quad
\sqrt{m} \Vert \bm{W}_t^\star - \bm{W}^{(0)} \Vert_2 \leq \sqrt{2 \bm{r}_t^T \bm{H}^{-1} \bm{r}_t} \leq \sqrt{2 N / \lambda_0}
$$
for all $t \in [T]$ and $x \in \mathcal{X}$ as long as $m \geq \text{poly}(T, L, N, \lambda_0^{-1}, \log(1 / \delta))$ where $\bm{H}$ is the neural tangent kernel matrix, $\lambda_0$ is the minimum eigenvalue of $\bm{H}$ and $\bm{r}_t = [r_t(a_1) \cdots r_t(a_N)]^T$.
\end{lemma}

\begin{lemma}[Lemma B.1 by \textcite{zhou_neural_2020}]
Let $\bm{H}$ be the neural tangent kernel matrix and let $\bm{G}_0 = [g(a_1 ; \bm{W}^{(0)}) \cdots g(a_N ; \bm{W}^{(0)})] / \sqrt{m} \in \mathbb{R}^{p \times N}$.
Then, with probability at least $1 - \delta$, we have
$$
\Vert \bm{G}_0^T \bm{G}_0 - \bm{H} \Vert_F \leq N \epsilon
$$
as long as $m \geq \text{poly}(L, \epsilon^{-1}, \log(1 / \delta))$.
\end{lemma}

\begin{lemma}[Lemma B.2 by \textcite{zhou_neural_2020}]\label{lemma:lazy-training}
Let $\bm{W}$ be the parameter trained by the algorithm $\textsc{TrainNN}$ with learning rate $\eta \leq \mathcal{O}((m \lambda + TmL)^{-1})$ and initial weight $\bm{W}^{(0)}$.
Then, with probability at least $1 - \delta$, we have
$$
\Vert \bm{W} - \bm{W}^{(0)} \Vert_2 \leq 2 \sqrt{T / (m \lambda)}
$$
as long as the network width satisfies $m \geq \text{poly}(T, L, \lambda^{-1}, \log(1 / \delta))$.
\end{lemma}

\subsection{More NTK theory}

\begin{lemma} \label{lemma:kernel-drift}
Let $\bm{W}$ be close to the initial weight $\bm{W}^{(0)}$ such that $\Vert \bm{W} - \bm{W}^{(0)} \Vert_2 \leq 2 \sqrt{T / (m \lambda)}$.
Let $\bm{G} = [g(a_1 ; \bm{W}) \cdots g(a_N ; \bm{W})] / \sqrt{m}$ and $\bm{G}_0 = [g(a_1 ; \bm{W}^{(0)}) \cdots g(a_N ; \bm{W}^{(0)})] / \sqrt{m}$.
Then, with probability at least $1 - \delta$, as long as $m \geq \text{poly}(T, L, \lambda^{-1})$, we have
$$
\Vert \bm{G}_0^T \bm{G}_0 - \bm{G}^T \bm{G} \Vert_F \leq \mathcal{O}\left(
m^{- 1 / 3} (\log m) T^{1/3} N^2 \lambda^{-1/3} L^8
\right).
$$
\end{lemma}

\begin{proof}
For ease of exposition, we write $g(\cdot) = g(\cdot ; \bm{W})$ and $g_0(\cdot) = g(\cdot ; \bm{W}^{(0)})$.
Note that
$$
\begin{aligned}
\Vert \bm{G}_0^T \bm{G}_0 - \bm{G}^T \bm{G} \Vert_F^2
&= \frac{1}{m^2} \sum_{i, j \in [N]} \left(\langle g(a_i), g(a_j) \rangle - \langle g_0(a_i), g_0(a_j) \rangle \right)^2 \\
&= \frac{1}{m^2} \sum_{i, j \in [N]} \left(
\langle g(a_i) - g_0(a_i), g(a_j) \rangle - \langle g_0(a_i), g_0(a_j) - g(a_j) \rangle
\right)^2 \\
&\leq
\frac{2}{m^2} \sum_{i, j \in [N]} \left( \Vert g(a_i) - g_0(a_i) \Vert_2^2 \Vert g(a_j) \Vert_2^2 + \Vert g_0(a_i) \Vert_2^2 \Vert g_0(a_j) - g(a_j) \Vert_2^2 \right) \\
&\leq
\frac{2}{m^2} N^2 \mathcal{O}\left(
(\log m) T^{1/3} m^{5 / 3} \lambda^{-1/3} L^8
\right) \\
&=
\mathcal{O}\left(
(\log m) T^{1/3} N^2 m^{- 1 / 3} \lambda^{-1/3} L^8
\right) \\
\end{aligned}
$$
where the first inequality follows by Cauchy-Schwarz and $(a - b)^2 \leq 2a^2 + 2 b^2$, and the second inequality follows by Lemma~\ref{lemma:gradient-deviation} and Lemma~\ref{lemma:gradient-bound}.
\end{proof}

\begin{lemma}\label{lemma:information-gain-drift-bound}
Consider a weight $\bm{W}$ close to the initial weight $\bm{W}^{(0)}$ such that $\Vert \bm{W} - \bm{W}^{(0)} \Vert_2 \leq 2 \sqrt{T / (m \lambda)}$.
Let $\varphi$ and $\varphi^{(0)}$ be feature mappings equivalent to $g(\cdot ; \bm{W}) / \sqrt{m}$ and $g(\cdot ; \bm{W}^{(0)}) / \sqrt{m}$ respectively.
Then, as long as $m \geq \text{poly}(T, L, \lambda^{-1})$, we have
$$
\gamma_{\varphi, T} \leq \gamma_{\varphi^{(0)}, T} + \mathcal{O}\left(
m^{- 1 / 3} (\log m) T^{4/3} N^{5/2} \lambda^{-1/3} L^8
\right).
$$
\end{lemma}

\begin{proof}
Let $\bm{G} = [g(a_1 ; \bm{W}) \cdots g(a_N ; \bm{W})] / \sqrt{m}$ and $\bm{G}_0 = [g(a_1 ; \bm{W}^{(0)}) \cdots g(a_N ; \bm{W}^{(0)})] / \sqrt{m}$.
For ease of exposition, write $S = S_{g(\cdot ; \bm{W}) / \sqrt{m}}$ and $S_0 = S_{g(\cdot ; \bm{W}^{(0)}) / \sqrt{m}}$.
Then, $S(TP / \sigma, 1) = \frac{T}{\sigma} \bm{G} D_P \bm{G}^T + I_p$ and $S_0(TP / \sigma, 1) = \frac{T}{\sigma} \bm{G}_0 D_P \bm{G}_0^T + I_p$ where $D_P = \text{diag}(P(a_1), \dots, P(a_N)) \in \mathbb{R}^{N \times N}$.
Using the Sylvester's identity $\log \det(I + AB) = \log \det(I + BA)$, we have
$$
\begin{aligned}
\log \det &S(TP / \sigma, 1) = \log \det ((T/\sigma) G D_P G^T + I_p)
= \log \det ((T / \sigma) D_P G^T G + I_N) \\
&= \log \det ((T / \sigma) D_P G_0^T G_0 + I_N + (T / \sigma) D_P G^T G - (T / \sigma) D_P G_0^T G_0) \\
&\leq
\log \det ((T / \sigma) D_P G_0^T G_0 + I_N) \\
&\quad\quad+ \langle
((T / \sigma) D_P G_0^T G_0 + I_N)^{-1},
(T / \sigma) D_P G^T G - (T / \sigma) D_P G_0^T G_0)
\rangle \\
&\leq
\log \det ((T / \sigma) D_P G_0^T G_0 + I_N) \\
&\quad\quad+ 
\Vert ((T / \sigma) D_P G_0^T G_0 + I_N)^{-1} \Vert_F
\Vert (T / \sigma) D_P G^T G - (T / \sigma) D_P G_0^T G_0)
\Vert_F \\
\end{aligned}
$$
where the first inequality follows by the concavity of $\log\det(\cdot)$ and the last inequality follows by Cauchy-Schwarz.
To bound the second term on the right hand side, we can bound the first factor by
$$
\Vert ((T / \sigma) D_P G_0^T G_0 + I_N)^{-1} \Vert_F
\leq \sqrt{N}\Vert ((T / \sigma) D_P G_0^T G_0 + I_N)^{-1} \Vert_2
\leq \sqrt{N}
$$
where the first inequality uses the identity $\Vert A \Vert_F \leq \sqrt{N} \Vert A \Vert_2$ for $A \in \mathbb{R}^{N \times N}$ and the second inequality uses $(T / \sigma) D_P G_0^T G_0 + I_N \succcurlyeq I_N$.
Also, we can bound the second factor by
$$
(T / \sigma) \Vert D_P (G^T G - G_0^T G_0) \Vert_F \leq
(T / \sigma) \Vert D_P \Vert_2 \Vert G^T G - G_0^T G_0 \Vert_F
\leq
(T / \sigma) \Vert G^T G - G_0^T G_0 \Vert_F
$$
where the first inequality uses the identity $\Vert AB \Vert_F \leq \Vert A \Vert_2 \Vert B \Vert_F$ and the second inequality uses $\Vert D_P \Vert_2 \leq 1$.
Using the bound of the two factors and using Lemma~\ref{lemma:kernel-drift} for bounding $\Vert G^T G - G_0^T G_0 \Vert_F$ gives
$$
\log \det S(TP / \sigma, 1) \leq \log \det S_0(TP / \sigma, 1)
+ \mathcal{O} \left(
(\log m) T^{4/3} N^{5/2} m^{- 1 / 3} \lambda^{-1/3} L^8
\right).
$$
Maximizing over $P \in \mathcal{P}_\mathcal{X}$ on the left hand side and denoting the maximizer by $P^\star$, we get
$$
\begin{aligned}
\gamma_{\varphi, T}
&\leq
\log \det S_0 (TP^\star / \sigma, 1)
+ \mathcal{O} \left(
(\log m) T^{4/3} N^{5/2} m^{- 1 / 3} \lambda^{-1/3} L^8
\right) \\
&\leq
\gamma_{\varphi^{(0)}, T}
+ \mathcal{O} \left(
(\log m) T^{4/3} N^{5/2} m^{- 1 / 3} \lambda^{-1/3} L^8
\right)
\end{aligned}
$$
where the second inequality follows since $\gamma_{\varphi^{(0)}, T}$ maximizes $\log \det S_0(TP / \sigma, 1)$ over $P \in \mathcal{P}_\mathcal{X}$.
This completes the proof.

\end{proof}

\begin{lemma}\label{lemma:lazy-training-error}
Let $\bm{W}$ be a parameter returned by the $\textsc{TrainNN}$ algorithm.
For each $t \in [T]$, let $\bm{W}_t^\star$ be a parameter that satisfies $r_t(x) = \langle g(x; \bm{W}^{(0)}), \bm{W}_t^\star - \bm{W}^{(0)} \rangle$ for all $x \in \mathcal{X}$ and $\Vert \bm{W}_t^\star - \bm{W}^{(0)} \Vert_2 \leq \sqrt{2 \bm{r}_t^T \bm{H}^{-1} \bm{r}_t / m}$.
Such a parameter $\bm{W}_t^\star$ exists by Lemma~\ref{lemma:nn-realizability}.
Then, with probability at least $1 - \delta$, we have
$$
\vert r_t(x) - \langle g(x; \bm{W}), \bm{W}_t^\star - \bm{W}^{(0)} \rangle \vert \leq \epsilon
$$
for all $t \in [T]$ and $x \in \mathcal{X}$ as long as $m \geq \text{poly}(T, L, N, \lambda^{-1}, \lambda_0^{-1}, \log(1 / \delta), \epsilon^{-1})$.
\end{lemma}

\begin{proof}
By Lemma~\ref{lemma:lazy-training}, we have with high probability that $\Vert \bm{W} - \bm{W}^{(0)} \Vert_2 \leq 2 \sqrt{T / (m \lambda)}$ as long as $m \geq \text{poly}(T, L, \lambda^{-1}, \log(1 / \delta))$ which allows us to use Lemma~\ref{lemma:gradient-deviation} and Lemm~\ref{lemma:gradient-bound} to get
$$
\Vert g(x; \bm{W}^{(0)}) - g(x; \bm{W}) \Vert_2 \leq \mathcal{O}( \sqrt{\log m} T^{1/6} m^{1/3} \lambda^{-1/6} L^{7/2} )
$$
with high probability as long as $m \geq \text{poly}(T, L, \lambda^{-1}, \log(1 / \delta))$.
Hence, we have
$$
\begin{aligned}
\vert r_t(x) - \langle g(x; \bm{W}), \bm{W}_t^\star - \bm{W}^{(0)} \rangle \vert
&=
\vert \langle g(x; \bm{W}^{(0)}) - g(x; \bm{W}), \bm{W}_t^\star - \bm{W}^{(0)} \rangle \vert \\
&\leq
\Vert g(x; \bm{W}^{(0)}) - g(x; \bm{W}) \Vert_2 \Vert \bm{W}_t^\star - \bm{W}^{(0)} \Vert_2 \\
&\leq
\mathcal{O}(\sqrt{\log m} T^{1/6} m^{-1/6} \lambda^{-1/6} \lambda_0^{-1/2} N^{1/2} L^{7/2})
\leq \epsilon
\end{aligned}
$$
for all $t \in [T]$ and $x \in \mathcal{X}$ as long as $m \geq \text{poly}(T, L, N, \lambda^{-1}, \lambda_0^{-1}, \epsilon^{-1})$.
This completes the proof.
\end{proof}

\begin{lemma} \label{lemma:information-gain-upper-bound}
Let $\bm{W}$ be close to the initial weight $\bm{W}^{(0)}$ such that $\Vert \bm{W} - \bm{W}^{(0)} \Vert_2 \leq 2 \sqrt{T / (m \lambda)}$.
Let $\varphi : \mathcal{X} \rightarrow \mathbb{R}^N$ be a feature mapping equivalent to $g(\cdot ; \bm{W}) / \sqrt{m}$.
Then, we have
$$
\gamma_{\varphi, T} = \mathcal{O}(N \log (TL)).
$$
\end{lemma}

\begin{proof}
Using the identity $\det A \leq (\frac{1}{N} \Tr(A))^N$ for positive semi-definite $A \in \mathbb{R}^{N \times N}$, we have
$$
\begin{aligned}
\log \det S_\varphi(\sigma^{-1} TP, 1)
&\leq N \log \left( \frac{1}{N} \Tr(S_\varphi(\sigma^{-1} TP, 1)) \right) \\
&= N \log \left(\frac{T}{\sigma N} \sum_{x \in \mathcal{X}} P(x) \Vert \varphi(x) \Vert_2^2 + 1 \right) \\
&\leq N \log \left(\frac{T}{\sigma N} \sum_{x \in \mathcal{X}} P(x) \mathcal{O}(L) + 1 \right) \\
&= \mathcal{O}(N \log (TL))
\end{aligned}
$$
where the second inequality uses $\Vert \varphi(x) \Vert_2^2 = \Vert g(x ; \bm{W}) \Vert_2^2 / m$ due to equivalence and Lemma~\ref{lemma:gradient-bound}.
This completes the proof.
\end{proof}

\subsection{Concentration bound on reward estimates}

In this subsection, we prove the following concentration bound for the reward estimate analogous to Lemma~\ref{lemma:reward-estimate-concentration}.

\begin{lemma}[c.f. Lemma~\ref{lemma:reward-estimate-concentration}] \label{lemma:reward-estimate-concentration-nn}
Let $\mathcal{I} \subseteq [T]$ be a time interval.
Let $m_t$ be the strategy index used at time $t$ by OPNN and $\varphi^{(m)}$ the feature mapping computed by OPNN using data in the cumulative block $\mathcal{C}(m - 1)$.
Let $\bm{\varphi} = \{ \varphi_t \}_{t \in \mathcal{I}}$ be the sequence of feature mappings used by OPNN where $\varphi_t = \varphi^{(m_t)}$.
If $j$ is such that $m_t \leq j$ for all $t \in \mathcal{I}$, then with probability at least $1 - \frac{2\delta}{C}$, we have for all $x \in \mathcal{X}$ that
$$
\begin{aligned}
\vert \widehat{\mathcal{R}}_{\bm{\varphi}, \mathcal{I}}(x) - \mathcal{R}_\mathcal{I}(x) \vert
&\leq
\frac{\xi_j}{\vert \mathcal{I} \vert} \sum_{t \in \mathcal{I}} \Vert \varphi_t(x) \Vert_{S_{\varphi_t}(P_t, \sigma / T)^{-1}}^2 + \frac{\log(CN/\delta)}{\xi_j \vert \mathcal{I} \vert} \\
&\quad\quad+ \sqrt{\frac{\sigma}{T}} \Vert \varphi_t(x) \Vert_{S_{\varphi_t}(P_t, \sigma / T)^{-1}} + \epsilon
\end{aligned}
$$
as long as $m \geq \text{poly}(T, L, N, \lambda^{-1}, \lambda_0^{-1}, \log(1 / \delta), \epsilon^{-1})$ where $\xi_j = \mu_j / (4 \gamma_{\varphi^{(0)}, T})$ and $\widehat{\mathcal{R}}_{\bm{\varphi}, \mathcal{I}} \coloneqq \frac{1}{\vert \mathcal{I} \vert} \sum_{t \in \mathcal{I}} \widehat{\mathcal{R}}_{\varphi_t, t}$.
\end{lemma}

First, we show the following distributional properties of the IPS estimator analogous to Lemma~\ref{lemma:one-point-estimator}.

\begin{lemma}[c.f. Lemma~\ref{lemma:one-point-estimator}]\label{lemma:one-point-estimator-nn}
Let $m_t$ be the strategy index used by OPNN at time $t$ and let $\varphi_t : \mathcal{X} \rightarrow \mathbb{R}^N$ be the feature mapping equivalent to $g(\cdot ; \bm{W}^{(m_t)}) / \sqrt{m}$ used by OPNN at time $t$.
Let $P_t = P^{(m_t)}$ be the strategy used at time $t$.
Then, with probability at least $1 - \delta$, the IPS estimator $\widehat{\mathcal{R}}_{\varphi_t, t}(x)$ satisfies
$$
\begin{aligned}
\vert \widehat{\mathcal{R}}_{\varphi_t, t}(x) \vert &\leq \frac{\gamma_{\varphi_t, T}}{\mu_{m_t}} \\
\vert \mathbb{E}_t[\widehat{\mathcal{R}}_{\varphi_t, t}(x)] - r_{t}(x) \vert &\leq
\sqrt{\frac{\sigma}{T}} \Vert \varphi_t(x) \Vert_{S_{\varphi_t}(P_t, \sigma / T)^{-1}}  + \epsilon \\
\Var_t [ \widehat{\mathcal{R}}_{\varphi_t, t}(x)] &\leq \Vert \varphi_t(x) \Vert^2_{S_{\varphi_t}(P_t, \sigma / T)^{-1}}
\end{aligned}
$$
for all $x \in \mathcal{X}$ and $t \in [T]$ as long as $m \geq \text{poly}(T, L, N, \lambda^{-1}, \lambda_0^{-1}, \log(1 / \delta), \epsilon^{-1})$.
\end{lemma}

\begin{proof}
The first and the third inequalities follow by  the same proof as in Lemma~\ref{lemma:one-point-estimator}.
We focus on the second inequality.
By Lemma~\ref{lemma:lazy-training-error}, with probability at least $1 - \delta$, there exists $\xi_{t, x}$ with $\vert \xi_{t, x} \vert \leq \epsilon_0$ such that $r_t(x) = \langle g(x; \bm{W}^{(m_t)}), \bm{W}_t^\star - \bm{W}^{(0)} \rangle + \xi_{t, x}$ for all $t \in [T]$ and $x \in \mathcal{X}$ as long as $m \geq \text{poly}(T, L, N, \lambda^{-1}, \lambda_0^{-1}, \log(1 / \delta), \epsilon_0^{-1})$.
Writing $S_t = S_{g(\cdot; \bm{W}^{(m_t)}) / \sqrt{m}}(P_t, \sigma / T)$ and $g_t(\cdot) = g(\cdot ; \bm{W}^{(m_t)}) / \sqrt{m}$ for convenience, we have
$$
\begin{aligned}
\mathbb{E}_t [ \widehat{\mathcal{R}}_{\varphi_t, t}(x) ]
&= \mathbb{E}_t [\varphi_t(x)^T S_{\varphi_t} (P_t, \sigma / T)^{-1} \varphi_t (x_t) r_t(x_t)] \\
&= \mathbb{E}_t [g_t(x)^T S_t^{-1} g_t(x_t) (\sqrt{m} g_t(x_t)^T (\bm{W}_t^\star - \bm{W}^{(0)}) + \xi_{t, x})] \\
&= \sqrt{m} g_t(x)^T S_t^{-1} (S_t - (\sigma / T) I) (\bm{W}_t^\star - \bm{W}^{(0)})
+ \mathbb{E}_t [g_t(x)^T S_t^{-1} g_t(x_t) \xi_{t, x}] \\
&=
r_t(x) - \xi_{t, x} - \frac{\sigma \sqrt{m}}{T} g_t(x)^T S_t^{-1} (\bm{W}_t^\star - \bm{W}^{(0)})
+ \mathbb{E}_t [g_t(x)^T S_t^{-1} g_t(x_t) \xi_{t, x}]
\end{aligned}
$$
where the first equality uses the fact that the term with the noise $\eta_t$ vanishes due to independence and the second equality uses Lemma~\ref{lemma:useful-identity}.
Hence, writing $\widetilde{S}_t = S_{\varphi_t}(P_t, \sigma / T)$, we have
$$
\begin{aligned}
\vert \mathbb{E}_t [ \widehat{\mathcal{R}}_{\varphi_t, t}(x)] &- r_t(x) \vert
\leq
\vert \xi_{t, x} \vert
+ \frac{\sigma \sqrt{m}}{T} \vert g_t(x)^T S_t^{-1} (\bm{W}_t^\star - \bm{W}^{(0)}) \vert
+ \mathbb{E}_t [\vert g_t(x)^T S_t^{-1} g_t(x_t) \xi_{t, x} \vert] \\
&\leq
\epsilon_0 + \frac{\sigma \sqrt{m}}{T} \Vert g_t(x) \Vert_{S_t^{-1}} \Vert \bm{W}_t^\star - \bm{W}^{(0)} \Vert_{S_{t}^{-1}} + \epsilon_0 \mathbb{E}_t[ \Vert g_t(x) \Vert_{S_{t}^{-1}} \Vert g_t(x_t) \Vert_{S_{t}^{-1}}] \\
&\leq
\epsilon_0 + \sqrt{\frac{\sigma m}{T}} \Vert \varphi_t(x) \Vert_{\widetilde{S}_t^{-1}} \Vert \bm{W}_t^\star - \bm{W}^{(0)} \Vert_2 + \epsilon_0 \mathbb{E}_t [
\Vert \varphi_t(x) \Vert_{\widetilde{S}_{t}^{-1}} \Vert \varphi_t(x_t) \Vert_{\widetilde{S}_{t}^{-1}}]
\end{aligned}
$$
where the second inequality uses Cauchy-Schwarz and the last inequality uses Lemma~\ref{lemma:useful-identity} and $S_t^{-1} \preccurlyeq (T / \sigma)I$.
Since $P_t = (1 - \mu_{m_t}) P^{(m_t)} + \mu_{m_t} \pi_{\varphi_t, \mathcal{X}} \succcurlyeq \mu_{m_t} \pi_{\varphi_t, \mathcal{X}}$ , we have $\widetilde{S}_t \succcurlyeq \mu_{m_t} S_{\varphi_t} ( \pi_{\varphi_t, \mathcal{X}}, \sigma / T)$ and it follows that
$$
\Vert \varphi_t(x) \Vert_{\widetilde{S}_t^{-1}}^2
\leq \frac{1}{\mu_{m_t}} \Vert \varphi_t(x) \Vert^2_{S_{\varphi_t}(\pi_{\varphi_t, \mathcal{X}}, \sigma / T)^{-1}}  \leq \frac{\gamma_{\varphi_t, T}}{\mu_{m_t}}
\leq C T^{1/2} N \log(TL)
$$
for some constant $C$ where the second inequality follows by Lemma~\ref{lemma:optimal-design} and the last inequality follows by $\mu_j \geq T^{-1/2}$ and Lemma~\ref{lemma:information-gain-upper-bound}. 
Also, by Lemma~\ref{lemma:nn-realizability}, we have $\sqrt{m} \Vert \bm{W}_t^\star - \bm{W}^{(0)} \Vert_2 \leq 1$ with probability at least $1 - \delta$.
Hence, we can further bound the bias term by
$$
\begin{aligned}
\vert \mathbb{E}_t [ \widehat{\mathcal{R}}_{\varphi_t, t}(x)] - r_t(x) \vert
&\leq
\epsilon_0 + \sqrt{\frac{\sigma}{T}} \Vert \varphi_t(x) \Vert_{S_{\varphi_t}(P_t, \sigma / T)^{-1}} + \epsilon_0
C T^{1/2} N \log(TL) \\
&\leq
\sqrt{\frac{\sigma}{T}} \Vert \varphi_t(x) \Vert_{S_{\varphi_t}(P_t, \sigma / T)^{-1}} + \epsilon
\end{aligned}
$$
where we set $\epsilon_0$ sufficiently small such that $\epsilon_0 + \epsilon_0 CT^{1/2} N \log (TL) \leq \epsilon$ and choose $m \geq \text{poly}(T, L, N, \lambda^{-1}, \lambda_0^{-1}, \log(1 / \delta), \epsilon^{-1})$ appropriately.
This completes the proof.
\end{proof}

Using the previous lemma, we are ready to prove Lemma~\ref{lemma:reward-estimate-concentration-nn}.

\begin{refproof}[Lemma~\ref{lemma:reward-estimate-concentration-nn}]
Fix an action $x \in \mathcal{X}$ and consider a martingale difference sequence $\{z_{t, x} \}_{t \in \mathcal{I}}$ where $z_{t, x} = \widehat{\mathcal{R}}_{\varphi_t, t}(x) - \mathbb{E}_t[ \widehat{\mathcal{R}}_{\varphi_t, t}(x)]$.
We can bound $z_{t, x}$ for all $t \in \mathcal{I}$ by
$$
z_{t, x} \leq \vert \widehat{\mathcal{R}}_{\varphi_t, t}(x) \vert + \mathbb{E}_t[ \vert \widehat{\mathcal{R}}_{\varphi_t, t}(x) \vert]
\leq
\frac{2\gamma_{\varphi_t, T}}{\mu_{m_t}}
\leq
\frac{4 \gamma_{\varphi^{(0)}, T}}{\mu_j}
$$
where the second inequality uses Lemma~\ref{lemma:one-point-estimator-nn} to bound $\vert \widehat{\mathcal{R}}_{\varphi_t, t}(x) \vert$ and the last inequality uses Lemma~\ref{lemma:information-gain-drift-bound} to choose $m \geq \text{poly}(T, N, \lambda, L)$ that satisfies $\gamma_{\varphi_t, T} \leq 2 \gamma_{\varphi^{(0)}, T}$.
Also, we have $\Var_t[z_{t, x}] = \Var_t[\widehat{\mathcal{R}}_{\varphi_t, t}(x)] \leq \Vert \varphi_t(x) \Vert^2_{S_{\varphi_t}(P_t, \sigma / T)^{-1}}$ by Lemma~\ref{lemma:one-point-estimator-nn}.
Using the Freedman inequality (Lemma~\ref{lemma:freedman}) on $\{ z_{t, x} \}_{t \in \mathcal{I}}$ we get with probability at least $1 - \frac{\delta}{CN}$ that
$$
\begin{aligned}
\widehat{\mathcal{R}}_{\bm{\varphi}, \mathcal{I}}(x)
&- \mathcal{R}_\mathcal{I}(x)
= \frac{1}{\vert \mathcal{I} \vert} \sum_{t \in \mathcal{I}} (z_{t, x} + \mathbb{E}_t[\widehat{\mathcal{R}}_{\varphi_t, t}(x)] - \mathcal{R}_{\mathcal{I}}(x)) \\
&\leq
\frac{\xi_j}{\vert \mathcal{I} \vert} \sum_{t \in \mathcal{I}} \Vert \varphi_t(x) \Vert^2_{S_{\varphi_t}(P_t, \sigma / T)^{-1}} + \frac{\log(CN / \delta)}{\xi_j \vert \mathcal{I} \vert} + \sqrt{\frac{\sigma}{T}} \Vert \varphi_t(x) \Vert_{S_{\varphi_t}(P_t, \sigma / T)^{-1}} + \epsilon
\end{aligned}
$$
where $\xi_j = \mu_j / (4 \gamma_{\varphi^{(0)}, T})$ and we use Lemma~\ref{lemma:one-point-estimator-nn} to bound the bias term $\mathbb{E}_t[\widehat{\mathcal{R}}_{\varphi^{(m_t)}, t}(x)] - \mathcal{R}_{\mathcal{I}}(x)$.
A union bound over all $x \in \mathcal{X}$ and the reverse case $\mathcal{R}_\mathcal{I}(x) - \widehat{\mathcal{R}}_{\bm{\varphi}, \mathcal{I}}(x)$ completes the proof.
\end{refproof}

\begin{lemma}[c.f. Lemma~\ref{lemma:reward-estimate-concentration-gap}] \label{lemma:reward-estimate-concentration-gap-nn}
Let $m_t$ be the strategy index used at time $t$ by OPNN and $\varphi^{(m)}$ the feature mapping computed by OPNN using data in the cumulative block $\mathcal{C}(m - 1)$.
Let $\bm{\varphi} = \{ \varphi_t \}_{t \in \mathcal{I}}$ be the sequence of feature mappings used by OPNN where $\varphi_t = \varphi^{(m_t)}$.
With high probability, when running the OPNN algorithm, we have for all block indices $j = 0, 1, \dots$ and actions $x \in \mathcal{X}$ that
\begin{align}
\vert \widehat{\mathcal{R}}_{\bm{\varphi}, \mathcal{C}(j)}(x) - \mathcal{R}_{\mathcal{C}(j)}(x) \vert
&\leq \frac{1}{2} \Delta_{\mathcal{C}(j)}(x) + V_{\mathcal{C}(j)} + \frac{c_0}{4} \mu_j \label{eqn:reward-estimate-concentration-nn1}\\
\Delta_{\mathcal{C}(j)}(x)
&\leq 2 \widehat{\Delta}_{\bm{\varphi}, \mathcal{C}(j)}(x) + 4 V_{\mathcal{C}(j)} + c_0 \mu_j \label{eqn:reward-estimate-concentration-nn2}\\
\widehat{\Delta}_{\bm{\varphi}, \mathcal{C}(j)}(x)
&\leq 2 \Delta_{\mathcal{C}(j)}(x) + 4 V_{\mathcal{C}(j)} + c_0 \mu_j \label{eqn:reward-estimate-concentration-nn3}
\end{align}
where $c_0 = (40 + 16 \sqrt{\alpha})$.
\end{lemma}

\begin{proof}
Apart from dealing with the error $\epsilon$ when applying Lemma~\ref{lemma:one-point-estimator-nn} due to the finiteness of the width of the network and bounding $\gamma_{\varphi_t, T} \leq \gamma_{\varphi^{(0)}, T} + \epsilon$ using Lemma~\ref{lemma:information-gain-drift-bound}, the proof is exactly the same as that for Lemma~\ref{lemma:reward-estimate-concentration-gap}.
As for dealing with $\epsilon$, we set $\epsilon \leq 1$ by choosing $m \geq \text{poly}(T, L, N, \lambda^{-1}, \lambda_0^{-1}, \log(1 / \delta), \epsilon^{-1})$ appropriately when applying Lemma~\ref{lemma:one-point-estimator-nn} and Lemma~\ref{lemma:information-gain-drift-bound}.
\end{proof}

Now, we are ready to prove Theorem~\ref{theorem:opnn}.

\begin{refproof}[Theorem~\ref{theorem:opnn}]
The proof is exactly the same as that of Theorem~\ref{theorem:opkb}.
Instead of using Lemma~\ref{lemma:reward-estimate-concentration-gap} as in the proof of Theorem~\ref{theorem:opkb}, we use Lemma~\ref{lemma:reward-estimate-concentration-gap-nn} for the reward estimate concentration bound and the suboptimality gap estimate concentration bound.
\end{refproof}

\section{Analysis of ADA-OPNN}

The analysis of ADA-OPNN is exactly the same as the analysis of ADA-OPKB presented in Section~\ref{section:analysis-ada-opkb} with the following adjustments.
In place of Lemma~\ref{lemma:reward-estimate-concentration} and Lemma~\ref{lemma:reward-estimate-concentration-gap} use Lemma~\ref{lemma:reward-estimate-concentration-nn} and Lemma~\ref{lemma:reward-estimate-concentration-gap-nn}.

\section{Equivalence of feature mappings} \label{section:feature-mapping-equivalence}

Recall that OPKB and ADA-OPKB use a feature mapping equivalent to a feature mapping corresponding to a given kernel.
Also, OPNN and ADA-OPNN use a feature mapping equivalent to the feature mapping induced by the neural network.
In this section, we show that the choice of feature mapping does not affect the algorithm and the analysis.
Note that the algorithm and the analysis depend on the feature mapping $\varphi$ only through the quantities $\Vert \varphi(x) \Vert_{S_\varphi(P, \lambda)^{-1}}^2$ and $\log \det S_\varphi(P, \lambda)$.
The following lemmas show that these quantities are not affected by the choice of the equivalent feature mapping.

\begin{lemma}\label{lemma:useful-identity}
Let $\psi : \mathcal{X} \rightarrow \ell^2$ (or $\psi : \mathcal{X} \rightarrow \mathbb{R}^p$) be a feature mapping.
Let $\varphi : \mathcal{X} \rightarrow \mathbb{R}^N$ be an equivalent feature mapping.
Then, for all $x, x' \in \mathcal{X}$, we have
$$
\varphi(x)^T S_\varphi(P, \lambda)^{-1} \varphi(x') = \psi(x)^T S_\psi(P, \lambda)^{-1} \psi(x').
$$
\end{lemma}
\begin{proof}
We prove the more general case $\psi : \mathcal{X} \rightarrow \ell^2$.
Let $\Phi = [ \varphi(a_1) \cdots \varphi(a_N) ]^T \in \mathbb{R}^{N \times N}$ and $\Psi = [ \psi(a_1) \cdots \psi(a_N) ]^T \in \mathbb{R}^{N \times \infty}$.
The infinite matrix $\Psi$ can be thought of a linear operator $\Psi : \ell^2 \rightarrow \mathbb{R}^N$ with $\Psi(\cdot) = ( \langle \psi(a_1), \cdot \rangle, \dots, \langle \psi(a_N), \cdot \rangle )$.
We denote by $\Psi^T : \mathbb{R}^N \rightarrow \ell^2$ the linear operator with $\Psi^T(w) = \sum_{i = 1}^N w_i \varphi(a_i)$.
By the definition of equivalence of feature mappings, we have $\Phi \Phi^T = \Psi \Psi^T = K$ where $K = [ \langle \psi(x), \psi(x') \rangle ]_{x, x' \in \mathcal{X}}$ is the kernel matrix.
Defining $D_P = \text{diag}(P(a_1), \dots, P(a_N))$, we can write $S_\varphi(P, \lambda) = \Phi^T D_P \Phi + \lambda I_N$ and $S_\psi(P, \lambda) = \Psi^T D_P \Psi + \lambda I$.
Note that
$$
S_\psi(P, \lambda) \Psi^T
= (\Psi^T D_P \Psi + \lambda I) \Psi^T
= \Psi^T(D_P \Psi \Psi^T + \lambda I_N)
= \Psi^T(D_P K + \lambda I_N).
$$
Applying the inverses of $S_\psi(P, \lambda)$ and $(D_P K + \lambda I_N)$ on both sides, we get $\Psi^T (D_PK + \lambda I_N)^{-1} = S_\psi(P, \lambda)^{-1} \Psi^T$
It follows that
$$
\begin{aligned}
\psi(a_i)^T S_\psi(P, \lambda)^{-1} \psi(a_j)
&=
\langle \psi(a_i), S_\psi(P, \lambda)^{-1} \Psi^T e_j \rangle \\
&=
\langle
  \psi(a_i),
  \Psi^T (D_P K + \lambda I_N)^{-1} e_j
\rangle \\
&= \langle \psi(a_i), \Psi^T w \rangle
\end{aligned}
$$
where $e_j \in \mathbb{R}^N$ is the unit vector with $j$th entry 1 and $w = (D_P K + \lambda I_N)^{-1} e_j$.
Since $\langle \psi(a_i), \Psi^T w \rangle = \langle \psi(a_i), \sum_{j = 1}^N w_j \psi(a_j) \rangle = \sum_{j = 1}^N w_j k(a_i, a_j) = e_i^T K w$, it follows by standard matrix algebra that
\begin{align*}
\psi(a_i)^T S_\psi(P, \lambda)^{-1} \psi(a_j)
&= e_i^T K w \\
&= e_i^T \Phi \Phi^T (D_P \Phi \Phi^T + \lambda I_N)^{-1} e_j \\
&= e_i^T \Phi(\Phi^T D_P \Phi + \lambda I_N)^{-1} \Phi^T e_j \\
&= \varphi(a_i)^T S_\varphi(P, \lambda)^{-1} \varphi(a_j)
\end{align*}
for all $1 \leq i, j \leq N$ where the second to last equality uses the fact that $\Phi^T (D_P \Phi \Phi^T + \lambda I_N) = (\Phi^T D_P \Phi + \lambda I_N) \Phi^T$, which implies $\Phi^T (D_P \Phi \Phi^T + \lambda I_N)^{-1} = (\Phi^T D_P \Phi + \lambda I_N)^{-1} \Phi^T$.
This completes the proof.
\end{proof}

\begin{lemma}
Let $\varphi_1 : \mathcal{X} \rightarrow \mathbb{R}^{p_1}$ and $\varphi_2 : \mathcal{X} \rightarrow \mathbb{R}^{p_2}$ be equivalent feature mappings.
Then, we have
$$
\log \det S_{\varphi_1}(P, \lambda) = \log \det S_{\varphi_2}(P, \lambda).
$$
\end{lemma}

\begin{proof}
Let $\Phi_1 = [ \varphi_1(a_1) \cdots \varphi_1(a_N) ]^T \in \mathbb{R}^{N \times p_1}$ and $\Phi_2 = [ \varphi_2(a_1) \cdots \varphi_2(a_N) ]^T \in \mathbb{R}^{N \times p_2}$.
By the definition of equivalence of feature mappings, we have $\Phi_1 \Phi_1^T = \Phi_2 \Phi_2^T = K$ for some kernel matrix $K \in \mathbb{R}^{N \times N}$.
Defining $D_P = \text{diag}(P(a_1), \dots, P(a_N))$, we can write $S_{\varphi_1}(P, \lambda) = \Phi_1^T D_P \Phi_1 + \lambda I_N$ and $S_{\varphi_2}(P, \lambda) = \Phi_2^T D_P \Phi_2 + \lambda I_N$.
Using the Sylvester's determinant identity $\det(AB + I) = \det(BA + I)$, we get
$$
\begin{aligned}
\log \det S_{\varphi_1}(P, \lambda) &= \log \det (\Phi_1^T D_P \Phi_1 + \lambda I_{p_1}) \\
&= \log \det (\Phi_1 \Phi_1^T D_P + \lambda I_N) \\
&= \log \det (\Phi_2 \Phi_2^T D_P + \lambda I_N) \\
&= \log \det (\Phi_2^T D_P \Phi_2 + \lambda I_{p_2}) \\
&= \log \det S_{\varphi_2}(P, \lambda)
\end{aligned}
$$
which completes the proof.
\end{proof}

\section{Additional experiments} \label{section:additional-experiments}

In this section, we provide additional experimental results under a simulated environment with the reward function $r_t(x) = 0.8 \cos(3 x^T \theta + \phi(t))$ where the action $x$ and the parameter $\theta$ are randomly sampled from the unit sphere in $\mathbb{R}^d$, and $\phi(t)$ denotes the phase over time.
We use the parameters tuned in Section~\ref{section:experiment} for all the experiments in this section.

\begin{figure}
    \centering
    \begin{subfigure}[b]{0.475\textwidth}
        \centering
        \includegraphics[width=\textwidth, height=1.2in]{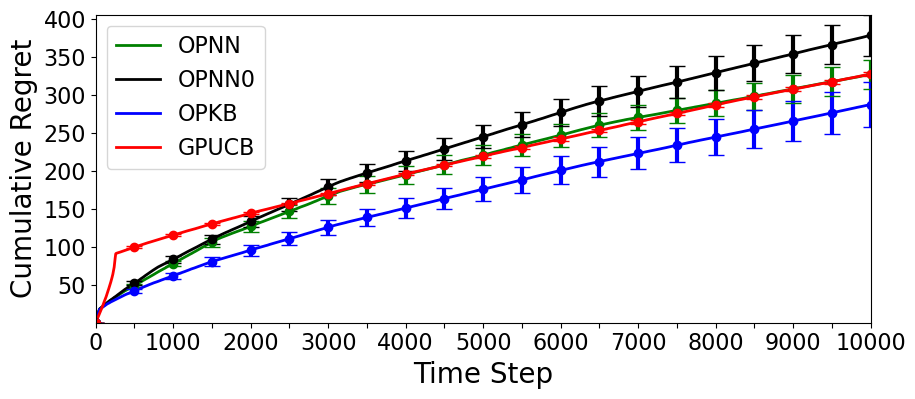}
        \caption[a]%
        {{\small Stationary cosine bandit}}
    \end{subfigure}
    \hfill
    \begin{subfigure}[b]{0.475\textwidth}  
        \centering 
        \includegraphics[width=\textwidth, height=1.2in]{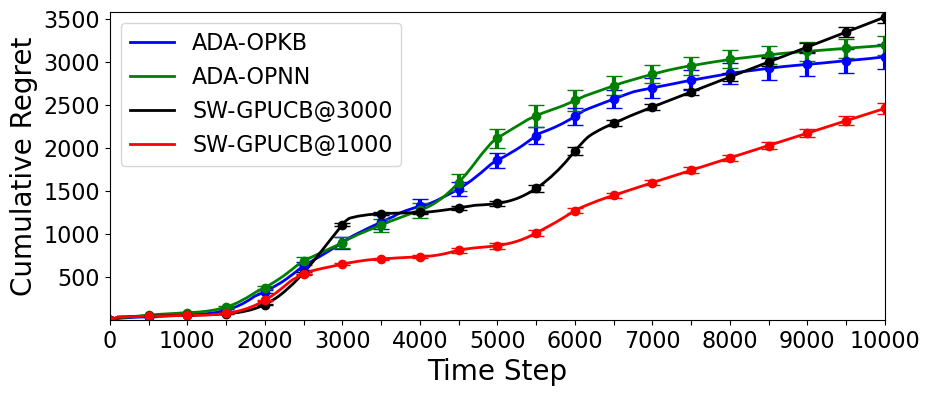}
        \caption[]%
        {{\small Slowly-varying cosine bandit}}    
    \end{subfigure}
    \caption{Cumulative regret comparison of algorithms in cosine bandit environments}
    \label{figure:additional-experiments}
\end{figure}

\subsection{Algorithm Tuning} \label{subsection:algorithm-tuning}

We tune SW-GPUCB, WGPUCB, ADA-OPKB and ADA-OPNN algorithms under the single switch environment.
For SW-GPUCB, we do a grid search for $\lambda$ over the range $\{0.01, 0.02, 0.05, 0.1, \dots, 100\}$, the UCB scale parameter $v$ over $[0.001, 1]$, and the window size over $\{100, 200, 500, 1000, \dots, 10000\}$.
See Algorithm~\ref{alg:igp-ucb} for the definition of $\lambda$.
For WGPUCB, we do a grid search for $\lambda$ over the range $\{0.01, 0.02, 0.05, 0.1, \dots, 100\}$, the UCB scale parameter over $\{0.001, 0.002, 0.005, 0.01, \dots, 1\}$, and the discounting factor over $\{0.99, 0.995, 0.999, 0.9995, 0.9999 \}$.
See Algorithm~\ref{alg:igp-ucb} for the definition of $\lambda$.
For ADA-OPKB and ADA-OPNN, we do a grid search for $\sigma$ over $\{1, 2, 5, 10, 20, 50, 100, 200, 500, 1000\}$ and $c_0, c_1, c_2, c_3, c_4$ over $\{0.001, 0.002, 0.005, 0.01, \dots, 100\}$.
For ADA-OPNN, we do a grid search for the learning rate $\eta$ over $\{10^{-9}, 10^{-8}, 10^{-7} \}$, training steps $J$ over $\{100, 1000, 10000\}$ and regularization parameter $\lambda$ over $\{1, 10, 100, 1000\}$.
We use a neural network of depth $L = 3$ and width $m = 2048$.

\subsection{Stationary cosine bandits}

We perform an experiment to demonstrate that OPNN benefits from dynamically adapting the feature mapping.
We use the cosine bandits described earlier with the phase fixed at $\phi(t) = 0$.
For a comparison, we run the algorithm OPNN0 that does not train the neural network for updating the feature mapping and uses the feature mapping induced by the initial weight of the neural network for all blocks.

The cumulative regrets averaged over 50 random seeds are shown in plot (b) of Figure~\ref{figure:additional-experiments}.
Error bars indicates standard errors of the means.
OPNN outperforms OPNN0, suggesting that updating feature mapping by training the neural network with observed data is beneficial.
Also, note that the performance of OPNN is comparable to GPUCB and OPKB.

\subsection{Slowly-varying cosine bandits}

We perform an experiment on slowly-varying bandits to demonstrate that our change detection based algorithms ADA-OPKB and ADA-OPNN adapt to slowly-varying environments.
We use the cosine bandit described earlier with varying phase $\phi(t)$.
We keep $\phi(t) = 0$ from time 0 to 1000, then let it grow from 0 to $\pi$ linearly from time 1000 to 3000.
From time 4000 to 6000, we let $\phi(t)$ grow again from $\pi$ to $2\pi$ linearly, and then keep $\phi(t) = 2\pi$ until the end of the simulation.

The cumulative regrets averaged over 25 random seeds under the slowly-varying cosine environment are shown in plot(b) of Figure~\ref{figure:additional-experiments}.
Error bars indicate standard errors of the means.
Note that SW-GPUCB with window size 3000, which is the best tuned parameter for the switching environment in Section~\ref{section:experiment}, is outperformed by the change detection based algorithms ADA-OPKB and ADA-OPNN.
If we tune SW-GPUCB again and use SW-GPUCB with window size 1000, SW-GPUCB performs the best.
Similarly, the best tuned WGPUCB under the single switching environment in Section~\ref{section:experiment} is outperformed by ADA-OPKB and ADA-OPNN in the slowly varying environment.

\end{document}